\documentclass{article}

\usepackage[final]{neurips_2024}
\usepackage{caption, subcaption}

\usepackage[utf8]{inputenc} 
\usepackage[T1]{fontenc}    
\usepackage{hyperref}       
\usepackage{url}            
\usepackage{booktabs}       
\usepackage{amsfonts}       
\usepackage{nicefrac}       
\usepackage{xcolor}         
\usepackage{smile}
\usepackage{enumitem}
\let\emptyset\varnothing
\newcommand{\BB}{\mathbb{B}}
\DeclareMathOperator*{\Expt}{\mathbb{E}}
\newcommand{\epsrnd}{\kappa}
\newcommand{\method}{{certified estimator}}
\newcommand{\Leps}{L_{\eps}}

\newcommand{\Lzeta}{L_{\zeta}}
\newcommand{\Lplus}{l_{+}}
\newcommand{\Kepsh}{\cK^{\eps}_h}
\newcommand{\polylog}{\mathrm{polylog}}
\newcommand{\mainalg}{\texttt{Cert-LSVI-UCB}}
\newcommand{\subalg}{\texttt{Cert-LinUCB}}
\usepackage{todonotes}

\DeclareFontFamily{U}{mathx}{}
\DeclareFontShape{U}{mathx}{m}{n}{<-> mathx10}{}
\DeclareSymbolFont{mathx}{U}{mathx}{m}{n}
\DeclareMathAccent{\widecheck}{0}{mathx}{"71}

\usepackage{colortbl}
\definecolor{LightCyan}{rgb}{0.8, 0.9, 1}
\definecolor{LightGray}{rgb}{0.9,0.9,0.9}
\title{Achieving Constant Regret in Linear Markov Decision Processes}

\author{%
Weitong Zhang\thanks{equal contribution}\\
School of Data Science and Society\\
University of North Carolina at Chapel Hill\\
Chapel Hill, NC 27599 \\
\texttt{weitongz@unc.edu} \\
\And
Zhiyuan Fan$^*$\\
EECS \\
Massachusetts Institute of Technology \\
Cambridge, MA 02139 \\
\texttt{fanzy@mit.edu} \\
\AND
Jiafan He \\
Department of Computer Science \\
University of California, Los Angeles \\
Los Angeles, CA 90095\\
\texttt{jiafanhe19@ucla.edu} \\
\And
Quanquan Gu \\
Department of Computer Science \\
University of California, Los Angeles \\
Los Angeles, CA 90095\\
\texttt{qgu@cs.ucla.edu} \\
}

\begin{document}

\maketitle

\begin{abstract}
We study the constant regret guarantees in reinforcement learning (RL). Our objective is to design an algorithm that incurs only finite regret over infinite episodes with high probability. We introduce an algorithm, \mainalg, for misspecified linear Markov decision processes (MDPs) where both the transition kernel and the reward function can be approximated by some linear function up to misspecification level $\zeta$. At the core of~\mainalg~is an innovative \method, which facilitates a fine-grained concentration analysis for multi-phase value-targeted regression, enabling us to establish an instance-dependent regret bound that is constant w.r.t. the number of episodes. Specifically, we demonstrate that for a linear MDP characterized by a minimal suboptimality gap $\Delta$, \mainalg~has a cumulative regret of $\tilde{\mathcal{O}}(d^3H^5/\Delta)$ with high probability, provided that the misspecification level $\zeta$ is below $\tilde{\mathcal{O}}(\Delta / (\sqrt{d}H^2))$. Here $d$ is the dimension of the feature space and $H$ is the horizon. Remarkably, this regret bound is independent of the number of episodes $K$. To the best of our knowledge, \mainalg~is the first algorithm to achieve a constant, instance-dependent, high-probability regret bound in RL with linear function approximation without relying on prior distribution assumptions. 
\end{abstract}

\section{Introduction}
Reinforcement learning (RL) has been a popular approach for teaching agents to make decisions based on feedback from the environment. RL has shown great success in a variety of applications, including robotics~\citep{kober2013reinforcement}, gaming~\citep{mnih2013playing}, and autonomous driving. In most of these applications, there is a common expectation that RL agents will master tasks after making only a bounded number of mistakes, even over indefinite runs. However, theoretical support for this expectation is limited in RL literature: in the worst case, existing works such as \citet{jin2020provably, ayoub2020model, wang2019optimism} only provided $\tilde \cO(\sqrt{K})$ regret upper bounds with $K$ being the number of episodes; in the instance-dependent case, \citet{simchowitz2019non, yang2021q,he2021logarithmic} achieved logarithmic high-probability regret upper bounds (e.g., $\tilde \cO(\Delta^{-1}\log K)$) for both tabular MDPs and MDPs with linear function approximations, provided a minimal suboptimality gap $\Delta$. However, these findings suggest that an agent's regret increases with the number of episodes $K$, contradicting to the expectation of finite mistakes in practice. To close this gap between theory and practice, there is a recent line of work proving constant regrets bound for RL and bandits, suggesting that an RL agent's regret may remain bounded even when it encounters an indefinite number of episodes. \citet{papini2021reinforcement, zhang2021provably} have provided instance-dependent constant regret bound under certain coverage assumptions on the data distribution. However, verifying these data distribution assumptions can be difficult or even infeasible. On the other hand, it is known that high-probability constant regret bound can be achieved unconditionally in multi-armed bandits~\citep{abbasi2011improved} and contextual linear bandits if and only if the misspecification is sufficiently small with respect to the minimal sub-optimality gap~\citep{zhang2023interplay}. This raises a critical question:
\begin{center}
    \emph{Is it possible to design a reinforcement learning algorithm that incurs only constant regret under minimal assumptions?}
    
\end{center}

To answer this question, we introduce a novel algorithm, which we refer to as~\mainalg, for reinforcement learning with linear function approximation. To encompass a broader range of real-world scenarios characterized by large state-action spaces and the need for function approximation, we consider the \emph{misspecified linear MDP}~\citep{jin2020provably} setting, where both the transition kernel and reward function can be approximated by a linear function with approximation error $\zeta$. We show that, with our innovative design of~\method~and novel analysis,~\mainalg~achieves constant regret without relying on any prior assumption on data distributions. Our key contributions are summarized as follows:
\begin{itemize}[leftmargin=*,nosep]
    \item We introduce a parameter-free algorithm, referred to as \mainalg, featuring a novel \method~for testing when the confidence set fails. This certified estimator enables \mainalg~to achieve a constant, instance-dependent, high probability regret bound of $\tilde \cO(d^3H^5/\Delta)$ for tasks with a suboptimality gap $\Delta$, under the condition that the misspecification level $\zeta$ is bounded by $\zeta < \tilde \cO\big(\Delta/(\sqrt{d}H^2) \big)$. This bound is termed a \emph{high probability constant regret bound}, indicating that it does not depend on the number of episodes $K$. We note that this constant regret bound matches the logarithmic expected regret lower bound of $\Omega(\Delta^{-1}\log K)$, suggesting that our result is valid and optimal in terms of the dependence on the suboptimality gap $\Delta$.
    \item When restricted to a well-specified linear MDP (i.e., $\zeta = 0$), the constant high probability regret bound improves the previous logarithmic result $\tilde \cO(d^3H^5\Delta^{-1}\log K)$ in~\citet{he2021logarithmic} by a $\log K$ factor. Our results suggest that the total suboptimality incurred by~\mainalg~remains constantly bounded, regardless of the number of episodes $K$. In contrast to the previous constant regret bound achieved by~\citet{papini2021reinforcement}, our regret bound does not require any prior assumption on the feature mapping, such as the UniSOFT assumption made in~\citet{papini2021reinforcement}. To the best of our knowledge, \mainalg~is the first algorithm to achieve a \emph{high probability constant regret bound} for MDPs without prior assumptions on data distributions. We further show that this constant regret high-probability bound does not violate the logarithmic expected regret bound by letting $\delta = 1 / K$ \footnote{The detailed conversion is presented in  Remark~\ref{rm:const}.}. 
\end{itemize}
\noindent\textbf{Notation.} Vectors are denoted by lower case boldface letters such as $\xb$, and matrices by upper case boldface letters such as $\Ab$. We denote by $[k]$ the set $\{1, 2, \cdots, k\}$ for positive integers $k$. We use $\log x$ to denote the logarithm of $x$ to base $2$. For two non-negative sequence $\{a_n\}, \{b_n\}$, $a_n \leq \cO(b_n)$ means that there exists a positive constant $C$ such that $a_n \le Cb_n$; $a_n \leq \tilde \cO(b_n)$ means there exists a positive constant $k$ such that $a_n \leq \cO(b_n \log^k b_n)$; $a_n \geq \Omega(b_n)$ means that there exists a positive constant $C$ such that $a_n \ge Cb_n$; $a_n \geq \tilde \Omega(b_n)$ means there exists a positive constant $k$ such that $a_n \geq \Omega(b_n \log^{-k} b_n)$; $a_n \geq \omega(b_n)$ means that $\lim_{n\rightarrow \infty} b_{n}/a_{n}=0$. For a vector $\xb  \in \RR^d$ and a positive semi-definite matrix $\Ab \in \RR^{d \times d}$, we define $\|\xb\|_{\Ab}^2=\xb^\top \Ab\xb$. For any set $\cC$, we use $|\cC|$ to denote its cardinality. We denote the identity matrix by $\Ib$ and the empty set by $\emptyset$. The total variation distance of two distribution measures $\PP(\cdot)$ and $\QQ(\cdot)$ is denoted by $\|\PP(\cdot) - \QQ(\cdot)\|_{\mathrm{TV}}$.

\section{Related Work}  
\newcolumntype{g}{>{\columncolor{LightCyan}}c}
\begin{table}[t]
\centering
\begin{tabular}{gggg}
\toprule
\rowcolor{white}
Algorithm & Allow Misspecification?
& Result \\
\midrule
\rowcolor{white}
\texttt{LSVI-UCB}~\citep{he2021logarithmic} & $\times$ & $\tilde\cO(d^3H^5\Delta^{-1} \log(K))$ \\
\rowcolor{LightGray}
\texttt{LSVI-UCB}~\citep{papini2021reinforcement} & $\times$ & $\tilde\cO(d^3H^5\Delta^{-1} \log(1 / \lambda))$ \\
\mainalg~(ours, Theorem~\ref{thm:main})&  \checkmark & $\tilde\cO(d^3H^5\Delta^{-1})$ \\
\bottomrule
\end{tabular}   
\caption{Instance-dependent regret bounds for different algorithms under the linear MDP setting. Here $d$ is the dimension of the linear function $\bphi(s, a)$, $H$ is the horizon length, $\Delta$ is the minimal suboptimality gap. All results in the table represent high probability regret bounds. The regret bound depends the number of episodes $K$ in \citet{he2021logarithmic} and the minimum positive eigenvalue $\lambda$ of features mapping in \citet{papini2021leveraging}. \textbf{Allow Misspecification?} indicates if the algorithm can (\checkmark) handle the misspecified linear MDP or not ($\times$).}
\label{tab:rw}
\end{table}
\paragraph{Instance-dependent regret bound in RL.}
Although most of the theoretical RL works focus on worst-case regret bounds, instance-dependent (a.k.a., problem-dependent, gap-dependent) regret bound is another important bound to understanding how the hardness of different instance can affect the sample complexity of the algorithm. For tabular MDPs, \citet{jaksch2010near} proved a $\tilde \cO(D^2S^2A\Delta^{-1}\log K)$ instance-dependent regret bound for average-reward MDP where $D$ is the diameter of the MDP and $\Delta$ is the policy suboptimal gap. \citet{simchowitz2019non} provided a lower bound for episodic MDP which suggests that the any algorithm will suffer from $\Omega(\Delta^{-1})$ regret bound. \citet{yang2021q} analyzed the optimistic $Q$-learning and proved a $\cO(SAH^6\Delta^{-1}\log K)$ logarithmic instance-dependent regret bound. In the domain of linear function approximation, \citet{he2021logarithmic} provided instance-dependent regret bounds for both linear MDPs (i.e., $\tilde \cO(d^3H^5\Delta^{-1} \log K)$) and linear mixture MDPs (i.e., $\tilde \cO(d^2H^5\Delta^{-1} \log K)$). Furthermore, \citet{dann2021beyond} provided an improved analysis for this instance-dependent result with a redefined suboptimal gap. \citet{zhang2023on} proved a similar logarithmic instance-dependent bound with~\citet{he2021logarithmic} in misspecified linear MDPs, showing the relationship between misspecification level and suboptimality bound. Despite all these bounds are logarithmic depended on the number of episode $K$, many recent works are trying to remove this logarithmic dependence. \citet{papini2021reinforcement} showed that under the linear MDP assumption, when the distribution of contexts $\bphi(s, a)$ satisfies the `diversity assumption'~\citep{hao2020adaptive} called `UniSOFT', then LSVI-UCB algorithm may achieve an expected constant regret w.r.t. $K$. \citet{zhang2021provably} showed a similar result on bilinear MDP~\citep{yang2019reinforcement}, and extended this result to offline setting, indicating that the algorithm only need a finite offline dataset to learn the optimal policy. Table \ref{tab:rw} summarizes the most relevant results mentioned above for the ease of comparison with our results. 

\paragraph{RL with model misspecification.}
All of the aforementioned works consider the well-specified setting and ignore the approximation error in the MDP model. To better understand this misspecification issue, \citet{du2019good} showed that having a good representation is insufficient for efficient RL unless the approximation error (i.e., misspecification level) by the representation is small enough. In particular, \citet{du2019good} showed that an $\tilde \Omega (\sqrt{H / d})$ misspecification will lead to $\Omega(2^H)$ sample complexity for RL to identify the optimal policy, even with a generative model. On the other hand, a series of work~\citep{jin2020provably, zanette2020learning, zanette2020frequentist} provided $\tilde \cO(\sqrt K + \zeta K)$-type regret bound for RL in various settings, where $\zeta$ is the misspecification level\footnote{The misspecification level for these upper bounds is measured in the total variation distance between the ground truth transition kernel and approximated transition kernel, which is strictly stronger than the infinite-norm misspecification used in~\citet{du2019good}.} and we ignore the dependence on the dimension of the feature mapping $d$ and the planing horizon $H$ for simplicity. These algorithms, however, require the knowledge of misspecification level $\zeta$, thus are not \emph{parameter-free}. Another concern for these algorithms is that some of the algorithms~\citep{jin2020provably} would possibly suffer from a \emph{trivial asymptotic regret}, i.e., $\text{Regret}(k) > \omega(k\zeta \cdot \text{poly}(d, H, \log(1/\delta)))$, as suggested by~\citet{vial2022improved}. This means the performance of the RL algorithm will possibly degenerate as the number of episodes $k$ grows. To tackle these two issues, \citet{vial2022improved} propose the \texttt{Sup-LSVI-UCB} algorithm which requires a parameter $\eps_{\text{tol}}$. When $\eps_{\text{tol}} = d / \sqrt{K}$, the proposed algorithm is \emph{parameter-free} but will have a trivial \emph{asymptotic regret bound}. When $\eps_{\text{tol}} = \zeta$, the algorithm will have a non-trivial \emph{asymptotic regret bound} but is not \emph{parameter-free} since it requires knowledge of the misspecification level. Another series of works~\citep{he22corruptions, lykouris2021corruption, wei2022model} are working on the \emph{corruption robust} setting. In particular, ~\citet{lykouris2021corruption, wei2022model} are using  the \emph{model-selection} technique to ensure the robustness of RL algorithms under adversarial MDPs.


\section{Preliminaries}
We consider episodic Markov Decision Processes, which are denoted by $\cM(\cS, \cA, H, \{r_h\}, \{\PP_h\})$. Here, $\cS$ is the state space, $\cA$ is the finite action space, $H$ is the length of each episode, $r_h: \cS \times \cA \mapsto [0, 1]$ is the reward function at stage $h$ and $\PP_h(\cdot | s, a)$ is the transition probability function at stage $h$. The policy $\pi=\{\pi_h\}_{h=1}^H$ denotes a set of policy functions $\pi_h: \cS \mapsto \cA$ for each stage $h$. For given policy $\pi$, we define the state-action value function $Q_h^\pi(s, a)$ and the state value function $V_h^\pi(s)$ as 
\begin{align}
    Q_h^\pi(s, a) = r_h(s, a) + \Expt\Big[ \textstyle{\sum_{h' = h + 1}^H} r_{h'}\big(s_{h'}, \pi_{h'}(s_{h'})\big) \ \Big|\ s_h = s, a_h = a\Big], V_h^\pi(s) = Q_h^\pi\big(s, \pi_h(s)\big), \notag 
\end{align}
where $s_{h'+1}\sim \PP_h(\cdot|s_{h'},a_{h'})$.
The optimal state-action value function $Q_h^*$ and the optimal state value function $V_h^*$ are defined by $Q_h^*(s, a) = \max_{\pi} Q_h^\pi(s, a), V_h^*(s) = \max_{\pi} V_h^\pi(s)$. 

By definition, both the state-action value function $Q_h^\pi(s, a)$ and the state value function $V_h^\pi(s)$ are bounded by $[0, H]$ for any state $s$, action $a$ and stage $h$. For any function $V: \cS \mapsto \RR$, we denote by $[\PP_h V](s, a) = \EE_{s' \sim \PP_h(\cdot | s, a)} V(s')$ the expected value of $V$ after transitioning from state $s$ given action $a$ at stage $h$ and $[\BB_h V](s, a) = r_h(s, a) + [\PP_h V](s, a)$ where $\BB$ is referred to as the \emph{Bellman operator}. 
For each stage $h \in [H]$ and policy $\pi$, the Bellman equation, as well as the Bellman optimality equation, are presented as follows
\begin{align}
    Q_h^\pi(s, a) &= r_h(s, a) + [\PP_h V_{h+1}^\pi](s, a) := [\BB_h V_{h+1}^\pi](s, a),\notag \\
    Q_h^*(s, a) &= r_h(s, a) + [\PP_h V_{h+1}^*](s, a) := [\BB_h V_{h+1}^*](s, a). \notag
\end{align}

We use regret to measure the performance of RL algorithms. It is defined as $\text{Regret}(K) = \sum_{k=1}^K \big(V_1^*(s_1^k) - V_1^{\pi^k}(s_1^k)\big)$, where $\pi^k$ represents the agent's policy at episode $k$.
This definition quantifies the cumulative difference between the expected rewards that could have been obtained by following the optimal policy and those achieved under the agent's policy across the first $K$ episodes, measuring the total loss in performance due to suboptimal decisions.

We consider linear function approximation in this work, where we adopt the \emph{misspecified linear MDP} assumption, which is firstly proposed in~\citet{jin2020provably}.
\begin{assumption}[$\zeta$-Approximate Linear MDP,~\citealt{jin2020provably}]\label{asm:mdp}
For any $\zeta \le 1$, we say a MDP $\cM(\cS, \cA, H, \{r_h\}, \{\PP_h\})$ is a $\zeta$\emph{-approximate linear MDP} with a feature map $\bphi: \cS \times \cA \mapsto \RR^d$, if for any $h \in [H]$, there exist $d$ \emph{unknown} (signed) measures $\bmu_h = \big(\mu_h^{(1)}, \cdots, \mu_h^{(d)}\big)$ over $\cS$ and an unknown vector $\btheta_h \in \RR^d$ such that for any $(s, a) \in \cS \times \cA$, we have 
\begin{align}
    \big\| \PP_h(\cdot | s, a) - \la \bphi(s, a), \bmu_h(\cdot) \ra\big\|_{\mathrm {TV}} \le \zeta,\quad \big| r_h(s, a) - \la \bphi(s, a), \btheta_h\ra \big | \le \zeta, \notag
\end{align}
w.l.o.g. we assume $\forall (s, a) \in \cS \times \cA: \|\bphi(s, a)\| \le 1$ and $\forall h \in [H]: \|\bmu_h(\cS)\| \le \sqrt{d}, \|\btheta_h\| \le \sqrt{d}$.
\end{assumption}
The $\zeta$\emph{-approximate linear MDP} suggests that for any policy $\pi$, the state-action value function $Q_h^\pi$ can be approximated by a linear function of the given feature mapping $\bphi$ up to some misspecification level, which is summarized in the following proposition.
\begin{proposition}[Lemma C.1,~\citealt{jin2020provably}]\label{prop:linear}
For a $\zeta$\emph{-approximate linear MDP}, for any policy $\pi$, there exist corresponding weights $\{\wb_h^\pi\}_{h \in [H]}$ where $\wb_h^\pi = \btheta_h + \int V_{h+1}^\pi(s') \mathrm d\bmu_h(s')$ such that for any $(s, a, h) \in \cS \times \cA \times [H]$, $\big|Q_h^\pi(s, a) - \la \bphi(s, a), \wb_h^\pi\ra \big| \le 2H\zeta$. We have $\|\wb_h^\pi\|_2 \le 2H\sqrt{d}$.
\end{proposition}

Next, we introduce the definition of the suboptimal gap as follows.
\begin{definition}[Minimal suboptimality gap]\label{def:gap}
For each $s\in\cS, a\in\cA$ and step $h\in[H]$, the suboptimality gap $\text{gap}_h(s,a)$ is defined by $\Delta_h(s,a)=V_h^*(s)-Q^*_h(s,a)$
and the minimal suboptimality gap $\Delta$ is defined by $\Delta=\min_{h,s,a}\big\{\Delta_h(s,a): \Delta_h(s,a)\ne 0\big\}$.
\end{definition}
Notably, a task with a larger $\Delta$ means it is easier to distinguish the optimal action $\pi_h^*(s)$ from other actions $a \in \cA$, while a task with lower gap $\Delta$ means it is more difficult to distinguish the optimal action.

\section{Proposed Algorithms}\label{sec:alg}

\subsection{Main algorithm: \mainalg}
\begin{algorithm}[t]
\caption{$\mainalg$}\label{alg:LSVI}
\begin{algorithmic}[1]
\STATE Set $V^k_{H+1}(s) = 0$ for all $(s, k) \in \cS \times [K]$, $\cC^k_{h,l} = \emptyset$ for all $(h,l) \in [H] \times \NN^+, \lambda=16$
\FOR {episode $k = 1, \cdots, K$}
\STATE Set $L_k = \max\{\lceil\log_4 (k / d)\rceil, 0\}$ \label{ln:max-l}
\FOR {step $h = H, \cdots, 1$}
\FOR {phase $l = 1, \cdots, L_k + 1$}
\STATE $\Ub^k_{h,l} = \lambda\Ib + \sum_{\tau \in \cC^{k-1}_{h,l}} \bphi^\tau_h(\bphi^\tau_h)^{\top}$ \label{ln:reg-1}
\STATE $\wb^k_{h,l} = (\Ub^k_{h,l})^{-1} \sum_{\tau \in \cC^{k-1}_{h,l}} \bphi^\tau_h\big(r^\tau_h + \hat V^k_{h+1}(s^\tau_{h+1})\big)$ \label{ln:reg-2}
\STATE $\tilde\Ub^{k, -1}_{h,l} = \epsrnd_{l}\big\lceil(\Ub^k_{h,l})^{-1} / \epsrnd_{l} \big\rfloor$, $\tilde \wb^k_{h,l} = \epsrnd_{l} \big\lceil \wb^k_{h,l} / \epsrnd_{l} \big\rfloor$ where $\epsrnd_{l} = 0.01 \cdot 2^{-4l}d^{-1}$ \label{ln:quanti}
\ENDFOR
\STATE $\hat V^k_h(s^\tau_h), \cdot, \cdot, \cdot = \subalg(s^\tau_h; \{\tilde \wb_{h, l}^k\}_l, \{\tilde \Ub_{h, l}^{k, -1}\}_l, L_k)$ for all $\tau \in [k - 1]$ \label{ln:call}
\ENDFOR 

\STATE Observe $s^k_1 \in \cS$
\FOR {step $h = 1, \cdots, H$}
\STATE $\cdot, \pi^k_h(s^k_h), l^k_h(s^k_h), f^k_h(s^k_h) = \subalg(s^k_h; \{\tilde \wb_{h, l}^k\}_l, \{\tilde \Ub_{h, l}^{k, -1}\}_l, L_k)$ \label{ln:call-2}
\STATE $\cC^{k}_{h, l^k_h(s^k_h)} = \cC^{k-1}_{h, l^k_h(s^k_h)} \cup \{k\}$ \textbf{ if } $f^k_h(s^k_h) = 1$ \textbf{ else } $\cC^{k-1}_{h, l^k_h(s^k_h)}$ \label{ln:add}
\STATE $\cC^k_{h,l} = \cC^{k-1}_{h,l}$ for all $l \neq l_h^k(s_h^k)$ 
\STATE Play $\pi^k_h(s^k_h)$, set $\bphi^k_h = \bphi\big(s^k_h, \pi^k_h(s^k_h)\big)$, receive $r^k_h$ and observe $s^k_{h+1} \in \cS$ \label{ln:play}
\ENDFOR 
\ENDFOR
\end{algorithmic}
\end{algorithm}
We begin by introducing our main algorithm \mainalg,  which is a modification of the \texttt{Sup-LSVI-UCB}~\citep{vial2022improved}. As presented in Algorithm~\ref{alg:LSVI}, for each episode $k$, our algorithm maintains a series of index sets $\cC_{k, h}^l$ for each stage $h \in [H]$ and phase $l$. The algorithm design ensures that for any episode $k$, the maximum number of phases $l$ is bounded by $L_k \le \max\{\lceil\log_4 (k / d)\rceil, 0\}$. During the exploitation step, for each phase $l$ associated with the index set $\cC_{k-1, h}^l$, the algorithm constructs the estimator vector $\wb_{h,l}^k$ by solving the following ridge regression problem in Line~\ref{ln:reg-1} and Line~\ref{ln:reg-2}:
\begin{align*}
    \textstyle{{\wb}^k_{h,l}\leftarrow \argmin_{\wb\in \RR^d}\lambda\|\wb\|_2^2+\sum_{\tau \in \cC_{h,l}^{k-1}}\big(\wb^{\top}\bphi_h^{\tau}-r_h^{\tau}-\hat{V}_{h+1}^k(s_{h+1}^\tau)\big)^2.}
\end{align*} 
After calculating the estimator vector $\wb_{h,l}^k$ in Line~\ref{ln:quanti}, the algorithm quantilizes $\wb_{h,l}^k$ and $(\Ub^k_{h,l})^{-1}$ to the precision of $\epsrnd_{l}$. Similar to \texttt{Sup-LSVI-UCB}~\citep{vial2022improved}, we note $\tilde \Ub^{k, -1}_{h,l}$ is the quantized version of inverse covariance matrix $(\Ub^k_{h,l})^{-1}$ rather than the inverse of quantized covariance matrix $(\tilde \Ub^k_{h,l})^{-1}$. 
The main difference between our implementation and that in~\citet{vial2022improved} is that we use a layer-dependent quantification precision $\epsrnd_{l}$ instead of the global quantification precision $\epsrnd = 2^{-4L} / d$, which enables our algorithm get rid of the dependence on $\cO(\log K)$ in the maximum number of phases $L_k$.

After obtaining $\tilde \wb^k_{h,l}$ and $\tilde \Ub^{k, -1}_{h,l}$, a subroutine, \subalg, is called to calculate an optimistic value function $\hat V^k_h(s_h^\tau)$ for all historical states $s_h^\tau$ in Line~\ref{ln:call}. Then the algorithm transits to stage $h-1$ and iteratively computes $\tilde \wb^{k}_{h,l}$ and $\tilde \Ub^{k, -1}_{h,l}$ for all phase $l$ and stage $h \in [H]$. 

In the exploration step, the algorithm starts to do planning from the initial state $s^k_1$. For each observed state $s^k_h$, the same subroutine, \subalg, will be called in Line~\ref{ln:call-2} for the policy $\pi^k_h(s^k_h)$, the corresponding phase $l^k_h(s^k_h)$, and a flag $f^k_h(s^k_h)$. If the flag $f_h^k(s_h^k) = 1$, the algorithm adds the index $k$ to the index set $\cC^k_{h, l^k_h(s^k_h)}$ in Line~\ref{ln:add}. Otherwise, the algorithm skips the current index $k$ and all index sets remain unchanged. Finally, the algorithm executes policy $\pi^k_h(s^k_h)$, receives reward $r_h^k$ and observes the next state $s^k_{h+1}$ in Line~\ref{ln:play}.
\subsection{Subroutine: \subalg}
\begin{algorithm}[t]
\caption{$\subalg: \big(s; \{\tilde \wb_{h, l}^k\}_l, \{\tilde \Ub_{h, l}^{k, -1}\}_l, L\big) \mapsto \big(\hat V^k_h(s), \pi_h^k(s), l_h^k(s), f_h^k(s)\big)$} \label{alg:Lin}
\begin{algorithmic}[1]
\STATE \textbf{input:}  $s \in \cS, \forall l: \tilde \wb_{h, l}^k \in \RR^d, \tilde \Ub_{h, l}^{k, -1} \in \RR^{d\times d}, L \in \NN^+$
\STATE \textbf{output:} $\hat V^k_h(s) \in \RR, \pi_h^k(s) \in \cA, l_h^k(s) \in \NN^+, f_h^k(s) \in \{0, 1\}$
\STATE $\cA^{k}_{h,1}(s) = \cA, \widecheck{V}^{k}_{h,0}(s) = 0, \hat{V}^{k}_{h,0}(s) = H$
\FOR {phase $l = 1, \cdots, L + 1$}
\STATE Set $Q^k_{h,l}(s, a) = \big\langle\bphi(s, a), \tilde \wb^k_{h,l}\big\rangle$
\STATE Set $\pi^k_{h,l}(s) = \argmax_{a \in \cA^k_{h,l}} Q^k_{h,l}(s, a), V^k_{h,l}(s) = Q^k_{h,l}\big(s, \pi^k_{h,l}(s)\big)$
\IF {$l > L$ }\label{ln:cond1}
    \STATE \textbf{return} $\big(\hat V^k_h(s), \pi_h^k(s), l_h^k(s), f_h^k(s)\big) = \big(\hat V^{k}_{h,l-1}(s), \pi^k_{h,l-1}(s), l, 1\big)$ 
\ELSIF {$\gamma_l \cdot \max_{a \in \cA^k_{h,l}(s)}\|\bphi(s, a)\|_{\tilde\Ub^{k, -1}_{h,l}}  \geq 2^{-l}$} \label{ln:cond2}
    \STATE \textbf{return} $\big(\hat V^k_h(s), \pi_h^k(s), l_h^k(s), f_h^k(s)\big) = \big(\hat V^{k}_{h,l-1}(s), \argmax_{a \in \cA^k_{h,l}(s)}\|\bphi(s, a)\|_{\tilde\Ub^{k, -1}_{h,l}}, l, 1\big)$
\ELSIF {$\max\big\{V^k_{h,l}(s) - 3 \cdot 2^{-l}, \widecheck{V}^{k}_{h,l-1}(s)\big\} > \min\big\{V^k_{h,l}(s) + 3 \cdot 2^{-l}, \hat{V}^{k}_{h,l-1}(s)\big\}$} \label{ln:cond3}
    \STATE \textbf{return} $\big(\hat V^k_h(s), \pi_h^k(s), l_h^k(s), f_h^k(s)\big) = \big(\hat V^{k}_{h,l-1}(s), \pi^k_{h,l-1}(s), l, 0\big)$ \label{ln:except}
\ELSE 
    \STATE $\hat{V}^k_{h,l}(s) = \min\big\{V^k_{h,l}(s) + 3 \cdot 2^{-l}, \hat{V}^{k}_{h,l-1}(s)\big\}$ \label{ln:optim}
    \STATE $\widecheck{V}^k_{h,l}(s) = \max\big\{V^k_{h,l}(s) - 3 \cdot 2^{-l}, \widecheck{V}^{k}_{h,l-1}(s)\big\}$ \label{ln:pess}
    \STATE $\cA^{k}_{h,l+1}(s) = \Big\{a \in \cA^k_{h,l}(s): Q^k_{h,l}(s, a) \geq V^k_{h,l}(s) - 4 \cdot 2^{-l}\Big\}$\label{ln:cond4}
\ENDIF
\ENDFOR
\end{algorithmic}
\end{algorithm}
Next we introduce subroutine \subalg, improved from \texttt{Sup-Lin-UCB-Var} \citep{vial2022improved} that computes the optimistic value function $\hat V_h^k$. The algorithm is described as follows. Starting from phase $l=1$, the algorithm first calculates the estimated state-action function $Q_{h,l}^k(s,a)$ as a linear function over the quantified parameter $\tilde \wb_{h,l}^k$ and feature mapping $\bphi(s,a)$, following Proposition~\ref{prop:linear}.
After calculating the estimated state-action value function $Q_{h,l}^k(s)$, the algorithm computes the greedy policy $\pi_{h,l}^k(s)$ and its corresponding value function $V_{h,l}^k(s)$.

Similar to \texttt{Sup-Lin-UCB-Var}~\citep{vial2022improved}, our algorithm has several conditions starting from Line~\ref{ln:cond1} to determine whether to stop at the current phase or to eliminate the actions and proceed to the next phase $l+1$, which are listed in the following conditions. 
\begin{itemize}[leftmargin=*,nosep]
    \item \noindent \textbf{Condition 1}: In Line~\ref{ln:cond1}, if the current phase $l$ is greater than the maximum phase $L$, we directly stop at that phase and take the greedy policy on previous phase $\pi_h^k(s)=\pi_{h,l - 1}^k(s)$.
    \item \textbf{Condition 2}: In Line~\ref{ln:cond2}, if there exists an action whose uncertainty $\|\bphi(s,a)\|_{\tilde \Ub_{h,l}^{k.-1}}$ is greater than the threshold $2^{-l}\gamma_l^{-1}$, our algorithm will perform exploration by selecting that action.
    \item \textbf{Condition 3}: In Line~\ref{ln:cond3}, we compare the value of the pessimistic value function $\widecheck V_{h, l}^k(s)$ and the optimistic value function $\hat V_{h, l}^k(s)$ which will be assigned in Line~\ref{ln:optim} and Line~\ref{ln:pess}, if the pessimistic estimation will be greater than the optimistic estimation, we will stop at that phase and take the greedy policy on previous phase $\pi_h^k(s) = \pi_{h, l-1}^k(s)$. Only in this case, the Algorithm~\ref{alg:Lin} outputs flag $f_h^k(s) = 0$, which means this observation will not be used in Line~\ref{ln:add} in Algorithm~\ref{alg:LSVI}.
    \item \textbf{Condition 4}: In the default case in Line~\ref{ln:cond4}, the algorithm proceeds to the next phase after eliminating actions.\label{cond:4}
\end{itemize}

Notably, in \textbf{Condition 4}, since the expected estimation precision in the $l$-th phase is about $\tilde \cO(2^{-l})$, our algorithm can eliminate the actions whose state-action value is significantly less than others, i.e., less than $\tilde \cO(2^{-l})$, while retaining the remaining actions for the next phase.

Specially, our algorithm differs from that in~\citet{vial2022improved} in terms of \textbf{Condition 3} to certify the performance of the estimation. In particular, a well-behaved estimation should always guarantee that the optimistic estimation is greater than the pessimistic estimation. According to Line~\ref{ln:optim} and Line~\ref{ln:pess}, this is equivalent to the confidence region for $l$-th phase has intersection of the previous confidence region $[\widecheck V_{h, l-1}^k(s), \hat V_{h, l-1}^k(s)]$. Otherwise, we hypothesis the estimation on $l$-th phase is corrupted by either misspecification or bad concentration event, thus will stop the algorithm. We will revisit the detail of this design later. 

It's important to highlight that our algorithms provide unique approaches when compared with previous works. In particular, \citet{he2021uniformpac} does not eliminate actions and combines estimations from all layers by considering the minimum estimated optimistic value function. This characteristic prevents their algorithm from achieving a uniform PAC guarantee in the presence of misspecification. For a more detailed comparison with \citet{he2021uniformpac}, please refer to Appendix~\ref{app:compare-he}. Additionally, \citet{lykouris2021corruption, wei2022model} focus on a model-selection regime where a set of base learners are employed in the algorithms, whereas we adopt a multi-phase approach similar with SupLinUCB rather than conducting model selection over base learners.

\section{Constant Regret Guarantee}

\begin{theorem}\label{thm:main}
    Under Assumption~\ref{asm:mdp}, let $\gamma_l = 5(l+20+\lceil\log(ld)\rceil)dH\sqrt{\log (16ldH / \delta)}$ for some fixed $0 < \delta <1/4$. With probability at least $1 - 4\delta$, if misspecification level $\zeta$ is below $\tilde \cO\big(\Delta/(\sqrt{d}H^2)\big)$ where $\Delta$ is the minimal suboptimality gap, then for all $K \in \NN^+$, the regret of Algorithm~\ref{alg:LSVI} is upper bounded by
    \begin{align*}
        \textrm{Regret}(K) \leq \tilde \cO\big(d^3H^5\Delta^{-1} \log(1/\delta)\big).
    \end{align*}
    This regret bound is constant w.r.t. the episode $K$.
\end{theorem}

Theorem~\ref{thm:main} demonstrates a constant regret bound with respect to number of episodes $K$. Compared with~\citet{papini2021reinforcement}, our regret bound does not require any prior assumption on the feature mapping  $\bphi$, such as the \emph{UniSOFT} assumption made in~\citet{papini2021reinforcement}. In addition, compared with the previous logarithmic regret bound~\citet{he2021logarithmic} in the well-specified setting, our constant regret bound removes the $\log K$ factor, indicating the cumulative regret no longer grows w.r.t. the number of episode $K$, with high probability. 

\begin{remark}\label{rm:const}
As discussed in~\citet{zhang2023interplay} in the misspecified linear bandits, Our \emph{high probability} constant regret bound does not violate the lower bound proved in~\citet{papini2021reinforcement}, which says that certain diversity condition on the contexts is necessary to achieve an \emph{expected} constant regret bound. 
When extending this high probability constant regret bound to the expected regret bound, we have 
\begin{align*}
    \Expt[\text{Regret}(K)] \le \tilde \cO\big(d^3H^5\Delta^{-1}\log(1 / \delta)\big) \cdot (1 - \delta) + \delta K,
\end{align*}
which depends on the number of episodes $k$.
To obtain a sub-linear expected regret, we can choose $\delta = 1 / K$, which yields a logarithmic expected regret $ \tilde\cO(d^3H^5\Delta^{-1}\log K)$ and does not violate the lower bound in \citet{papini2021reinforcement}.
\end{remark}

\begin{remark}
    \citet{du2019good} provide a lower bound showing the interplay between the misspecification level $\zeta$ and suboptimality gap $\Delta$ in a weaker setting, which we discuss in detail in Appendix~\ref{app:lowerbound}. Along with the result from~\citet{du2019good}, our results suggests that ignoring the dependence on $H$, $\zeta = \tilde \cO(\Delta / \sqrt{d})$ plays an important seperation for if a misspeficied model can be efficiently learned. This result is also aligned with the positive result and negative result for linear bandits~\citep{lattimore2020learning, zhang2023interplay}.
\end{remark}

\section{Technical Challenges and Highlight of Proof Techniques} \label{sec:techniques}
In this section, we highlight several major challenges in obtaining the constant regret under misspecified linear MDP assumption and how our method, especially the~\method, tackles these challenges. 

\subsection{Challenge 1. Achieving layer-wise local estimation error.}
In the analysis of the value function under misspecified linear MDPs, we follow the multi-phase estimation strategy~\citep{vial2022improved} to eliminate suboptimal actions and improve the robustness of the next phase estimation. 
Similar approaches have been observed in~\citet{zhang2023interplay, chu2011contextual} within the framework of (misspecified) linear bandits. However, unlike linear bandits, when constructing the empirical value function $\hat{V}_{h}$ for stage $h$ in linear MDPs, \citet{jin2020provably} requires a covering statement on value functions to ensure the convergence of the regression, which is written by: (see Lemma D.4 in~\citet{jin2020provably} for details)
\begin{align} \label{eq:proof-sketch-coff-radius}
     \Big\|\textstyle{\sum_{\tau \in \cC}}\bphi^\tau_h \big[\hat{V}_{h+1}^k(s^\tau) - \EE[\hat{V}_{h+1}^k(s^\tau)] \big] \Big\|_{\Ub_h^{-1}} \leq \tilde \cO_H \Big(\sqrt{d \log (|\cC|) + \log (|\cV_{h+1}^k| / \delta)} + \sqrt{d}\epsrnd\Big), 
\end{align}
where we employ notation $\tilde \cO_H$ to obscure the dependence on $H$ to simplify the presentation. We use the notation $\cV_{h+1}^k$ to denote as an $\epsrnd$-covering, (or quantification in \citet{takemura2021parameter, vial2022improved}) for the value functions $\hat{V}_{h+1}^k$.
However, in the multi-phase algorithm, the empirical value function $\hat V_{h+1}^k$ from the subsequent stage $h+1$, which is formulated using all pairs of parameters $\big\{\wb_{h, \ell}^k, \Ub_{h, \ell}^k\big\}_\ell$ in $L$ phases. Consequently, the covering number $\log |\cV_{h+1}^k|$ is directly proportional to the number of phases $L = \cO(\log K)$.

Therefore, when analyzing any single phase $l$, prior analysis cannot eliminate the $\log K$ term from~\eqref{eq:proof-sketch-coff-radius} to achieve a \emph{local} estimation error independent that is independent of the logarithmic number of global episodes $\log K$. Furthermore, due to the algorithm design of previous methods~\citep{vial2022improved}, additional $\log K$ terms may be introduced by global quantification (i.e., $\eps_{tol} = d / \sqrt{K})$.
\paragraph{Our approach: \subalg.} 
To tackle this challenge, we introduce the \method~into Algorithm~\ref{alg:Lin} and use a `local quantification' to ensure the quantification error of each phase $l$ depend on the local phase $\tilde \cO(l)$ instead of the global parameter $\log K$. The \method~works as follows: Considering the concentration term we need to control for each phase $l$: 
\begin{align} \label{eq:uncertainty-cert-linucb}
\Big\|\textstyle{\sum_{\tau \in \cC_{h, l}^k}}\bphi^\tau_h \big[\hat{V}_{h+1}^k(s^\tau) - \EE[\hat{V}_{h+1}^k(s^\tau)]\big] \Big\|_{(\Ub^k_{h,l})^{-1}},
\end{align}
as discussed in \textbf{Challenge 1}, the function class $\hat\cV_{h+1}^k \ni \hat V_{h+1}^k$ involves $L = \cO(\log K)$ parameters, leading to a $\log K$ dependence in the results when using traditional routines. The idea of \method~is to get rid of this by not directly controlling $\log |\cV_{h+1}^k|$. 
Instead, \method~establishes a covering statement for the value function class $\cV_{h+1, \Lplus}^k \ni \hat{V}_{h+1, \Lplus}^k$, where $\hat V_{h+1, \Lplus}^k$ is the value function that only incorporates the first $\Lplus$ phases of parameters $\big\{\wb_{h, \ell}^k, \Ub_{h, \ell}^k\big\}_\ell$. Under this framework, the covering statement becomes:
\begin{lemma}[Lemma~\ref{lm:concentration-V}, informal]\label{lm:informal-b-4}
Let $\hat V_{h+1, \Lplus}^k$ be the output of Algorithm~\ref{alg:Lin} terminated at phase $\Lplus \in \NN^+$, then with probability is at least $1 - 2\delta$,
\begin{align*}
\Big\|\textstyle{\sum_{\tau \in \cC_{h, l}^k}}\bphi^\tau_h \big[\hat{V}_{h+1, \Lplus}^k(s^\tau) - \EE[\hat{V}_{h+1, \Lplus}^k(s^\tau)]\big] \Big\|_{(\Ub^k_{h,l})^{-1}} \le \gamma_{l,\Lplus} = 5\Lplus dH\sqrt{\log(16ldH/\delta)}.
\end{align*}
\end{lemma}
Lemma~\ref{lm:informal-b-4} suggests a concentration inequality at any phase $\Lplus$, and the following lemma suggests that this procedure will only introduce an $\tilde \cO(2^{-\Lplus})$ error, under some faithful extension of the $\hat V^k_{h,\Lplus}(s)$:
\begin{lemma}[Lemma~\ref{lm:final-value-func-pre}, informal]
\label{lm:informal-b-2}
    For any $\Lplus \in \NN^+$, 
    $|\hat{V}^k_{h}(s) - \hat V^k_{h,\Lplus}(s)| \leq 6 \cdot 2^{-\Lplus}$.
\end{lemma}
Therefore, if a large enough $\Lplus$ can be reached in Algorithm~\ref{alg:Lin}, combining Lemma~\ref{lm:informal-b-4} and Lemma~\ref{lm:informal-b-2} allow us to bound \eqref{eq:uncertainty-cert-linucb} without introducing $\log K$ factors. The next lemma shows that the Line~\ref{ln:cond3} will only never be triggered in shallow layer $l$.
\begin{lemma}[Lemma~\ref{lm:interplay}, informal]\label{lm:informal-B-8}
With probability at least $1 - 2\delta$, for any $(k, h) \in [K] \times [H]$, Line~\ref{ln:cond3} in Algorithm~\ref{alg:Lin} can only be triggered on phase $l \ge \tilde \Omega\big(\log (1 / \zeta)\big)$.
\end{lemma}
Lemma~\ref{lm:informal-B-8} delivers a clear message: In the well-specified setting, Line~\ref{ln:cond3} will never be triggered ($l \ge \infty$). When the misspecification level is large, then Line~\ref{ln:cond3} will be more likely triggered, indicating it's harder for the algorithm to proceed to deeper layer.
The contribution of the \method~yields the following important lemma regarding the `local estimation error': 
\begin{lemma}[Lemma~\ref{lm:main}, Informal] \label{lm:main-shorten}
    With high probability, for any $\eps > \tilde\Omega(\sqrt{d}H^2\zeta)$ and $h \in [H]$, \mainalg~ensures
    $
    \sum_{k=1}^{\infty} \ind\big[V_h^*(s^k_h) - V^{\pi^k}_h(s^k_h) \geq \eps \big] \leq \tilde \cO\big(d^3 H^4 \eps^{-2}\big)$.
\end{lemma}
\begin{remark}
    \citet{he2021uniformpac} achieved a similar $\cO\big(d^3 H^5 \eps^{-2}\big)$ \emph{uniform-PAC} bound for (well-specified) linear MDP. Comparing with Lemma~\ref{lm:main-shorten} with $\zeta = 0$, one can find that our result is better than \citet{he2021uniformpac}. In addition, Lemma~\ref{lm:main-shorten} ensures this \emph{uniform-PAC} result under all stage $h \in [H]$ while \citet{he2021uniformpac} only ensure the $h = 1$. This improvement is achieved by a more efficient data selection strategy which we will discuss in detail in Appendix~\ref{app:compare-he}.
\end{remark}
\subsection{Challenge 2. Achieving constant regret from local estimation error}
In misspecified linear bandits, \citet{zhang2023interplay} concludes their proof by controlling $\sum_{k=1}^\infty \ind[V_1^*(s_1^k) - V_1^\pi(s_1^k) \ge \Delta]$\footnote{We employ the RL notations and set $h = 1$ for the ease of comparison.}. 
Although it is trivial showing that rounds with instantaneous regret $V_1^*(s_1^k) - V_1^\pi(s_1^k) < \Delta$ is optimal in bandits (i.e., $V_1^*(s_1^k) = V_1^\pi(s_1^k)$), previous works fail to reach a similar result for RL settings. 
This difficulty arises from the randomness inherent in MDPs: Consider a policy $\pi$ that is optimal at the initial stage $h = 1$. After the initial state and action, the MDP may transition to a state $s_2'$ with a small probability $p$ where the policy $\pi$ is no longer optimal, or to another state $s_2$ where $\pi$ remains optimal until the end. In this context, the gap between $V_1^*(s_1)$ and $V_1^\pi(s_1)$ can be arbitrarily small, given a sufficiently small $p > 0$:
\begin{align*}
    V_1^*(s_1) - V_1^\pi(s_1) = p\big(V_2^*(s_2') - V_2^\pi(s_2')\big) + (1 - p) \big(V_2^*(s_2) - V_2^\pi(s_2)\big) = p\big(V_2^*(s_2') - V_2^\pi(s_2')\big).
\end{align*}
Therefore, one cannot easily draw a constant regret conclusion simply by controlling $\sum_{k=1}^\infty \ind[V_1^*(s_1^k) - V_1^\pi(s_1^k) \ge \Delta]$ since the gap between $V_1^*(s_1^k) - V_1^\pi(s_1^k)$ needs to be further fine-grained controlled. In short, the existence of $\Delta$ describing the minimal gap between $V^*(s) - Q^*(s, a)$ cannot be easily applied to controlling regret $V^*(s) - V^\pi(s)$.

\paragraph{Our approach: A fine-grained concentration analysis}
We address this challenge by providing a fine-grained concentration analysis in connecting the gap with the regret. Notice that the regret $V^*_h(s_h) - V^{\pi^k}_h(s_h)$ in episode $k$ is the expectation of cumulative suboptimality gap $\Expt[\sum_{h=1}^{H} \Delta^k_h]$ taking over trajectory $\{s^k_h\}_{h=1}^H$.
In addition, the variance of the random variable can be self-bounded according to 
\begin{align*}
\Var\Big[\textstyle{\sum_{h=1}^{H}} \Delta^k_h\Big] &\leq \Expt\Big[\Big(\textstyle{\sum_{h=1}^{H}} \Delta^k_h\Big)^2\Big] \leq H^2 \Expt\Big[\textstyle{\sum_{h=1}^{H}} \Delta^k_h\Big] = H^2\big(V^*_1(s^k_1) - V^{\pi^k}_1(s^k_1)\big).
\end{align*}
Denote $\eta_k$ be the difference between $V^*_h(s_h) - V^{\pi^k}_h(s_h)$ and the actual $\sum_{h=1}^{H} \Delta^k_h$.
Freedman inequality (Lemma~\ref{lm:freedman}) implies that $\sum_{t=1}^{T}\eta^t \geq aC$ and $\sum_{t=1}^{T}\Var[\eta^t] \leq vC$ happens at the same with a small probability for certain constant $a$ and $v$. Using a fine-grained union bound statement over $C$, we can reach the following statement indicates the cumulative regret can be upper bounded using the cumulative suboptimality gap:
\begin{lemma}[Lemma~\ref{lm:refined-regret-gap}, Informal] \label{lm:refined-regret-gap-shorten}
    The following statement holds with high probability:
    \begin{align*}
        \textstyle{{\sum_{k=1}^{K}}\big(V^*_h(s_h) - V^{\pi^k}_h(s_h) \big) \leq \tilde \cO\Big({\sum_{k=1}^{K}}\sum_{h=1}^{H}} \Delta^k_h + H^2\Big).
    \end{align*}
\end{lemma}
Comparing with Lemma~6.1 in~\citet{he2021logarithmic}, Lemma~\ref{lm:refined-regret-gap-shorten} eliminates the $\log K$ dependence, which is achieved by the aforementioned fine-grained union bound. As a result, together with Lemma~\ref{lm:main-shorten}, we reach the desired statement that \mainalg~achieves constant regret bound when the misspecification is sufficiently small against the minimal suboptimality gap.
\section{Conclusions and Limitations}\label{sec:conc}
In this work, we proposed a new algorithm, called \method, for reinforcement learning with a misspecified linear function approximation. Our algorithm is parameter-free and does not require prior knowledge of misspecification level $\zeta$ or the suboptimality $\Delta$. Our algorithm is based on a novel \method~and provides the first constant regret guarantee for misspecified linear MDPs and (well-specified) linear MDPs. 

\paragraph{Limitations.} Despite these advancements, several aspects of our algorithm and analysis warrant further investigation. One significant open question is whether the dependency on the planning horizon and dimension $d, H$ can achieve optimal instance-dependent regret bounds. For the gap-independent regret bounds, the regret lower bound is $\Omega(d\sqrt{H^3 K})$ as shown by \citet{zhou2021nearly}, and this benchmark has recently been met by works such as \citet{he2022nearly, agarwal2022vo}. Additionally, our analysis assumes uniform misspecification across all actions. Investigating other types of misspecifications could lead to more sophisticated results, enhancing the algorithm's robustness and applicability to diverse real-world scenarios. This exploration remains an important direction for future research.

\begin{ack}
We thank the anonymous reviewers for their helpful comments. This work was done while WZ was a PhD student at UCLA. WZ is partially supported by UCLA dissertation year fellowship and the research fund from UCLA-Amazon Science Hub. JH and QG are partially supported by the research fund from UCLA-Amazon Science Hub. The views and conclusions contained in this paper are those of the authors and should not be interpreted as representing any funding agencies.
\end{ack}
\bibliographystyle{ims}
\bibliography{refs.bib}

\appendix

\section{Additional Related Work}
\paragraph{RL with linear function approximation.} Recent years have witnessed a line of work focusing on RL with linear function approximation to tackle RL tasks in large state space. A widely studied MDP model is linear MDP \citep{jin2020provably}, where both the transition kernel and the reward function are linear functions of a given feature mapping of the state-action pairs $\bphi(s,a)$. Several works have developed RL algorithms with polynomial sample complexity or sublinear regret bound in this setting. For example, LSVI-UCB \citep{jin2020provably} has an $\tilde \cO(\sqrt{d^3H^4K})$ regret bound, randomized LSVI \citep{zanette2020frequentist} has an $\tilde \cO(\sqrt{d^4H^5K})$ regret bound and \citet{ishfaq2021randomized} achieved an $\tilde \cO(\sqrt{d^3H^4K})$. \citet{he2022nearly} then improves this regret bound to a nearly minimax-optimal result $\tilde \cO(d\sqrt{H^3K})$ while \citet{agarwal2022vo} provides a general function approximation extension given the above result. 
Linear mixture/kernel MDPs \citep{modi2020sample,jia2020model,ayoub2020model,zhou2020provably} have also emerged as another model that enables model-based RL with linear function approximation. In this setting, the transition kernel is a linear function of a feature mapping on the triplet of state, action, and next state $\bphi(s,a,s')$. Nearly minimax optimal regrets can be achieved for both finite-horizon episodic MDPs \citep{ayoub2020model, zhou2021nearly} and infinite-horizon discounted MDPs \citep{zhou2020provably} under this assumption.

\section{Additional Discussions on Algorithm Design and Result}
\subsection{Comparison with~\citet{he2021uniformpac}}\label{app:compare-he}
It is worth comparing our algorithm with \citet{he2021uniformpac}, which also provides a uniform PAC bound for linear MDPs. Both our algorithm and theirs utilize a multi-phase structure that maintains multiple regression-based value function estimators at different phases. Despite this similarity, there are several major differences between our algorithm and that in \citet{he2021uniformpac}, which are highlighted as follows:
\begin{enumerate}[leftmargin=*,nosep,label=(\arabic*)]
\item In Line~\ref{ln:reg-2} of Algorithm~\ref{alg:LSVI}, when calculating the regression-based estimator, for different phase $l$, we use the same regression target $\hat V_{h+1}^k$, while their algorithm uses different $V_{h+1, l}^k$ for different phase $l$.  \label{enum:1}
\item When aggregating the regression estimators over all different $L_k$ phases, we follow the arm elimination method as in~\citet{chu2011contextual}, while \citet{he2021uniformpac} simply take the point-wise minimum of all estimated state-action functions, i.e., $Q(s, a) = \min_{l \in [L_k]}Q_{k, h}^l(s, a)$. \label{enum:2}
\item When calculating the phase $l_h^k(s_h^k)$ for a trajectory ${s_1^k, s_2^k, \cdots, s_H^k}$, \citet{he2021uniformpac} require that the phase $l_h^k(s_h^k)$ to be monotonically decreasing with respect to the stage $h$, i.e., $l_h^k(s_h^k) \le l_{h-1}^k(s_{h-1}^k)$ (see line 19 in Algorithm 2 in~\citet{he2021uniformpac}). Such a requirement will lead to a poor estimation for later stages and thus increase the sample complexity. In contrast, we do not have this requirement or any other requirements related to $l_h^k(s_h^k)$ and $l_{h-1}^k(s_{h-1}^k)$. \label{enum:3}
\end{enumerate}
As a result, by~\ref{enum:3}, \citet{he2021uniformpac} have to sacrifice some sample complexity to make their algorithm work for different target value functions $V_{h+1, l}^k$. As a comparison, since we use the same regression target for different phase $l$, we do not have to make such a sacrifice in~\ref{enum:3}. Moreover, by ~\ref{enum:2}, \citet{he2021uniformpac} cannot deal with linear MDPs with misspecification, while our algorithm can handle misspecification as in \citet{vial2022improved}.
\subsection{Discussion on Lower Bounds of Sample Complexity}\label{app:lowerbound}
We present a lower bound from~\citet{du2019good} to better illustrate the interplay between the misspecification level $\zeta$ and the suboptimality gap $\Delta$. 

\begin{assumption}[Assumption~4.3,~\citealt{du2019good}, $\zeta$-Approximate Linear MDP]\label{asm:du}
    There exists $\zeta > 0$, $\btheta_h \in \RR^d$ and $\bmu_h: \cS \mapsto \RR^d$ for each stage $h \in [H]$ such that for any $(s, a, s') \in \cS \times \cA \times \cS$, we have $\big|\PP_h(s' | s, a) - \la \bphi(s, a) , \bmu_h(s')\ra\big| \le \zeta$ and $\big|r(s, a) - \la \bphi(s, a), \btheta_h\ra\big| \le \zeta$.
\end{assumption}
\begin{theorem}[Theorem~4.2,~\citealt{du2019good}]\label{thm:negative}
    There exists a family of hard-to-learn linear MDPs with action space $|\cA| = 2$ and a feature mapping $\bphi(s,a)$ satisfying Assumption~\ref{asm:du}, such that for any algorithm that returns a $1/2$-optimal policy with probability $0.9$ needs to sample at least $\Omega(\min\{|\cS|, 2^H, \exp(d\zeta^2 / 16)\})$ episodes.
\end{theorem}
\begin{remark}
    As claimed in~\citet{du2019good}, Theorem~\ref{thm:negative} suggests that when misspecification in the $\ell_{\infty}$ norm satisfies $\zeta = \Omega (\Delta\sqrt{H / d})$, the agent needs an exponential number of episodes to find a near-optimal policy, where $\Delta = 1/2$ in their setting. It is worth noting that Assumption~\ref{asm:du} is a $\ell_\infty$ approximation for the transition matrix. Such a $\ell_{\infty}$ guarantee ($\|\cdot \|_{\infty} \le \zeta$) is weaker than 
    the $\ell_1$ guarantee ($\|\cdot\|_{1} \le \zeta$) provided in Assumption~\ref{asm:mdp}. 
    So it's natural to observe a positive result when making a stronger assumption and a negative result when making a weaker assumption. In addition, despite of this difference, one could find that $\zeta \sim \Delta / \sqrt{d}$ plays a vital role in determining if the task can be efficiently learned. Similar positive and negative results are also provided in~\citet{lattimore2020learning, zhang2023interplay} in the linear contextual bandit setting (a special case of linear MDP with $H = 1$).
\end{remark}
\section{Constant Regret Guarantees for \mainalg}\label{app:proof}
In this section, we present the proof of Theorem~\ref{thm:main}. To begin with, we recap the notations used in the algorithm and introduce several shorthand notations that would be employed for the simplicity of latter proof. The notation table is presented in Table~\ref{tab:notations}.Any proofs not included in this section are deferred to Appendix~\ref{app:proof1}.
\begin{table}[!h]
\centering
\begin{tabular}{cl}
\hline
Notation & Meaning \\
\hline
$\zeta$ & Misspecification level of feature map $\phi_h$. (see Definiton~\ref{asm:mdp}) \\ \rowcolor{LightGray}
$\Delta$ & Minimal suboptimality gap among $\Delta_h$. (see Definition~\ref{def:gap}) \\ 
$s_h^k,a_h^k$ & States and actions introduced in the episode $k$ by the policy $\pi_k$. \\ \rowcolor{LightGray}
$Q^\pi_h(s, a), V^\pi_h(s)$ & Ground-truth state-action value function and state value function of policy $\pi$. \\
$Q^*_h(s, a), V^*_h(s)$ & The optimal ground-truth state-action value function and state value function. \\ \rowcolor{LightGray}
$\Delta_h(s, a)$ & Suboptimal gap with respect to the optimal policy $\pi^*$. (see Definition~\ref{def:gap}) \\ 
$\PP_h, \BB_h$ & The ground-truth transition kernel and the Bellman operator. \\ \rowcolor{LightGray}
$\epsrnd_{l}$ & The quantification precision in the phase $l$. (see Algorithm~\ref{alg:LSVI}) \\
$\gamma_l $ & The confidence radius in the phase $l$. (see Theorem~\ref{thm:main}) \\ \rowcolor{LightGray}
$\cC^k_{h,l}$ & Index sets during phase $l$ in the episode $k$. (see Algorithm~\ref{alg:LSVI}) \\ 
$\wb^k_{h,l}, \Ub^k_{h,l}$ & Empirical weights and covariance matrix in the phase $l$. (see Algorithm~\ref{alg:LSVI}) \\ \rowcolor{LightGray}
$\tilde \wb^k_{h,l}, \tilde \Ub^k_{h,l}$ & Quantified version of $\wb^k_{h,l}$ and $\Ub^k_{h,l}$. (see Algorithm~\ref{alg:LSVI}) \\ 
$\hat{V}^k_h(s)$ & The overall optimistic state value function. (see Algorithm~\ref{alg:Lin}) \\ \rowcolor{LightGray}
$Q^k_{h,l}(s, a)$ & Empirical state-action value function in phase $l$. (see Algorithm~\ref{alg:Lin}) \\
$V^k_{h,l}(s)$ & Empirical state value function in phase $l$. (see Algorithm~\ref{alg:Lin}) \\ \rowcolor{LightGray}
$\hat{V}^k_{h,l}(s)$ & Optimistic state value function in phase $l$. (see Definition~\ref{def:func-class}) \\ 
$\widecheck{V}^k_{h,l}(s)$ & Pessimistic state value function in phase $l$. (see Algorithm~\ref{alg:Lin}) \\ \rowcolor{LightGray}
$\pi^k_h$ & Policy played in the episode $k$. (see Algorithm~\ref{alg:Lin}) \\  
$\pi^k_{h,l}$ & Policy induced at state $s$ during phase $l$ of episode $k$. (see Algorithm~\ref{alg:Lin}) \\ \rowcolor{LightGray}
$l^k_h(s)$ & The index of the phase at which state $s$ stops in episode $k$. (see Algorithm~\ref{alg:Lin}) \\ 
$\phi^k_h$ & The feature vector observed in the episode $k$. (see Algorithm~\ref{alg:LSVI}) \\ \rowcolor{LightGray}
$\cV^k_{h,l}$ & Function family of all optimistic state function $\hat{V}^k_{h,l}$. (see Definition~\ref{def:func-class}) \\ 
$\gamma_{l,\Lplus}$ & The confidence radius with covering on phase $\Lplus$. (see Definition~\ref{def:good-event}) \\ \rowcolor{LightGray}
$\Lplus$ & The phase offsets for the covering statement. (see Lemma~\ref{lm:concentration-hatV})\\  
$\chi$ & The inflation on misspecification. (see Lemma~\ref{lm:estimate-error}) \\ \rowcolor{LightGray}
$\Lzeta$ & The deepest phase that tolerance $\zeta$ misspecification. (see Lemma~\ref{lm:interplay}). \\ 
$\Leps$ & The shallowest phase that guarantees $\eps$ accuracy. (see Lemma~\ref{lm:uncertainty-sum}). \\ \rowcolor{LightGray}
$\Delta^k_h$ & The suboptimaility gap of played policy $\pi^k_h$ at state $s^k_h$. (see Lemma~\ref{lm:constant-gap-chosen}) \\ 
$\cG_1$ & The event defined in Definition~\ref{def:good-event}. \\ \rowcolor{LightGray}
$\cG_2$ & The event defined in Definition~\ref{def:good-event2}. \\ 
$\cG_\eps$ & The condition defined in Definition~\ref{def:interplay-condition}. \\ \rowcolor{LightGray}
\hline
\end{tabular}   
\caption{Notations used in algorithm and proof}
\label{tab:notations}
\end{table}
\subsection{Quantized State Value Function set $\cV_{h, l}^k$.} \label{sec:quant-value}
To begin our proof, we first extend the definition of $\hat V_{h, l}^k$ to arbitrary $l$ and give a formal definition of the state value function class $\cV_{h, l}^k$ as we skip the detail of this definition in Section~\ref{sec:techniques}.
\begin{definition}\label{def:func-class}
    We extend the definition of state value function $\hat V_{h, l}^k$ to any tuple $(k, h, l) \in [K] \times [H] \times \NN^+$ by 
    \begin{align*}
        \hat V_{h, l}^k, \cdot, \cdot, \cdot = \subalg\big(s; \{\tilde \wb_{h, \ell}^{k}\}_{\ell=1}^l, \{\tilde \Ub_{h, \ell}^{k, -1}\}_{\ell=1}^l, l\big)
    \end{align*}
    We also define the state value function family $\cV_{h, l}^k$ be the set of all possible $\hat V_{h, l}^k$.
       \begin{align*}
        \cV_{h,l}^k = \Big\{\hat V_{h, l}^k \ \Big|\ \hat V_{h, l}^k, \cdot, \cdot, \cdot = \subalg\big(s; \{\tilde \wb_{\cdot, \ell}^{\cdot}\}_{\ell=1}^l, \{\tilde \Ub_{\cdot, \ell}^{\cdot, -1}\}_{\ell=1}^l, l\big) \Big\}
    \end{align*}
    where $\{\tilde \wb_{\cdot, \ell}^\cdot\}_{\ell=1}^l$ and $\{\tilde \Ub_{\cdot, \ell}^{\cdot, -1}\}_{\ell=1}^l$ are referring to \emph{any} possible parameters generated by Line~\ref{ln:quanti} in Algorithm~\ref{alg:LSVI}. 
\end{definition}

It is worth noting that one can check the definition of $\hat V_{h, l}^k$ here is consistent with those computed in Algorithm~\ref{alg:Lin} with $l < l^k_h(s)$. Therefore, we will not distinguish between the notations in the remainder of the proof.


The following lemma controls the distance between $\hat V_h^k(s)$ and $\hat V_{h, l}^k(s)$ for any phase $l$.


\begin{lemma}\label{lm:final-value-func-pre}
    For any $(k, h, s) \in [K] \times [H] \times \cS, l \in [l^k_h(s)-1]$, it holds that 
    \begin{align}
    \widecheck{V}^{k}_{h,l}(s) \leq \hat{V}^k_{h}(s) \leq \hat{V}^{k}_{h,l}(s),\quad |\hat{V}^k_{h}(s) - \hat{V}^{k}_{h,l}(s)| \leq 6 \cdot 2^{-l}. \notag 
    \end{align}
    Moreover, for any tuple $(k, h, s, \Lplus) \in [K] \times [H] \times \cS \times \NN^+$, the difference $|\hat{V}^k_{h}(s) - \hat V^k_{h,\Lplus}(s)|$ is bounded by $6 \cdot 2^{-\Lplus}$, following the extension of the definition scope of $\hat V^k_{h,\Lplus}$ as outlined in Definition~\ref{def:func-class}.
\end{lemma}
Lemma~\ref{lm:final-value-func-pre} suggests that given any phase $\Lplus$, $\hat V_{h, l}^k$ is close to $\hat V_h^k$. This enables us to construct covering on $\hat V_h^k$ using the covering on $\hat V_{h, l}^k$.

\subsection{Concentration of State Value Function $\hat V_h^k(s)$}
In this subsection, we provide a new analysis for bounding the self-normalized concentration of $\Big\|\sum_\tau \bphi_h^\tau\big([\PP_h \hat V_h^k](s_h^\tau, a_h^\tau) - \hat V_h^k(s_{h+1}^\tau)\big)\Big\|_{\Ub^{-1}}$ to get rid of the $\log k$ factor in~\citet{vial2022improved}. 

To facilitate our proof, we define the filtration list $\cF_h^k = \Big\{\big\{s_i^j, a_i^j\big\}_{i=1, j=1}^{H, k - 1}, \big\{s_i^k, a_i^k\big\}_{i=1}^{h}\Big\}$.
It is easy to verify that $s_h^k, a_h^k$ are both $\cF_{h}^k$-measurable. Also, for any function $V$ built on $\cF_{h}^k$, $[\PP_h V](s_h^k, a_h^k) - V(s_{h+1}^k)$ is $\cF_{h+1}^k$-measurable and it is also a zero-mean random variable conditioned on $\cF_{h}^k$.

The first lemma we provide is similar with~\citet{vial2022improved}, which shows the self-normalized concentration property for each phase $l$ and any function $V \in \cV_{h, l}^k$.
\begin{definition}\label{def:good-event}
    For some fixed mapping $l \mapsto \Lplus = \Lplus(l)$ that $\Lplus \geq l$, we define the bad event as
    \begin{align*}
        \cB_1(k, h, l, V) = \Bigg\{ \Bigg \|\sum_{\tau \in \cC^{k-1}_{h,l}} \bphi^\tau_h 
        \big([\PP_h V](s^\tau_h, a^\tau_h)
         - V(s^\tau_{h+1})\big)\Bigg\|_{(\Ub^k_{h,l})^{-1}} > \gamma_{l, \Lplus}\Bigg\}.
    \end{align*}
    The good event is defined by $\cG_1 = \bigcap_{k=1}^{K} \bigcap_{h=1}^{H} \bigcap_{l\geq 1} \bigcap_{V \in \cV_{h, \Lplus}^k} \cB_1^{\complement}(k, h, l, V)$ where we define $\gamma_{l, \Lplus} = 5\Lplus dH\sqrt{\log(16ldH/\delta)} = \tilde\cO(ldH\log(\delta^{-1}))$.
    \end{definition}
\begin{lemma}\label{lm:concentration-V}
The good event $\cG_1$ defined in Definition~\ref{def:good-event} happens with probability at least $1 - 2\delta$. 
\end{lemma}
Lemma~\ref{lm:concentration-V} establishes the concentration bounds for any given phase $l$. However, the total number of phases for the state value function $V_h^k(s)$ can be bounded only trivially by$l = \cO(\log K)$, resulting in $\log K$ dependence. To address this issue, the following lemma proposes a method to eliminate this logarithmic factor:
\begin{lemma}\label{lm:concentration-hatV}
    Under event $\cG_1$, for any $(k, h, l) \in [K] \times [H] \times \NN^+$,
    \begin{align}
        \Bigg \|\sum_{\tau \in \cC^{k-1}_{h,l}} \bphi^\tau_h 
        \big([\PP_h \hat{V}^k_{h+1}](s^\tau_h, a^\tau_h) - \hat{V}^k_{h+1}(s^\tau_{h+1})
        \big)\Bigg\|_{(\Ub^k_{h,l})^{-1}} \leq 1.1 \gamma_l. \label{eq:appendix-1}
    \end{align}
    where we set $\gamma_l = \gamma_{l,\Lplus}$ with $\Lplus = l + 20 + \lceil \log(ld) \rceil$.
\end{lemma}
Then Lemma~\ref{lm:concentration-hatV} immediately yields the following lemma regarding the estimation error of the state-action value function $Q_{h, l}^k$:
\begin{lemma}\label{lm:estimate-error}
Under event $\cG_1$, for any $(k, h, s) \in [K] \times [H] \times \cS, l \in [l^k_h(s) - f^k_h(s)], a_l \in \cA^k_{h,l}(s)$,
    \begin{align}
        \big|Q^{k}_{h,l}(s, a) - [\BB_h \hat{V}^k_{h+1}](s, a) \big| \leq 2 \cdot 2^{-l} +  \chi\sqrt{l}\zeta \label{eq:verify-1}
    \end{align}
where we define $\chi = 12\sqrt{d}H$.
\end{lemma}
Lemma~\ref{lm:estimate-error} build an estimation error for any $l \in [l_h^k(s) - 1]$. As we mentioned in the algorithm design, a larger $l$ here will lead to more precise estimation (a smaller $2^{-l}$ term in~\eqref{eq:verify-1}) but will suffer from a larger covering number (a larger $\gamma_l$ term in~\eqref{eq:verify-1}). Following a similar proof sketch from~\citet{vial2022improved}, the next lemma shows that any action that is not eliminated has a low regret, 
\begin{lemma}\label{lm:all-optimal-pre}
Fix some arbitrary $L_0 \geq 1$ and let $\chi = 12\sqrt{d}H$.
Under event $\cG_1$, for any $(k, h, s) \in [K] \times [H] \times \cS, l \in [\min\{L_0, l^k_h(s) - f^k_h(s)\}], a_{l+1} \in \cA^{k}_{h,l+1}(s)$,
    \begin{align*}
        \max_{a\in \cA}[\BB_h \hat{V}^k_{h+1}](s, a) - [\BB_h \hat{V}^k_{h+1}](s, a_{l+1}) \leq 8 \cdot 2^{-l} + 2l \cdot  \chi \sqrt{L_0} \zeta.
    \end{align*}
\end{lemma}
\subsection{The Impact of Misspecification Level $\zeta$}
Next, we are ready to show the criteria where Line~\ref{ln:cond3} in Algorithm~\ref{alg:Lin} will be triggered, which shows the impact of misspecification on this multi-phased estimation.
\begin{lemma}\label{lm:interplay}
Under event $\cG_1$, for any $(k, h) \in [K] \times [H]$ such that $f_h^k(s_h^k) = 0$, we have $l_h^k(s_h^k) > \Lzeta$ where $\Lzeta$ is the maximal integer satisfying $2^{-\Lzeta} \geq \chi \Lzeta^{1.5} \zeta$ for $\chi = 12\sqrt{d}H$, i.e., $\Lzeta = \Omega(\log (1/\zeta))$.
\end{lemma}
Equipped with Lemma~\ref{lm:interplay}, the following lemma suggests that how much estimation precision $\eps$ can be achieved by accumulating the error $2^{-l_h^k(s_h^k)}$ that occurred in Lemma~\ref{lm:estimate-error}.
\begin{lemma}\label{lm:uncertainty-sum}
Under event $\cG_1$ and for all $\eps > 0$, define $\Leps$ to be the minimal integer satisfying $2^{-\Leps} \le 0.01 \eps / H$, i.e., $\Leps = \lceil -\log( 0.01\eps/H)\rceil$. When $\Leps \le \Lzeta$, then for any $\cK \subseteq [K], h \in [H]$,
    \begin{align*}
        \sum_{k \in \cK} 2^{-l^k_{h}(s^k_{h})} \leq 0.01|\cK| \cdot \eps/H + 2^{12} \Leps d H \gamma_{\Leps}^2 \cdot \eps^{-1}.
    \end{align*}
\end{lemma}
The relationship between $\Leps \le \Lzeta$ can be translated to the relationship between $\eps$ and $\zeta$. We characterize this condition as follows:
\begin{definition}\label{def:interplay-condition}
    Condition $\cG_\eps$ is defined for a given $\eps$, and is satisfied if $\Lzeta \geq \Leps$ where $\Leps$ is the minimal integer satisfying $2^{-\Leps}\le 0.01 \eps / H$ and $\Lzeta$ is the maximal integer satisfying $2^{-\Lzeta} \geq \chi \Lzeta^{1.5} \zeta$. 
\end{definition}
\begin{lemma} \label{lm:interplay-condition}
    If $\eps \ge \Omega\big(\sqrt{d}H^2\zeta \log^2(1/\zeta)\big)$, then $\cG_\eps$ is satisfied.
\end{lemma}
\begin{proof}
    If $\eps \ge \Omega\big(\sqrt{d}H^2\zeta \log^2(1/\zeta)\big)$, we have 
    \begin{align*}
        2^{-\Leps} \ge 0.005 \eps /H \geq 2\chi \Lzeta^{1.5} \zeta \geq 2^{-\Lzeta}.
    \end{align*}
    where the first inequality is given by the definition of $\Leps$, the last inequality is given by the definition of $\Lzeta$, and the second inequality holds since $H \chi \Lzeta^{1.5} \leq \cO\big(\sqrt{d}H^2 \log^2(1/\zeta)\big)$, and the last inequality is given by the definition of $\Leps$ and $\Lzeta$, respectively. Since $2^{-l}$ decreases as $l$ increases, we can conclude that $\Leps \leq \Lzeta$.
\end{proof}

The above analysis of the interplay between misspecification level $\zeta$ and precision $\eps$ yields the following important lemma in our proof, showing a local decision error across all $h \in [H]$:
\begin{lemma}\label{lm:main}
    Under Assumption~\ref{asm:mdp}, let $\gamma_l = 5(l+20+\lceil\log(ld)\rceil)dH\sqrt{\log (16ldH / \delta)}$, for some fixed $0 < \delta <1/3$. With probability at least $1 - 3\delta$, for any $\eps > \Omega\big(\sqrt{d}H^2\zeta\log^2(1/\zeta)\big)$ and $h \in [H]$, we have
    \begin{align*}
    \sum_{k=1}^{\infty} \ind\Big[V_h^*(s^k_h) - V^{\pi^k}_h(s^k_h) \geq \eps \Big] \leq \cO\big(d^3 H^4 \eps^{-2} \log^4(dH\eps^{-1})\log(\delta^{-1}) \iota\big),
    \end{align*}
    where $\iota$ refers to some polynomial of $\log\log(dH\eps^{-1}\delta^{-1})$. This can also be written as 
    \begin{align*}
    \Pr \Big[\exists \eps > \eps_0, h \in [H], \sum_{k=1}^\infty \ind\Big[V_h^*(s_h^k) - V_h^{\pi^k}(s_h^k) > \eps\Big] > f(\eps, \delta)\Big] \leq \delta. \notag 
    \end{align*}
    with $\eps_0 = \tilde \Omega(\sqrt{d}H^2\zeta)$ and $f(\eps, \delta) = \tilde \cO(d^3 H^4 \eps^{-2} \log(\delta^{-1}))$.
\end{lemma}

\subsection{From Local Step-wise Decision Error to Constant Regret}

The next lemma shows that the total incurred suboptimality gap is constant if the minimal suboptimality gap $\Delta$ satisfies $\Delta > \eps_0$.
\begin{lemma} \label{lm:constant-gap-chosen}
    Suppose an RL algorithm \texttt{Alg.} satisfies 
    \begin{align}
    \Pr \Big[\exists \eps > \eps_0, h \in [H], \sum_{k=1}^\infty \ind\Big[V_h^*(s_h^k) - V_h^{\pi^k}(s_h^k) > \eps\Big] > f(\eps, \delta)\Big] \leq \delta, \notag 
    \end{align}
    such that $f(\eps, \delta) = \tilde \cO(C_1/\eps + C_2/\eps^2)$ where $C_1, C_2 > 0$ are constant in $\eps$, but may depend on other quantities such as $d, H, \log(\delta^{-1})$. If the minimal suboptimality gap $\Delta$ satisfies $\Delta > \eps_0$, then 
    \begin{align*}
        \sum_{k=1}^{K}\sum_{h=1}^H \Delta^k_h \leq \tilde \cO(C_2H/\Delta + C_1H)
    \end{align*}
    where $\Delta^k_h = \Delta_h\big(s^k_h, \pi^k_h(s^k_h)\big) = V^*_h(s^k_h) - Q^*_h\big(s^k_h, \pi^k_h(s^k_h)\big)$ is the suboptimality gap suffered in stage $h$ of episode $k$. 
\end{lemma}

The following Lemma is a refined version of Lemma~6.1 in \citet{he2021logarithmic} that removes the dependence between regret and number of episodes $K$.

\begin{lemma} \label{lm:refined-regret-gap}
    For each MDP $\cM(\cS, \cA, H, \{r_h\}, \{\PP_h\})$ and any $\delta > 0$, with probability at least $1 - \delta$, we have 
    $$\textrm{Regret}(K) < \tilde \cO\bigg(\sum_{k=1}^{K} \sum_{h=1}^H \Delta^k_h + H^2\log(1/\delta) \bigg).$$
\end{lemma}

We are now ready to prove Theorem~\ref{thm:main}:
\begin{proof}[Proof of Theorem~\ref{thm:main}]
    By plugging in Lemma~\ref{lm:main} and Lemma~\ref{lm:constant-gap-chosen} into Lemma~\ref{lm:refined-regret-gap}, we can reach the desired statement.
\end{proof}

\section{Proof of Lemmas in Appendix~\ref{app:proof}}\label{app:proof1}
In this section, we prove lemmas outlined in Appendix~\ref{app:proof}. Any proofs not included in this section are deferred to Appendix~\ref{app:add}.
\subsection{Proof of Lemma~\ref{lm:final-value-func-pre}}
\begin{proof}[Proof of Lemma~\ref{lm:final-value-func-pre}]
    According to the criteria for Line~\ref{ln:cond3}, we have $\widecheck{V}^{k}_{h,l}(s) \leq \hat{V}^{k}_{h,l}(s)$ for any $l \in [l^k_h(s)-1]$. From the definition of $\widecheck{V}^{k}_{h,l}(s)$ and $\hat{V}^{k}_{h,l}(s)$, they are monotonic in $l$ that $\hat{V}^{k}_{h,l-1}(s) \leq \hat{V}^{k}_{h,l}(s)$ and $\hat{V}^{k}_{h,l}(s) \leq \hat{V}^{k}_{h,l-1}(s)$ hold. Combining with $\hat{V}^k_{h+1}(s) = \hat{V}^k_{h,l^k_h(s)-1}$, we have  
    \begin{align}
    \forall l \in [l_h^k(s) - 1], \widecheck{V}^{k}_{h,l}(s) \leq \hat{V}^k_{h}(s) \leq \hat{V}^{k}_{h,l}(s)\label{eqq:1}
    \end{align}
    From the definition of $\hat{V}^{k}_{h,l}(s)$ and $\widecheck{V}^{k}_{h,l}(s)$, we have 
    \begin{align}
        0 \le \hat{V}^{k}_{h,l}(s) - \widecheck{V}^{k}_{h,l}(s) \leq \big(\hat{V}^{k}_{h,l}(s) - V^k_{h,l}(s)\big) + \big(V^k_{h,l}(s) - \widecheck{V}^{k}_{h,l}(s)\big) \leq 6 \cdot 2^{-l}.\label{eqq:2}
    \end{align}
    Plugging~\eqref{eqq:1} into~\eqref{eqq:2}, we conclude that for any phase $l \in [l^k_h(s)-1]$, it holds that $|\hat{V}^k_{h}(s) - \hat{V}^{k}_{h,l}(s)| \leq 6 \cdot 2^{-l}$ . 

    Now consider the extended state value function $\hat{V}^k_{h,\Lplus}$ with an arbitrary $\Lplus \in \NN^+$. For every $s$ where $\Lplus \leq l^k_h(s)-1$, we have $|\hat{V}^k_{h}(s) - V^k_{h,\Lplus}(s)| \leq 6 \cdot 2^{-\Lplus}$ as reasoned above. For the other $s \in \cS$ where $\Lplus \geq l^k_h(s)$, we have $\hat{V}^k_{h,l}(s) = \hat{V}^k_{h}(s)$ following the procedure of Algorithm~\ref{alg:Lin}. This suggest that $|\hat{V}^k_{h}(s) - \hat{V}^k_{h,\Lplus}(s)| \leq 6 \cdot 2^{-\Lplus}$ always holds.
\end{proof}
\subsection{Proof of Lemma~\ref{lm:concentration-V}}

The following Lemma shows the rounding only cast bounded effects on the recovered parameters.
\begin{lemma} \label{lm:rounding-error} 
    For any $(k, h, s) \in [K] \times [H] \times \cS, l \in [l^k_h(s) - f^k_h(s)], a \in \cA^k_{h,l}(s)$, it holds that 
    \begin{align*}
        \big|\big\langle\bphi(s, a), \wb^k_{h,l}\big\rangle - \big\langle\bphi(s, a), \tilde \wb^k_{h,l}\big\rangle\big| \leq 0.01 \cdot 2^{-4l}, \big|\|\bphi(s, a)\|_{(\Ub^{k}_{h,l})^{-1}} - \|\bphi(s, a)\|_{\tilde\Ub^{k, -1}_{h,l}}\big| \leq 0.1 \cdot 2^{-2l}.
    \end{align*}
\end{lemma}
The following lemma shows the number of episodes that are taken into regression $|\cC^k_{h,l}|$ is bounded independently from the number of episodes $k$.
\begin{lemma} \label{lm:level-size-bound}
    For any tuple $(k, h, l) \in [K] \times [H] \times \NN^+$, we have $|\cC^k_{h,l}| \leq 16 l \cdot 4^l\gamma_l^{2} d$.
\end{lemma}
The following lemma shows the number of possible state value functions $|\cV^k_{h,l}|$ is bounded independently from the number of episodes $k$.
\begin{lemma} \label{lm:covering-number}
    For any tuple $(k, h, l) \in [K] \times [H] \times \NN^+$, we have $|\cV^k_{h,l}| \leq (2^{22}d^6H^4)^{l^2d^2}$.
\end{lemma}
Now we are ready to prove Lemma~\ref{lm:concentration-V}.
\begin{proof}[Proof of Lemma~\ref{lm:concentration-V}]
    Recall in Definition~\ref{def:good-event}, the good event defined by the union of each single bad event:
    \begin{align*}
        \cG_1 = \bigcap_{k=1}^{K} \bigcap_{h=1}^{H} \bigcap_{l\geq 1} \bigcap_{V \in \cV^k_{h,\Lplus}} \cB_1^{\complement}(k, h, l, V),
    \end{align*}
    where each single bad event is given by 
    \begin{align*}
        \cB_1(k, h,l, V) = \Bigg\{ \Bigg \|\sum_{\tau \in \cC^{k-1}_{h,l}} \bphi^\tau_h 
        \big([\PP_h V](s^\tau_h, a^\tau_h)
         - V(s^\tau_{h+1})\big)\Bigg\|_{(\Ub^k_{h,l})^{-1}} > \gamma_l\Bigg\},
    \end{align*}
    in which $[\PP_h V](s, a) = \Expt_{s' \sim \PP_h(\cdot|s, a)}V(s)$. 

    Consider some fixed $(h, l) \in [H] \times \NN^+$, $V \in \cV^K_{h,\Lplus}$. Arrange elements of $\cC^k_{h,l}$ in ascending order as $\{\tau_i\}_{i}$. 
    Since the environment sample $s^{\tau_i}_{h+1}$ according to $\PP_h(\cdot|s^{\tau_i}_h, a^{\tau_i}_h)$, we have $[\PP_h V](s^{\tau_i}_h, a^{\tau_i}_h) - V(s^{\tau_i}_{h+1})$ is $\cF_{h}^{\tau_{i}}$-measurable with $\Expt\big[[\PP_h V](s^{\tau_i}_h, a^{\tau_i}_h) - V(s^{\tau_i}_{h+1}) \big|\cF_{h}^{\tau_{i}}\big] = 0$. Since $0 \leq V(s^{\tau_i}_{h+1}) \leq H$, we have $|[\PP_h V](s^{\tau_i}_h, a^{\tau_i}_h) - V(s^{\tau_i}_{h+1})| \leq H$. 
    This further leads to 
    \begin{align*}
    & \Bigg \|\sum_{\tau \in \cC^{k-1}_{h,l}} \bphi^\tau_h 
        \big([\PP_h V](s^\tau_h, a^\tau_h)
         - V(s^\tau_{h+1})\big)\Bigg\|_{(\Ub^k_{h,l})^{-1}} \\
         &= \Bigg \|\sum_{i=1}^{|\cC^{k-1}_{h,l}|} \bphi^{\tau_i}_h 
        \big([\PP_h V](s^{\tau_i}_h, a^{\tau_i}_h)
         - V(s^{\tau_i}_{h+1})\big)\Bigg\|_{(\Ub^k_{h,l})^{-1}} \\
         &\leq  H\sqrt{2d \ln\big(1 + |\cC^{k}_{h,l}|/(d\lambda)\big) + 2\ln (l^2 H|\cV^K_{h,\Lplus}| / \delta)} \\
         &\leq H\sqrt{2d \ln(1 + l \cdot 4^l \gamma_l^2) + 2 \ln(l^2 H (2^{22}d^6H^4)^{\Lplus^2d^2}/\delta)} \\
         &\leq \gamma_{l, \Lplus},
    \end{align*}
    where the first inequality holds following from the good event of probability $1 - \delta / (l^2 H|\cV^K_{h,\Lplus}|)$ defined in Lemma~\ref{lm:hoeffding-concentration} over filtration $\{\cF^{\tau_i}_h\}_{i}$, 
    the second inequality is derived from combining Lemma~\ref{lm:level-size-bound} and Lemma~\ref{lm:covering-number}, and the last inequality is given by Lemma~\ref{lm:concentration-V-auxiliary}.
    According to Lemma~\ref{lm:hoeffding-concentration}, we have the bad event $\bigcup_{k=1}^{K} \cB_1(k, h, l, V)$ happens with probability at most $\delta / (l^2 H|\cV^K_{h,\Lplus}|)$. Taking union bound over all $(h, l) \in [H] \times \NN^+$, $V \in \cV^K_{h,\Lplus}$, we have the bad event happens with probability at most 
    \begin{align*}
        \Pr[\cG_1^\complement] \leq  \sum_{h=1}^H \sum_{l=1}^{\infty} \sum_{V \in \cV^K_{h,\Lplus}} \Pr\Big[\bigcup_{k=1}^{K} \cB_1(k, h, l, V) \Big] \leq \sum_{h=1}^H \sum_{l=1}^{\infty} \sum_{V \in \cV^K_{h,\Lplus}} \frac{\delta}{l^2 H|\cV^K_{h,\Lplus}|} \leq 2 \delta,
    \end{align*}
    where the last inequality holds due to $\sum_{n\geq 1} n^{-2} = \pi^2/6$. This completes our proof.
\end{proof}
\subsection{Proof of Lemma~\ref{lm:concentration-hatV}}
\begin{proof}[Proof of Lemma~\ref{lm:concentration-hatV}]
    Denote the martingale difference between $\hat V_{h, \Lplus}^k - \hat V_{h}^k$ as:
    \begin{align*}
        \mu^k_{h,l} = [\PP_h(\hat{V}^k_{h,\Lplus} - \hat{V}^k_{h+1})] (s^k_{h}, \pi^k_{h}(s^k_{h})) - \big(\hat{V}^k_{h,\Lplus}(s^k_{h+1}) -  \hat{V}^k_{h+1}(s^k_{h+1}) \big).
    \end{align*}
    By triangle inequality:
    \begin{align} \label{eq:concentration-hatV-main}
        &\quad\ \Bigg \|\sum_{\tau \in \cC^{k-1}_{h,l}} \bphi^\tau_h 
        \big([\PP_h \hat{V}^k_{h+1}](s^\tau_h, a^\tau_h) - \hat{V}^k_{h+1}(s^\tau_{h+1})
        \big)\Bigg\|_{(\Ub^k_{h,l})^{-1}} \notag \\
        &\leq \Bigg \|\sum_{\tau \in \cC^{k-1}_{h,l}} \bphi^\tau_h 
        \big([\PP_h V^k_{h,\Lplus}](s^\tau_h, a^\tau_h) - \hat{V}^k_{h,\Lplus}(s^\tau_{h+1})
        \big)\Bigg\|_{(\Ub^k_{h,l})^{-1}} + 
        \Bigg \|\sum_{\tau \in \cC^{k-1}_{h,l}} \bphi^\tau_h \mu^\tau_{h,\Lplus}\Bigg\|_{(\Ub^k_{h,l})^{-1}}.
    \end{align}
    According to the definition of event $\cG_1$, we can upper bound the first term by
    \begin{align} \label{eq:concentration-hatV-1}
        \Bigg \|\sum_{\tau \in \cC^{k-1}_{h,l}} \bphi^\tau_h 
        \big([\PP_h V^k_{h,\Lplus}](s^\tau_h, a^\tau_h) - \hat{V}^k_{h,\Lplus}(s^\tau_{h+1})
        \big)\Bigg\|_{(\Ub^k_{h,l})^{-1}} \leq \gamma_{l, \Lplus} = \gamma_l.
    \end{align}
    According to Lemma~\ref{lm:final-value-func-pre}, we have $|\hat{V}^k_{h,\Lplus}(s) -  \hat{V}^k_{h+1}(s)|\leq 6\cdot 2^{-\Lplus}$ for any $s \in \cS$. Thus, the difference can be bounded by $|\mu^\tau_{h,\Lplus}| \leq 6\cdot 2^{-\Lplus}$.
    Consequently, we can bound the second term by
    \begin{align} \label{eq:concentration-hatV-2-pre}
        \Bigg \|\sum_{\tau \in \cC^{k-1}_{h,l}} \bphi^\tau_h 
        \mu^\tau_{h,\Lplus}\Bigg\|_{(\Ub^k_{h,l})^{-1}} &\leq 6\cdot 2^{-\Lplus} \sqrt{|\cC^{k}_{h,l}|} \notag \\ 
        &\leq 6\cdot 2^{-\Lplus} \sqrt{16l \cdot 4^l\gamma_l^2d} \notag \\
        &= 24 \cdot 2^{l-\Lplus} \gamma_l \sqrt{ld},
    \end{align}
    where the first inequality is provided by Lemma~\ref{lm:misspecify}, utilizing the condition $|\mu^\tau_{h,\Lplus}| \leq 6\cdot 2^{-\Lplus }$, 
    the second inequality is from Lemma~\ref{lm:level-size-bound}.
    By plugging in the definition of $\Lplus$, we can further bound the final term of \eqref{eq:concentration-hatV-2-pre} by 
    \begin{align} \label{eq:concentration-hatV-2}
        \Bigg \|\sum_{\tau \in \cC^{k-1}_{h,l}} \bphi^\tau_h 
        \mu^\tau_{h,\Lplus}\Bigg\|_{(\Ub^k_{h,l})^{-1}} &\leq 24 \cdot 2^{l-\Lplus} \gamma_l \sqrt{ld} \leq 24 \cdot 2^{-20} \gamma_l \leq 0.1 \gamma_l.
    \end{align}
    Plugging \eqref{eq:concentration-hatV-1} and \eqref{eq:concentration-hatV-2} into \eqref{eq:concentration-hatV-main} yields the desired statement such that 
    \begin{align*}
        \Bigg \|\sum_{\tau \in \cC^{k-1}_{h,l}} \bphi^\tau_h 
        \big([\PP_h \hat{V}^k_{h+1}](s^\tau_h, a^\tau_h) - \hat{V}^k_{h+1}(s^\tau_{h+1})
        \big)\Bigg\|_{(\Ub^k_{h,l})^{-1}} \leq 1.1 \gamma_l,
    \end{align*}
    which concludes our proof.
\end{proof}
\subsection{Proof of Lemma~\ref{lm:estimate-error}}
The following lemma shows the state-action value function $Q^{k}_{h,l}(s, a)$ is always well estimated.
\begin{lemma} \label{lm:estimate-error-pre}
Under event $\cG_1$, for any $(k, h, l, s, a) \in [K] \times [H] \times \NN^+ \times \cS \times \cA$, 
    \begin{align*}
        \big|Q^{k}_{h,l}(s, a) - [\BB_h \hat{V}^k_{h+1}](s, a) \big| \leq \big(1.2 + 8\sqrt{ld}H \cdot 2^l\zeta\big) \gamma_l \|\bphi(s, a)\|_{(\Ub^{k}_{h,l})^{-1}} + 0.01 \cdot 2^{-4l} + 2H\zeta.
    \end{align*}
\end{lemma}

Equipped with Lemma~\ref{lm:rounding-error} and Lemma~\ref{lm:estimate-error-pre}, we are ready to prove Lemma~\ref{lm:estimate-error}.
\begin{proof}[Proof of Lemma~\ref{lm:estimate-error}]
    In case that $l \leq l^k_h(s) - f^k_h(s)$, for any $a \in \cA^k_{h,l}(s)$, we have that 
    \begin{align} \label{eq:estimate-error-result-tmp1}
        \|\bphi(s, a)\|_{(\Ub^{k}_{h,l})^{-1}} &\leq \|\bphi(s, a)\|_{\tilde\Ub^{k, -1}_{h,l}} + \big|\|\bphi(s, a)\|_{(\Ub^{k}_{h,l})^{-1}} - \|\bphi(s, a)\|_{\tilde\Ub^{k, -1}_{h,l}}\big| \notag \\
        &\leq 2^{-l}\gamma_l^{-1} + 0.1 \cdot 2^{-2l} \leq 1.1 \cdot 2^{-l}\gamma_l^{-1},
    \end{align}
    where the first inequality holds due to triangle inequality, and in the second inequality, the first term is satisfied since state $s$ passes the criterion in Line~\ref{ln:cond2} in phase $l$ and the second term follows from Lemma~\ref{lm:rounding-error}, and the last inequality is given by Lemma~\ref{lm:gamma-monotone} which implies $2^l > \gamma_l$. Plugging \eqref{eq:estimate-error-result-tmp1} into Lemma~\ref{lm:estimate-error-pre} gives 
    \begin{align*}
        \big|Q^{k}_{h,l}(s, a) - [\BB_h \hat{V}^k_{h+1}](s, a) \big| &\leq 0.01 \cdot 2^{-4l} + 1.32 \cdot 2^{-l} + 8.8\sqrt{ld}H\zeta + 2H\zeta \\
        &\leq 2 \cdot 2^{-l} + 12\sqrt{ld}H\zeta,
    \end{align*}
    which proves the desired statement.
\end{proof}
\subsection{Proof of Lemma~\ref{lm:all-optimal-pre}}
Equipped with Lemma~\ref{lm:estimate-error}, we are able to show several properties of the state value function $V^k_{h,l}$ throught the arm-elimination process.
The first lemma suggests that for any action $a_l \in \cA^k_{h,l}(s)$, there is at least one action $a_{l + 1} \in \cA^{k}_{h,l+1}(s)$ close to $a_l$ in terms of the Bellman operator $[\BB_h \hat{V}^k_{h+1}](s, a)$ after the elimination. 
\begin{lemma} \label{lm:exists-optimal-pre}
Under event $\cG_1$, for any $(k, h, s) \in [K] \times [H] \times \cS, l \in [\min\{L_0, l^k_h(s)-f^k_h(s)\}], a_l \in \cA^k_{h,l}(s)$, there exists $a_{l+1} \in \cA^{k}_{h,l+1}(s)$ that 
    \begin{align*}
        [\BB_h \hat{V}^k_{h+1}](s, a_l) - [\BB_h \hat{V}^k_{h+1}](s, a_{l+1}) \leq 2\chi\sqrt{L_0}\zeta
    \end{align*}
    where $\chi = 12\sqrt{d}H$ for arbitrary $L_0 \geq 1$.
\end{lemma}

Then the following lemma shows that by induction on stage $h \in [H]$, we can show the elimination process keep at least one near-optimal action $a_{l + 1} \in \cA^{k}_{h,l+1}(s)$.
\begin{lemma} \label{lm:exists-optimal}
    Under event $\cG_1$, for any $(k, h, s) \in [K] \times [H] \times \cS, l \in [\min\{L_0, l^k_h(s)-f^k_h(s)\}]$, there exists $a_{l + 1} \in \cA^{k}_{h,l+1}(s)$ that,
        \begin{align*}
            \max_{a\in \cA}[\BB_h \hat{V}^k_{h+1}](s, a) - [\BB_h \hat{V}^k_{h+1}](s, a_{l+1}) \leq 2l \cdot \chi\sqrt{L_0}\zeta.
        \end{align*}
    where $\chi = 12\sqrt{d}H$ for arbitrary $L_0 \geq 1$.
\end{lemma}

The following two lemmas indicate that the state value function $V^k_{h,l}(s)$ on stage $h$ is a good estimation for the state value function given by Bellman operator $V(s) = \max_{a\in \cA}[\BB_h \hat{V}^k_{h+1}](s, a)$.
\begin{lemma} \label{lm:value-func-lower}
    Under event $\cG_1$, for any $(k, h, s) \in [K] \times [H] \times \cS, l \in [\min\{L_0, l^k_h(s) - f^k_h(s)\}]$,
    \begin{align*}
        \max_{a\in \cA}[\BB_h \hat{V}^k_{h+1}](s, a) - V^k_{h,l}(s) &\leq 2 \cdot 2^{-l} + (2l-1)\chi\sqrt{L_0}\zeta.
    \end{align*}
    where $\chi = 12\sqrt{d}H$ for arbitrary $L_0 \geq 1$.
\end{lemma}

\begin{lemma} \label{lm:value-func-upper}
    Under event $\cG_1$, for any $(k, h, s) \in [K] \times [H] \times \cS, l \in [\min\{L_0, l^k_h(s) - f^k_h(s)\}]$,
    \begin{align*}
        V^k_{h,l}(s) - \max_{a\in \cA}[\BB_h \hat{V}^k_{h+1}](s, a) &\leq 2 \cdot 2^{-l} + \chi\sqrt{L_0}\zeta,
    \end{align*}
    where $\chi = 12\sqrt{d}H$ for arbitrary $L_0 \geq 1$.
\end{lemma}

Now we are ready to show any actions remaining in the elimination process are near-optimal.
\begin{proof}[Proof of Lemma~\ref{lm:all-optimal-pre}]
First, according to Lemma~\ref{lm:value-func-lower}, we can write
    \begin{align} \label{eq:all-optimal-1}
        \max_{a\in \cA}[\BB_h \hat{V}^k_{h+1}](s, a) - V^k_{h,l}(s) &\leq 2 \cdot 2^{-l} + (2l-1) \chi\sqrt{L_0}\zeta.
    \end{align}
    Any action $a_{l+1} \in \cA^k_{h,l+1}(s)$ passes the elimination process will satisfy:
    \begin{align} \label{eq:all-optimal-2}
        Q^k_{h,l}(s, a_{l+1}) \geq V^k_{h,l}(s) - 4 \cdot 2^{-l}.
    \end{align} 
    According to Lemma~\ref{lm:estimate-error} with the condition that $l \leq L_0$, we have that the empirical state-action value function $Q^k_{h,l}(s, \cdot)$ is a good estimation for $[\BB_h \hat{V}^k_{h+1}](s, \cdot)$ among every $a_{l+1} \in \cA^k_l(s)$ under event $\cG_1$:
    \begin{align} \label{eq:all-optimal-3}
        \big|[\BB_h \hat{V}^k_{h+1}](s, a_{l+1}) - Q^k_{h,l}(s, a_{l+1}) \big| &\leq 2 \cdot 2^{-l} + \chi\sqrt{L_0}\zeta. 
    \end{align}
    Combining \eqref{eq:all-optimal-1}, \eqref{eq:all-optimal-2}, and \eqref{eq:all-optimal-3} gives 
    \begin{align*} 
        &\quad\ \max_{a\in \cA}[\BB_h \hat{V}^k_{h+1}](s, a) - [\BB_h \hat{V}^k_{h+1}](s, a_{l+1}) \notag \\
        &= \big(\max_{a\in \cA}[\BB_h \hat{V}^k_{h+1}](s, a) - V^k_{h,l}(s)\big) +\big(V^k_{h,l}(s) - Q^k_{h,l}(s,a_{l+1}) \big) + \big(Q^k_{h,l}(s,a_{l+1}) - [\BB_h \hat{V}^k_{h+1}](s, a_{l+1})\big) \notag \\
        &\leq \big(2 \cdot 2^{-l} + (2l-1)\chi\sqrt{L_0}\zeta\big) + 4 \cdot 2^{-l} + \big(2 \cdot 2^{-l} + \chi\sqrt{L_0}\zeta \big) \\
        &= 8 \cdot 2^{-l} + 2l\cdot \chi\sqrt{L_0}\zeta,
    \end{align*}
    which proves the desired statement.
\end{proof}

\subsection{Proof of Lemma~\ref{lm:interplay}}

The following two lemmas demonstrate that, at stage $h$, both the optimistic state value function, $\hat{V}^k_{h,l}(s)$, and the pessimistic state value function, $\widecheck{V}^k_{h,l}(s)$, exhibit a gap relative to the state value function determined by the Bellman operator, given as $V(s) = \max_{a\in \cA}[\BB_h \hat{V}^k_{h+1}](s, a)$.

\begin{lemma} \label{lm:opti-value-func-lower}
    Under event $\cG_1$, for any $(k, h, s) \in [K] \times [H] \times \cS, l \in [\min\{L_0, l^k_h(s) - f^k_h(s)\}]$,
    \begin{align*}
         \min\big\{V^k_{h,l}(s) + 3 \cdot 2^{-l}, \hat{V}^{k}_{h,l-1}(s)\big\} - \max_{a\in \cA}[\BB_h \hat{V}^k_{h+1}](s, a) & \geq 2^{-l} - (2l-1)\chi\sqrt{L_0}\zeta,
    \end{align*}
    where $\chi = 12\sqrt{d}H$ for arbitrary $L_0 \geq 1$. In case that $l \leq l^k_h(s) - 1$, the inequality is equivalent to 
    \begin{align*}
         \hat{V}^k_{h,l}(s) - \max_{a\in \cA}[\BB_h \hat{V}^k_{h+1}](s, a) & \geq 2^{-l} - (2l-1)\chi\sqrt{L_0}\zeta.
    \end{align*} 
\end{lemma}

\begin{lemma} \label{lm:pess-value-func-upper}
    Under event $\cG_1$, for any $(k, h, s) \in [K] \times [H] \times \cS, l \in [\min\{L_0, l^k_h(s) - f^k_h(s)\}]$,
    \begin{align*}
         \max_{a\in \cA}[\BB_h \hat{V}^k_{h+1}](s, a) - \max\big\{V^k_{h,l}(s) - 3 \cdot 2^{-l}, \widecheck{V}^{k}_{h,l-1}(s)\big\} &\geq 2^{-l} - \chi\sqrt{L_0}\zeta,
    \end{align*}
    where $\chi = 12\sqrt{d}H$ for arbitrary $L_0 \geq 1$. In case that $l \leq l^k_h(s) - 1$, the inequality is equivalent to 
    \begin{align*}
         \max_{a\in \cA}[\BB_h \hat{V}^k_{h+1}](s, a) - \widecheck{V}^{k}_{h,l}(s) &\geq 2^{-l} - \chi\sqrt{L_0}\zeta.
    \end{align*} 
\end{lemma}

\begin{proof}[Proof of Lemma~\ref{lm:interplay}]
    Set $L_0 = \Lzeta$ be the maximal integer satisfying $2^{-\Lzeta} - \chi \Lzeta^{1.5}\zeta \geq 0$.
    Combining Lemma~\ref{lm:pess-value-func-upper} and Lemma~\ref{lm:opti-value-func-lower}, for any $l \in [\min\{L_0, l^k_h(s) - f^k_h(s)\}]$, we have that 
    \begin{align*}
        &\quad \min\big\{V^k_{h,l}(s) + 3 \cdot 2^{-l}, \hat{V}^{k}_{h,l-1}(s)\big\} - \max\big\{V^k_{h,l}(s) - 3 \cdot 2^{-l}, \widecheck{V}^{k}_{h,l-1}(s)\big\} \\
        &= \big( \hat{V}^k_{h,l}(s) - \max_{a\in \cA}[\BB_h \hat{V}^k_{h+1}](s, a)\big) + \big( \max_{a\in \cA}[\BB_h \hat{V}^k_{h+1}](s, a) - \widecheck{V}^{k}_{h,l}(s)\big) \\
        &\geq \big(2^{-l}  - (2l-1)\chi\sqrt{L_0}\zeta\big) + \big(2^{-l} - \chi\sqrt{L_0}\zeta\big)  \\
        &= 2\cdot 2^{-l} - 2l \cdot \chi\sqrt{L_0}\zeta \\
        &\geq 2\cdot 2^{-L_0} - 2\chi L_0^{1.5}\zeta \geq 0.
    \end{align*}
    where the second inequality holds since $2^{-l}$ decreases as $l$ increases and the last inequality holds according to the selection of $L_0$.

    When $f^k_h(s) = 0$, consider $l = l^k_h(s)$. The above reasoning indicates 
    the criterion in Line~\ref{ln:cond3} can never satisfied. Thus $f^k_h(s) = 0$ can only happen if $l^k_h(s) > L_0 = \Lzeta$.
\end{proof}
\subsection{Proof of Lemma~\ref{lm:uncertainty-sum}}

By partitioning $[K]$ based on whether Algorithm~\ref{alg:Lin} stops before phase $\Leps$, we can prove Lemma~\ref{lm:uncertainty-sum}. Specifically, Lemma~\ref{lm:level-size-bound} bounds the number of episodes in which Algorithm~\ref{alg:Lin} stops before phase $\Leps$. This allows us to establish an upper bound for the desired summation over these episodes. Furthermore, for episodes that stop after phase $\Leps$, the contribution of $2^{-l^k_{h}(s^k_{h})} \gamma_{l^k_{h}(s^k_{h})}$ is small according to the definition of $\Leps$.

\begin{proof}[Proof of Lemma~\ref{lm:uncertainty-sum}]
    Denote $\cC^K_{h,+} = [K] - \bigcup_{l=1}^{\Leps-1} \cC^K_{h,l}$. In this sense, we have
    \begin{align} \label{eq:uncertainty-sum-main}
        \sum_{k \in \cK} 2^{-l^k_{h}(s^k_{h})} &= \sum_{k \in \cK \cap \cC^K_{h,+}} 2^{-l^k_{h}(s^k_{h})}  + \sum_{l=1}^{\Leps-1} \sum_{k \in \cK \cap \cC^K_{h,l}} 2^{-l^k_{h}(s^k_{h})}.
    \end{align}

    From the construction of $\cC^K_{h,l}$, we have $l^k_h(s^k_h) = l$ for any $k \in \cC^K_{h,l}$. 
    Fix some $k \in \cC^K_{h,+}$. If $f^k_h(s^k_h) = 0$, we have $l^k_h(s^k_h) \geq \Lzeta \geq \Leps$ where the first inequality is given by Lemma~\ref{lm:interplay} and the second inequality is given by the assignment of $\Leps$. Otherwise, we have $l^k_h(s^k_h) \geq \Leps$ according to the definition of $\cC^K_{h,l}$. 
    This indicates $l^k_h(s^k_h) \geq \Leps$ holds for any $k \in \cC^K_{h,+}$. This allow is to bound the first term by 
    \begin{align} \label{eq:uncertainty-sum-1}
        \sum_{k \in \cK \cap \cC^K_{h,+}} 2^{-l^k_{h}(s^k_{h})} \leq \sum_{k \in \cK \cap \cC^K_{h,+}}  2^{-\Leps} \leq 0.01|\cK| \cdot \eps/H,
    \end{align}
    where the first inequality holds since $l^k_h(s^k_h) > \Leps$ and the second inequality holds from both $2^{-\Leps} \leq 0.01\eps/H$ and $|\cK \cap \cC^K_{h,+}| \leq |\cK|$.

    Furthermore, we can bound the second term by 
    \begin{align} \label{eq:uncertainty-sum-2}
        \sum_{l=1}^{\Leps - 1} \sum_{k \in \cK \cap \cC^K_{h,l}} 2^{-l^k_{h}(s^k_{h})} &\leq \sum_{l=1}^{\Leps - 1} |\cK \cap \cC^K_{h,l}| \cdot 2^{-l} \notag \\
        &\leq \sum_{l=1}^{\Leps-1} 16l \cdot 4^l\gamma_l^2 d \cdot 2^{-l}  \notag \\
        & \leq 16\Leps d \cdot 2^{\Leps}\gamma_{\Leps}^2 
        \leq 2^{12}\Leps d H \gamma_{\Leps}^2\eps^{-1}.
    \end{align}
    where the second inequality is given by Lemma~\ref{lm:level-size-bound}, and the last inequality holds due to $0.005\eps/H \leq 2^{-\Leps}$ which is because $\Leps$ is a minimal integer such that $2^{-\Leps} \leq 0.01\eps/H$.
    
    Finally, plugging \eqref{eq:uncertainty-sum-1} and \eqref{eq:uncertainty-sum-2} into \eqref{eq:uncertainty-sum-main} gives
    \begin{align*}
        \sum_{k \in \cK} 2^{-l^k_{h}(s^k_{h})} \leq 0.01|\cK| \cdot \eps/H + 2^{12} \Leps d H \gamma_{\Leps}^2\eps^{-1}.
    \end{align*}
\end{proof}
\subsection{Proof of Lemma~\ref{lm:main}}
The following lemma provides an upper bound for the underestimation error of the empirical state value function $\hat{V}^k_h$ with respect to the optimal state value function $V^*_h$.
\begin{lemma} \label{lm:final-value-func-lower}
Under event $\cG_1$ and for all $\eps > 0$ that $\cG_\eps$ is satisfied, for any $(k, h, s) \in [K] \times [H] \times \cS$,
    \begin{align*}
        V^*_h(s) - \hat{V}^k_h(s) \leq 0.07\eps.
    \end{align*}
\end{lemma}

As $\hat{V}^k_{h}$ represents an empirical state value function with potentially optimal policy $\pi^{k}_h(s)$, the following lemma provides an upper bound for the overestimation error of $\hat{V}^k_{h}$ with respect to deploying the policy $\pi^{k}_h(s)$ on the ground-truth transition kernel.
\begin{lemma} \label{lm:local-overestimation}
    Under event $\cG_1$ and for all $\eps > 0$ that $\cG_\eps$ is satisfied, for any $(k, h, s) \in [K] \times [H] \times \cS$, 
    \begin{align*}
        \hat{V}^k_h(s) - [\BB_h \hat{V}^k_{h+1}](s, \pi^{k}_h(s)) \leq 20 \cdot 2^{-l^k_h(s)} + 0.16\eps/H.
    \end{align*}
\end{lemma}

To start with, we define a good event according to:
\begin{definition}\label{def:good-event2}
    For some $\eps > 0$, let $\Kepsh = \{k \in [K]: V^*_h(s^k_h) - V^{\pi^k}_h(s^k_h) \geq \eps\}$.
    We define the bad event as
    \begin{align*}
        \cB_2(h, \eps) = \Bigg\{ \sum_{k \in \Kepsh}\sum_{h'=h}^{H} \eta^k_{h'} > 4\sqrt{H^3|\Kepsh|\log (4H|\Kepsh|\log(\eps^{-1})/\delta)}\Bigg\}.
    \end{align*}
    where $\eta^k_h = [\PP_h(\hat{V}^k_{h+1} - V^{\pi^k}_{h+1})] (s^k_{h}, \pi^k_{h}(s^k_{h})) - \big(\hat{V}^k_{h+1}(s^k_{h+1}) -  V^{\pi^k}_{h+1}(s^k_{h+1}) \big).$
    The good event is defined as $\cG_2 = \bigcap_{h=1}^H\bigcap_{l\geq 1} \cB_2^{\complement}(h,2^{-l})$.
\end{definition}
The following lemma provides the concentration property such that the cumulative sample error is small with high probability.
\begin{lemma} \label{lm:extra-concentration}
    Event $\cG_2$ happens with probability at least $1 - \delta$. 
\end{lemma}

Using the above results, we can bound the instantaneous regret of any subsets once the misspecification level is appropriately controlled,
\begin{lemma} \label{lm:one-step-regret-final}
    Under event $\cG_1, \cG_2$ and for all $\eps > 0$ that $\cG_\eps$ is satisfied, for any $\cK \subseteq [K]$ and $h \in [H]$, it satisfies that 
    \begin{align*}
        \sum_{k \in \cK}\big(V^*_h(s^k_h) - V^{\pi^k}_h(s^k_h) \big) \leq 0.49|\cK| \eps + 2^{17}\Leps d H^2 \gamma_{\Leps}^2 \eps^{-1} + 4\sqrt{H^3|\cK|\log (4H|\cK|\log(\eps^{-1})/\delta)}.
    \end{align*}
\end{lemma}
With all lemmas stated above, we can show \mainalg achieves constant step-wise decision error.
The following lemma gives a sufficient condition that $\cG_\eps$ defined in Definition~\ref{def:interplay-condition} is satisfied.

Now, we are ready to prove Lemma~\ref{lm:main}.

\begin{proof}[Proof of Lemma~\ref{lm:main}]
    We focus on the case where the good event $\cG_1 \cap \cG_2 \cap \cG_\eps$ occurs. By the union bound statement over Lemma~\ref{lm:concentration-V} and Lemma~\ref{lm:extra-concentration}, and Lemma~\ref{lm:interplay-condition}, this good event happens with a probability of at least $1-3\delta$ and requires $\eps \geq \Omega(\zeta \sqrt{d}H^2 \log^2(dH\zeta^{-1}))$.
    W.l.o.g, consider $\Kepsh$ for some $h\in[H]$ and $\eps = 2^{-l}$ where $l>0$ is an integer.
    On the one hand, we have 
    \begin{align} \label{eq:theorem-lower}
        \sum_{k \in \Kepsh}\big(V^*_h(s^k_h) - V^{\pi^k}_h(s^k_h) \big) \geq |\Kepsh| \eps.
    \end{align}
    On the other hand, Lemma~\ref{lm:one-step-regret-final} gives 
    \begin{align} \label{eq:theorem-upper}
        \sum_{k \in \Kepsh}\big(V^*_h(s^k_h) - V^{\pi^k}_h(s^k_h) \big) &\leq 0.49|\Kepsh| \eps + 2^{17}\Leps d H^2 \gamma_{\Leps}^2 \eps^{-1} \notag \\
        &\quad + 4\sqrt{H^3|\Kepsh|\log (4H|\Kepsh|\log(\eps^{-1})/\delta)}.
    \end{align}
    Combining \eqref{eq:theorem-lower} and \eqref{eq:theorem-upper} gives
    \begin{align*}
        0.51|\Kepsh| \eps  &\leq 2^{17} \Leps d H^2 \gamma_{\Leps}^2 \eps^{-1} + 4\sqrt{H^3|\Kepsh|\log (4H|\Kepsh|\log(\eps^{-1})/\delta)}.
    \end{align*}
    Plugging the value of $\gamma_{\Leps}$, we have
    \begin{align}
        0.51|\Kepsh| \eps &\leq 2^{22}\Leps(\Leps+\log(2^{20} dH))^2 d^3 H^4 \eps^{-1} \log(16\Leps d/\delta) \notag \\
        &\quad + 4\sqrt{H^3|\Kepsh|\log (4H|\Kepsh|\log(\eps^{-1})/\delta)}.\label{eqq:3}
    \end{align}
    According to Lemma~\ref{lm:main-thm-auilixary-1}, \eqref{eqq:3} implies 
    \begin{align*}
        |\Kepsh| \leq \cO(\Leps(\Leps+\log(dH))^2 d^3 H^4 \eps^{-2} \log(\Leps d) \log(\delta^{-1})\iota),
    \end{align*}
    where $\iota$ refers to a polynomial of $\log\log(dH\eps^{-1}\delta^{-1})$.
    With the definition of $\Leps$, we conclude that 
    \begin{align*}
        |\Kepsh| \leq \cO(d^3 H^4 \eps^{-2} \log^4(dH\eps^{-1})\log(\delta^{-1}) \iota).
    \end{align*}
\end{proof}
\subsection{Proof of Lemma~\ref{lm:constant-gap-chosen}}
\begin{proof}[Proof of Lemma~\ref{lm:constant-gap-chosen}]
    From the definition of suboptimality gap, we have 
    \begin{align} \label{eq:thm:gap-pac:gap-transfer}
        \Delta^k_h &= V^*_h(s^k_h) - [\BB_h V^*_{h+1}](s^k_h, \pi^k_h(s^k_h)) \notag \\
        &\leq V^*_h(s^k_h) - [\BB_h V^{\pi^k}_{h+1}](s^k_h, \pi^k_h(s^k_h)) \notag \\
        &= V^*_h(s^k_h) - V^{\pi^k}_{h}(s^k_h). 
    \end{align}
    According to the assumption,
    \begin{align*}
        \sum_{k=1}^{K}\ind\Big[ V_h^*(s_1^k) - V_h^{\pi^k}(s_h^k) \geq \eps\Big] \leq \Big(\frac{C_1}{\eps} + \frac{C_2}{\eps^2} \Big) \log^a \Big(\frac{C_1}{\eps} + \frac{C_2}{\eps^2} \Big)
    \end{align*}
    holds for every $\eps>\eps_0$ with probability at least $1 - \delta$, replacing the $V_h^*(s_1^k) - V_h^{\pi^k}(s_h^k)$ with its lower bound $\Delta^k_h$ yields for every $\eps > \eps_0$,
    \begin{align*}
        \sum_{k=1}^{K}\ind\Big[ \Delta^k_h \geq \eps\Big] \leq \Big(\frac{C_1}{\eps} + \frac{C_2}{\eps^2} \Big) \log^a \Big(\frac{C_1}{\eps} + \frac{C_2}{\eps^2} \Big).
    \end{align*}
    In addition, according to the definition of minimal suboptimality gap $\Delta$ in Definition~\ref{def:gap}, we have $\Delta^k_h$ is either $0$ or no less than $\Delta$. Since for any $x \in \{0\} \cup [\Delta, H]$, it holds that $x \leq \Delta\cdot \ind[x \geq \Delta] + \int_{\Delta}^{H} \ind[x \geq \eps] \ud \eps$, we decompose the total suboptimality incurred in stage $h$ by 
        \begin{align}
        \sum_{k=1}^{K} \Delta^k_h &\leq \sum_{k=1}^{K}\Bigg(\Delta \cdot \ind\Big[ \Delta^k_h \geq \Delta\Big] + \int_{\Delta}^{H}\ind\Big[ \Delta^k_h \geq \eps\Big] \ud \eps\Bigg) \notag\\
        &= \Delta\sum_{k=1}^{K}\ind\Big[ \Delta^k_h \geq \Delta\Big] + \int_{\Delta}^{H} \sum_{k=1}^{K}\ind\Big[ \Delta^k_h \geq \eps\Big] \ud \eps. \label{eq:weitong-1}
    \end{align}
    In case that $\eps_0 \leq \Delta$, the first term in~\eqref{eq:weitong-1} can be bounded by 
    \begin{align}  \label{eq:lm:constant-gap-chosen-term1}
        \Delta\sum_{k=1}^{K}\ind\Big[ \Delta^k_h \geq \Delta\Big] \leq \Delta \Big(\frac{C_1}{\Delta} + \frac{C_2}{\Delta^2} \Big) \log^a \Big(\frac{C_1}{\Delta} + \frac{C_2}{\Delta^2} \Big).
    \end{align}
    We can further bound the second term by 
    \begin{align} \label{eq:lm:constant-gap-chosen-term2}
        \int_{\Delta}^{H} \sum_{k=1}^{K}\ind\Big[ V_1^*(s_1^k) - V_1^{\pi^k}(s_1^k) \geq \eps\Big] \ud \eps &\leq \int_{\Delta}^{H} \Big(\frac{C_1}{\eps} + \frac{C_2}{\eps^2} \Big) \log^a \Big(\frac{C_1}{\eps} + \frac{C_2}{\eps^2} \Big) \ud \eps \notag \\
        &\leq \log^a \Big(\frac{C_1}{\Delta} + \frac{C_2}{\Delta^2} \Big) \cdot \Big(C_1\ln\frac{H}{\Delta} + \frac{C_2}{\Delta} \Big) \notag \\
        &\leq (C_1\log H + C_2/\Delta)\cdot\polylog(C_1, C_2, \Delta^{-1}).
    \end{align}
    Plugging \eqref{eq:lm:constant-gap-chosen-term1} and \eqref{eq:lm:constant-gap-chosen-term2} into \eqref{eq:weitong-1} with summation over $h \in [H]$, we conclude that the total suboptimality incurred in stage $h$ is bounded by 
    \begin{align*}
         \sum_{k=1}^{K} \sum_{h=1}^H \Delta^k_h &\leq H \cdot (C_1 + C_2/\Delta + C_1\log H +  C_2/\Delta)\cdot\polylog(C_1, C_2, \Delta^{-1})\\
        &\leq \tilde \cO(C_2H/\Delta + C_1H).
    \end{align*}
\end{proof}

\subsection{Proof of Lemma~\ref{lm:refined-regret-gap}}
\begin{proof}[Proof of Lemma~\ref{lm:refined-regret-gap}]
    For a given policy $\pi$ and any state $s_h \in \cS$, we have 
    \begin{align*}
        &V^*_h(s_h) - V^{\pi}_h(s_h) \\
        &= \big( V^*_h(s_h) - [\BB_hV^*_{h+1}](s_h, \pi_h(s_h)) \big) + \big([\BB_hV^*_{h+1}](s_h, \pi_h(s_h)) - [\BB_h V^{\pi}_{h+1}](s_h, \pi_h(s_h))\big) \\
        &= \Delta_h(s_h, \pi_h(s_h)) + [\PP_h(V^*_{h+1} - V^{\pi}_{h+1})](s_h, \pi_h(s_h)),
    \end{align*}
    where the second equality is given by the definition of suboptimality gap $\Delta_h(\cdot, \cdot)$ in Definition~\ref{def:gap}. Taking expectation on both sides with respect to the randomness of state-transition and taking telescoping sum over all $h \in [H]$ gives 
    \begin{align*}
        V^*_1(s_1) - V^{\pi}_h(s_1) &= \Expt \bigg[\sum_{h=1}^H \Delta_h(s_h, \pi_h(s_h))\bigg],
    \end{align*}
    where $s_{h+1} \sim \PP_h(\cdot|s_h, \pi_h(s_h))$. Let the filtration list be $\cF^k = \Big\{\big\{s_i^j, a_i^j\big\}_{i=1, j=1}^{H, k - 1}\Big\}$, we have 
    \begin{align*}
        \Expt \Big[\sum_{h=1}^H \Delta^k_h \Big| \cF^k\Big] = V^*_1(s^k_1) - V^{\pi^k}_h(s^k_1).
    \end{align*}
    Denote random variable $\eta^k = \big(V^*_1(s^k_1) - V^{\pi^k}_h(s^k_1) \big) - \sum_{h=1}^H \Delta^k_h$. We can see $\eta^k$ is $\cF_{k+1}$-measurable with $|\Expt[\eta^k|\cF^k]| = 0$. Furthermore, for the variance of $\eta^k$, we have 
    \begin{align*}
        \Var[\eta^k|\cF^k] &\leq \Expt \Big[ \Big(\sum_{h=1}^H \Delta^k_h \Big)^2 \Big| \cF^k\Big] \\
        &\leq H^2 \Expt\Big[ \sum_{h=1}^H \Delta^k_h \Big| \cF^k\Big] \\
        &= H^2 \big(V^*_1(s^k_1) - V^{\pi^k}_h(s^k_1) \big),
    \end{align*}
    where the first inequality holds due to $\Var[X] \leq \Expt[(X-t)^2]$ for any fixed $t$, the second inequality follows $0 \leq \Delta^k_h \leq H$. As a result, the total variance of the random variables $\{\eta^k\}$ can be bounded by 
    \begin{align*}
        \sum_{k=1}^{K}\Var[\eta^k|\cF^k] \leq \sum_{k=1}^{K}H^2 \big(V^*_1(s^k_1) - V^{\pi^k}_h(s^k_1) \big) = H^2 \textrm{Regret}(K).
    \end{align*}

    Let $F(x) = \sqrt{2xH^2\log(x / \delta)} + H^2 \log(x / \delta)$, 
    using peeling technique, we can write
    \begin{align} \label{eq:lm:refined-regret-gap-peeling}
        &\Pr\Big[\Big(\sum_{k=1}^K \eta^k\Big) \geq F(\textrm{Regret}(K)), 1 < \textrm{Regret}(K) \Big] \notag \\
        &= \Pr\Big[\Big(\sum_{k=1}^K \eta^k\Big) \geq F(\textrm{Regret}(K)), 1 < \textrm{Regret}(K), \sum_{k=1}^{K}\Var[\eta^k|\cF^k] \leq H^2\textrm{Regret}(K) \Big] \notag \\
        &\leq \sum_{i=1}^{\infty} \Pr\Big[\Big(\sum_{k=1}^K \eta^k\Big) \geq F(\textrm{Regret}(K)), 2^{i-1} < \textrm{Regret}(K) \leq 2^i, \sum_{k=1}^{K}\Var[\eta^k|\cF^k] \leq H^2\textrm{Regret}(K) \Big] \notag\\
        &\leq \sum_{i=1}^{\infty} \Pr\Big[\Big(\sum_{k=1}^K \eta^k\Big) \geq F(2^i), \sum_{k=1}^{K}\Var[\eta^k|\cF^k] \leq 2^iH^2 \Big] \notag \\
        &\leq \sum_{i=1}^{\infty} \exp\Big(\frac{-F(2^i)^2}{2^{i+1}H^2 + 2F(2^i)H^2/3}\Big),
    \end{align}
    where the last inequality follows Lemma~\ref{lm:freedman}.
    Plugging $F(x) = \sqrt{2xH^2\log(x/\delta)} + H^2\log(x/\delta)$ back into \eqref{eq:lm:refined-regret-gap-peeling} yields
    \begin{align*}
        &\Pr\Big[\Big(\sum_{k=1}^K \eta^k\Big) \geq \sqrt{2\textrm{Regret}(K)H^2\log(\textrm{Regret}(K)/\delta)} + H^2\log(\textrm{Regret}(K)/\delta), 1 < \textrm{Regret}(K) \Big] \\
        &\leq \sum_{i=1}^{\infty} \exp(-\log(2^{i}/\delta)) = \sum_{i=1}^{\infty} \delta/2^i = \delta.
    \end{align*}
    Therefore, whenever $\textrm{Regret}(K) > 1$, with probability at least $1-\delta$, we have 
    \begin{align*}
        \sum_{k=1}^K \eta^k < \sqrt{2\textrm{Regret}(K)H^2\log(\textrm{Regret}(K)/\delta)} + H^2\log(\textrm{Regret}(K)/\delta).
    \end{align*}
    Combining with the fact that $\textrm{Regret}(K) = \sum_{k=1}^K \eta^k +  \sum_{k=1}^{K} \sum_{h=1}^H \Delta^k_h$, we have 
    \begin{align*}
        \textrm{Regret}(K) < \sum_{k=1}^{K} \sum_{h=1}^H \Delta^k_h + \sqrt{2\textrm{Regret}(K)H^2\log(\textrm{Regret}(K)/\delta)} + H^2\log(\textrm{Regret}(K)/\delta),
    \end{align*}
    whenever $\textrm{Regret}(K) > 1$.
    Taking $x = \textrm{Regret}(K)$, $a = \sum_{k=1}^{K} \sum_{h=1}^H \Delta^k_h + H^2\log(\textrm{Regret}(K)/\delta)$, and $b = 2H^2\log(\textrm{Regret}(K)/\delta)$, inequality (1) yields 
    \begin{align}
        \textrm{Regret}(K) &\le 2\sum_{k=1}^{K} \sum_{h=1}^H \Delta^k_h + 2H^2\log(\textrm{Regret}(K)/\delta) + 4H^2\log(\textrm{Regret}(K)/\delta) \notag  \\
        &= 2\sum_{k=1}^{K} \sum_{h=1}^H \Delta^k_h + 6H^2\log(1 / \delta) + 6H^2\log(\textrm{Regret}(K)) \quad \label{eq:rev-1}
    \end{align}

According to the fact that $x \le a \log x + b \Rightarrow x \le 4a \log (2a) + 2b$, letting $x = \textrm{Regret}(K), a = 2\sum_{k=1}^{K} \sum_{h=1}^H \Delta^k_h + 6H^2\log(1 / \delta)$ and $b = 6H^2$,~\eqref{eq:rev-1} becomes \begin{align}
    \textrm{Regret}(K) \le \left(8\sum_{k=1}^{K} \sum_{h=1}^H \Delta^k_h + 24H^2\log(1 / \delta)\right)\log\left(4\sum_{k=1}^{K} \sum_{h=1}^H \Delta^k_h + 12H^2\log(1 / \delta)\right) + 12H^2. \notag 
\end{align} 
Hiding the logarithmic factors within the $\tilde O$ notation yields
\begin{align*}
    \textrm{Regret}(K) < \tilde \cO\Big(\sum_{k=1}^{K} \sum_{h=1}^H \Delta^k_h + H^2\log(1/\delta) \Big).
\end{align*}
\end{proof}

\section{Proof of Lemmas in Appendix~\ref{app:proof1}}\label{app:add}
In this section, we prove lemmas outlined in Appendix~\ref{app:proof1}. Any proofs not included in this section are deferred to Appendix~\ref{app:add2}.
\subsection{Proof of Lemma~\ref{lm:rounding-error}}
We first introduce the claim from~\citet{vial2022improved} controlling the rounding error:
\begin{lemma}[Claim~1,~\citealt{vial2022improved}, restate]\label{lm:rounding-err}
    For any $(k, h, l, s, a) \in [K] \times [H] \times \NN^+ \times \cS \times \cA$, we have
    \begin{align*}
        \bphi(s, a)^\top(\wb_{h, l}^k - \tilde \wb_{h, l}^k) \le \sqrt{d}\kappa_l, \hfill \big|\|\bphi(s, a)_{(\Ub_{h, l}^k)^{-1}} - \|\bphi(s, a)\|_{\tilde \Ub_{h, l}^{k, -1}}\big| \le \sqrt{d\kappa_l},
    \end{align*}
    where $\kappa_l$ is used to quantify the vector $\wb_{h, l}^k$ and inverse matrix $(\Ub_{h, l}^l)^{-1}$.
\end{lemma}
\begin{proof}[Proof of Lemma~\ref{lm:rounding-error}]
    From Lemma~\ref{lm:rounding-err} we have 
    \begin{align*}
        \big|\big\langle\bphi(s, a), \wb^k_{h,l}\big\rangle - \big\langle\bphi(s, a), \tilde \wb^k_{h,l}\big\rangle\big| \leq \sqrt{d}\epsrnd_{l} \leq 0.01 \cdot 2^{-4l}
        \end{align*}
        where the first inequality is due to Lemma~\ref{lm:rounding-err}, and the second inequality is valid due to $\epsrnd_{l} = 0.01 \cdot 2^{-4l}$.
        Similarly, we have
        \begin{align*}
        \big|\|\bphi(s, a)\|_{(\Ub^{k}_{h,l})^{-1}} - \|\bphi(s, a)\|_{\tilde\Ub^{k, -1}_{h,l}}\big| \leq \sqrt{d\epsrnd_{l}} \leq 0.1 \cdot 2^{-2l}.
    \end{align*}
\end{proof}
\subsection{Proof of Lemma~\ref{lm:level-size-bound}}
\begin{proof}[Proof of Lemma~\ref{lm:level-size-bound}]
    First, both $l^\tau_h(s^\tau_h) = l$ and $f^\tau_h(s^\tau_h) = 1$ held for any $\tau \in \cC^k_{h,l}$. This implies that the criteria for either Line~\ref{ln:cond1} or Line~\ref{ln:cond2} holds as $l = l^\tau_h(s^\tau_h)$. For $\tau$ that satisfies the first criterion, we have $l^\tau_h(s^\tau_h) > L_\tau$. Note that $L_\tau = \max\{\lceil\log_4(\tau/d)\rceil, 0\}$, so this only happens for $\tau < 4^ld$. For other $\tau$ that satisfies the second criterion, we have that 
    \begin{align*}
   \|\bphi^\tau_h\|_{(\Ub^{\tau}_{h,l})^{-1}} &\geq \|\bphi^\tau_h\|_{\tilde\Ub^{\tau, -1}_{h,l}} - \big|\|\bphi^\tau_h\|_{\tilde\Ub^{\tau, -1}_{h,l}} -  \|\bphi^\tau_h\|_{(\Ub^{\tau}_{h,l})^{-1}}\big| \geq 2^{-l}\gamma_l^{-1} - 0.1 \cdot 2^{-l} \gamma_l^{-1} = 0.9 \cdot 2^{-l}\gamma_l^{-1},
    \end{align*}
    where the first inequality holds due to the triangle inequality. In the second inequality, the first term $\|\bphi^\tau_h\|_{\tilde\Ub^{\tau, -1}_{h,l}}$ is bounded by criterion in Line~\ref{ln:cond2} while the second term $\big|\|\bphi^\tau_h\|_{\tilde\Ub^{\tau, -1}_{h,l}} -  \|\bphi^\tau_h\|_{(\Ub^{\tau}_{h,l})^{-1}}\big|$ follows from Lemma~\ref{lm:rounding-error}.

    Arrange elements of $\cC^k_{h,l}$ in ascending order as $\{\tau_i\}_{i}$. According to the above reasoning, the number of elements $\tau\in\cC^k_{h,l}$ that $\|\bphi^\tau_h\|_{(\Ub^{\tau}_{h,l})^{-1}} \geq 0.9 \cdot 2^{-l}\gamma_l^{-1}$ is at least $|\cC^k_{h,l}|-4^ld$. This gives
    \begin{align} \label{eq:level-size-bound-1}
        \sum_{i=1}^{|\cC^k_{h,l}|} \min\{1, \|\bphi^\tau_h\|_{(\Ub^{\tau}_{h,l})^{-1}}^2\} \geq (0.9 \cdot 2^{-l}\gamma_l^{-1})^2 \cdot (|\cC^k_{h,l}|-4^ld).
    \end{align}
    On the other hand, Lemma~\ref{lm:vector-potential} upper bounds the LHS of~\eqref{eq:level-size-bound-1} by  
    \begin{align} \label{eq:level-size-bound-2}
        \sum_{i=1}^{|\cC^k_{h,l}|} \min\{1, \|\bphi^\tau_h\|_{(\Ub^{\tau}_{h,l})^{-1}}^2\} \leq 2d\ln\big(1 + |\cC^k_{h,l}|/(d\lambda)\big).
    \end{align}
    Combining \eqref{eq:level-size-bound-1} and \eqref{eq:level-size-bound-2} gives 
    \begin{align} \label{eq:level-size-bound-main}
        0.81 \cdot 4^{-l}\gamma_l^{-2} (|\cC^k_{h,l}| - 4^ld) \leq 2d\ln\big(1 + |\cC^k_{h,l}|/(16d)\big).
    \end{align}
    From algebra analysis in Lemma~\ref{lm:level-size-bound-auilixary}, a necessary condition for \eqref{eq:level-size-bound-main} is $|\cC^k_{h,l}| \leq 16 l \cdot 4^l\gamma_l^{2} d$.
\end{proof}
\subsection{Proof of Lemma~\ref{lm:covering-number}}
We first present a claim from~\citet{vial2022improved} controlling the infinite norm of coefficient $\wb$.
\begin{lemma}[Claim~10,~\citealt{vial2022improved}]\label{lm:vial10}
For any $(k, h, l) \in [K] \times [H] \times \NN^+$, we have $\|\wb_{h, l}^k\|_{\infty} \le \|\wb_{h, l}^k\|_2 \le (2^ldH)^4$.
\end{lemma}
\begin{proof}[Proof of Lemma~\ref{lm:covering-number}]
    Denote $\cX_\ell$ as the set of all $\tilde \wb_{h, \ell}^k$ and $\cY_\ell$ as the set of all $\tilde \Ub_{h, \ell}^{k, -1}$. From the definition of $\cV^k_{h,l}$, we have that $|\cV^k_{h,l}| \leq \prod_{\ell=1}^{l} \big(|\cX_\ell| \cdot |\cY_\ell|\big)$. From Lemma~\ref{lm:vial10}, we have $\|\wb^{k}_{h,\ell}\|_{\infty} \leq (2^\ell dH)^4$. Note that $\wb^{k}_{h,\ell} \in \RR^d$, we have the number of different $\tilde \wb^{k}_{h,\ell}$ controlled by 
    \begin{align*}
        |\cX_\ell| \leq (1 + 2 \cdot (2^\ell dH)^4/\epsrnd_{\ell})^d \leq (2 \cdot (2^\ell dH)^4 \cdot 2^{6+4\ell}d)^d \leq 2^{(7+8\ell)d}d^{5d}H^{4d}.
    \end{align*}
    In addition, we have $\|(\Ub^{k}_{h,l})^{-1}\|_{\infty} \leq 1/\lambda = 1/16$. So we can bound the number of $\tilde \Ub_{h, \ell}^{k, -1}$ by 
    \begin{align*}
        |\cY_\ell| \leq (1 + 2 \cdot 1/(16\epsrnd_{\ell}))^{d^2} \leq (2 \cdot 2^{2+4\ell} d)^{d^2} \leq 2^{(3+4\ell)d^2}d^{d^2}.
    \end{align*}
    As a result, we can conclude that 
    \begin{align*}
        |\cV^k_{h,l}| &\leq \prod_{\ell=1}^{l} \big(|\cX_\ell| \cdot |\cY_\ell|\big) \leq \prod_{\ell=1}^{l} \big(2^{(7+8\ell)d}d^{5d}H^{4d} \cdot 2^{(3+4\ell)d^2}d^{d^2}\big) \leq (2^{22}d^5H^4)^{l^2d^2}.
    \end{align*}
\end{proof}
\subsection{Proof of Lemma~\ref{lm:estimate-error-pre}}
\begin{proof}[Proof of Lemma~\ref{lm:estimate-error-pre}]
    According to Proposition~\ref{prop:linear}, there exists a parameter $\wb_h$ such that for any $(s, a) \in \cS \times \cA$, it holds that $\big|\langle \bphi(s, a), \wb_h \rangle - [\BB_h \hat{V}^k_{h+1}](s, a) \big| \leq 2H\zeta$ . Denoting $\eta^\tau_h = \langle \bphi^\tau_h, \wb_h \rangle - [\BB_h \hat{V}^k_{h+1}](s^\tau_h, a^\tau_h)$ and $\eps^\tau_h = \big(\hat V^k_{h+1}(s^\tau_{h+1}) - [\PP_h \hat{V}^k_{h+1}](s^\tau_h, a^\tau_h)\big)$, we have 
    \begin{align} \label{eq:estimate-error-decompose}
        \Ub^{k}_{h,l} (\wb^k_{h,l} - \wb_h) &=\sum_{\tau \in \cC^{k-1}_{h,l}} \bphi^\tau_h\Big(r^\tau_h + \hat V^k_{h+1}(s^\tau_{h+1})\Big) - \Big (\lambda \Ib + \sum_{\tau \in \cC^{k-1}_{h,l}} \bphi^\tau_h(\bphi^\tau_h)^{\top} \Big)\wb_h \notag \\
        &= -\lambda \wb_h + \sum_{\tau \in \cC^{k-1}_{h,l}} \bphi^\tau_h\Big(r^\tau_h + \hat V^k_{h+1}(s^\tau_{h+1}) - \langle \bphi^\tau_h, \wb_h \rangle\Big) \notag \\
        &= -\lambda \wb_h + \sum_{\tau \in \cC^{k-1}_{h,l}} \bphi^\tau_h\Big(r^\tau_h + \hat V^k_{h+1}(s^\tau_{h+1}) - [\BB_h \hat{V}^k_{h+1}](s^\tau_h, a^\tau_h)\Big) + \sum_{\tau \in \cC^{k-1}_{h,l}} \bphi^\tau_h \eta^\tau_h \notag \\
        &= -\lambda \wb_h + \sum_{\tau \in \cC^{k-1}_{h,l}} \bphi^\tau_h \eps^\tau_h + \sum_{\tau \in \cC^{k-1}_{h,l}} \bphi^\tau_h \eta^\tau_h,
    \end{align}
    where the first equality holds due to the definition of $\Ub^k_{h,l}, \wb^k_{h,l}$, the second equality holds by rearranging the terms, the third equality holds according the definition of $\eta^\tau_h$, and the last equality holds from the relationship that $[\BB_h \hat{V}^k_{h+1}](s^\tau_h, a^\tau_h) = r^\tau_h + [\PP_h \hat{V}^k_{h+1}](s^\tau_h, a^\tau_h)$. Therefore, for any vector $
    \bphi \in \RR^d$, it holds that 
    \begin{align} \label{eq:estimate-error-main}
        \big|\big\langle\bphi,  \wb^k_{h,l} - \wb_h\big\rangle\big| 
        &= \big|\bphi^{\top} \big(\Ub^{k}_{h,l}\big)^{-1} \Ub^{k}_{h,l} (\wb^k_{h,l} - \wb_h) \big| \notag \\
        &= \Bigg|\bphi^{\top} (\Ub^{k}_{h,l}\big)^{-1} \cdot \Bigg(  -\lambda \wb_h + \sum_{\tau \in \cC^{k-1}_{h,l}} \bphi^\tau_h \eps^\tau_h + \sum_{\tau \in \cC^{k-1}_{h,l}} \bphi^\tau_h \eta^\tau_h\Bigg) \Bigg| \notag \\
        &\leq \|\bphi\|_{(\Ub^{k}_{h,l})^{-1}} \Bigg\|  -\lambda \wb_h + \sum_{\tau \in \cC^{k-1}_{h,l}} \bphi^\tau_h \eps^\tau_h + \sum_{\tau \in \cC^{k-1}_{h,l}} \bphi^\tau_h \eta^\tau_h \Bigg\|_{(\Ub^{k}_{h,l})^{-1}},
    \end{align}
    where the second equality follows \eqref{eq:estimate-error-decompose} and the inequality holds from Cauchy–Schwarz inequality (i.e., $|\xb^\top\Ub \yb| \leq \|\xb\|_{\Ub} \|\yb\|_{\Ub}$). From the triangle inequality, we have 
    \begin{align} \label{eq:estimate-error-triangle}
        &\Bigg\|  -\lambda \wb_h + \sum_{\tau \in \cC^{k-1}_{h,l}} \bphi^\tau_h \eps^\tau_h + \sum_{\tau \in \cC^{k-1}_{h,l}} \bphi^\tau_h \eta^\tau_h \Bigg\|_{(\Ub^{k}_{h,l})^{-1}} \notag \\
        &\leq \lambda \|\wb_h\|_{(\Ub^{k}_{h,l})^{-1}} + \Bigg\|\sum_{\tau \in \cC^{k-1}_{h,l}} \bphi^\tau_h \eps^\tau_h\Bigg\|_{(\Ub^{k}_{h,l})^{-1}} + \Bigg\|\sum_{\tau \in \cC^{k-1}_{h,l}} \bphi^\tau_h \eta^\tau_h\Bigg\|_{(\Ub^{k}_{h,l})^{-1}}.
    \end{align}
    There are three terms which we will bound respectively.
    For the first term, we have 
    \begin{align} \label{eq:estimate-error-term1}
        \lambda \|\wb_h\|_{(\Ub^{k}_{h,l})^{-1}} \leq 2\sqrt{d\lambda}H \leq 0.1\gamma_l,
    \end{align}
    where the first inequality holds due to the fact that $\|\wb_h\|_2 \le 2H\sqrt{d}$ as of Proposition~\ref{prop:linear} and the fact that $\Ub_{h, l}^k \succeq \lambda \Ib$. Under the good event $\cG_1$ and Lemma~\ref{lm:concentration-hatV}, the second term can be bounded by the following:
    \begin{align} \label{eq:estimate-error-term2}
        \Bigg\|\sum_{\tau \in \cC^{k-1}_{h,l}} \bphi^\tau_h \eps^\tau_h\Bigg\|_{(\Ub^{k}_{h,l})^{-1}} \leq 1.1\gamma_l.
    \end{align}
    And the last term can be bounded by:
    \begin{align} \label{eq:estimate-error-term3}
        \Bigg \|\sum_{\tau \in \cC^{k-1}_{h,l}} \bphi^\tau_h \eta^\tau_h \Bigg\|_{(\Ub^k_{h,l})^{-1}} &\leq 2H\zeta\sqrt{|\cC^{k}_{h,l}|}\leq  2H\zeta\sqrt{16l\cdot4^l\gamma_l^{2}d}=8\sqrt{ld}H \cdot 2^l\gamma_l\zeta,
    \end{align}
    where the first inequality is due to Lemma~\ref{lm:misspecify}, and the second inequality follows from Lemma~\ref{lm:level-size-bound}. Plugging \eqref{eq:estimate-error-triangle}, \eqref{eq:estimate-error-term1}, \eqref{eq:estimate-error-term2}, and \eqref{eq:estimate-error-term3} into \eqref{eq:estimate-error-main} gives 
    \begin{align} \label{eq:estimate-error-result1}
        \big|\big\langle\bphi,  \wb^k_{h,l} - \wb_h\big\rangle\big| \leq \big(1.2\gamma_l + 8\sqrt{ld}H \cdot 2^l\gamma_l\zeta\big)\|\bphi\|_{(\Ub^{k}_{h,l})^{-1}}.
    \end{align}
    So for any $(s, a) \in \cS \times \cA$, we have 
    \begin{align} \label{eq:estimate-error-result2}
        & \big|Q^{k}_{h,l}(s, a) - [\BB_h \hat{V}^k_{h+1}](s, a) \big| =  \big|\big\langle\bphi(s, a), \tilde \wb^k_{h,l}\big\rangle  - [\BB_h \hat{V}^k_{h+1}](s, a) \big| \notag \\
        &\leq \big|\big\langle\bphi(s, a), \tilde \wb^k_{h,l} -  \wb^k_{h,l}\big\rangle\big| + \big|\big\langle\bphi(s, a), \wb^k_{h,l} -  \wb_h\big\rangle\big| + \big|\big\langle\bphi(s, a),  \wb_h\big\rangle  - [\BB_h \hat{V}^k_{h+1}](s, a) \big| \notag \\
        &\leq 0.01 \cdot 2^{-4l} + \big(1.2 + 8\sqrt{ld}H \cdot 2^l\zeta\big)\gamma_l\|\bphi(s, a)\|_{(\Ub^{k}_{h,l})^{-1}} + 2H\zeta.
    \end{align}
    where the first inequality holds from the triangle inequality, and there are three terms in the second inequality which we will bound them respectively: the first term is given by Lemma~\ref{lm:rounding-error}, the second term follows \eqref{eq:estimate-error-result1}, and the third term holds from the definition of $\wb_h$.
\end{proof}
\subsection{Proof of Lemma~\ref{lm:exists-optimal-pre}}
\begin{proof}[Proof of Lemma~\ref{lm:exists-optimal-pre}]
    We prove by doing case analysis. In case that action $a_l \in \cA^k_{h,l+1}(s)$, we can assign $a_{l+1} = a_l \in \cA^k_{h,l+1}(s)$ so that 
    \begin{align} \label{eq:exists-optimal-1}
    [\BB_h \hat{V}^k_{h+1}](s, a_l) - [\BB_h \hat{V}^k_{h+1}](s, a_{l+1}) = 0.
    \end{align}
    On the other hand, in the case that $a_l \notin \cA^k_{h, l+1}(s)$, the action $a_l$ is eliminated with $ Q^k_{h,l}(s, a_l) < V^k_{h,l}(s) - 4 \cdot 2^{-l}$. Note in this case, there exists $a_{l+1} = \pi^k_{h,l}(s) \in \cA^k_{h,l+1}(s)$ such that 
    \begin{align} \label{eq:exists-optimal-2-1}
        Q^k_{h,l}(s, a_l) + 4 \cdot 2^{-l} < V^k_{h,l}(s) = Q^k_{h,l}(s,a_{l+1}).
    \end{align} 
    According to Lemma~\ref{lm:estimate-error} and the condition that $l \leq L_0$, we have that empirical state-value function $Q^k_{h,l}(s, \cdot)$ is a good estimation for $[\BB_h \hat{V}^k_{h+1}](s, \cdot)$ on actions $a_l, a_{l+1} \in \cA^k_l(s)$ under event $\cG_1$:
    \begin{align}
        \big|[\BB_h \hat{V}^k_{h+1}](s, a_l) - Q^k_{h,l}(s, a_l) \big| &\leq 2 \cdot 2^{-l} + \chi\sqrt{L_0}\zeta \label{eq:exists-optimal-2-2} \\
        \big|[\BB_h \hat{V}^k_{h+1}](s, a_{l+1}) - Q^k_{h,l}(s, a_{l+1}) \big| &\leq 2 \cdot 2^{-l} + \chi\sqrt{L_0}\zeta. \label{eq:exists-optimal-2-3} 
    \end{align}
    Moreover, 
    \begin{align} \label{eq:exists-optimal-2}
        &\quad\ [\BB_h \hat{V}^k_{h+1}](s, a_l) - [\BB_h \hat{V}^k_{h+1}](s, a_{l+1}) \notag \\
        &= \big([\BB_h \hat{V}^k_{h+1}](s, a_l) - Q^k_{h,l}(s, a_l)\big) \notag  \\
        &\quad +\big(Q^k_{h,l}(s, a_l) - Q^k_{h,l}(s,a_{l+1}) \big) + \big(Q^k_{h,l}(s,a_{l+1}) - [\BB_h \hat{V}^k_{h+1}](s, a_{l+1})\big) \notag \\
        &\leq 2 \cdot \big(2 \cdot 2^{-l} + \chi\sqrt{L_0}\zeta \big) - 4 \cdot 2^{-l} \notag \\
        &= 2\chi\sqrt{L_0}\zeta.
    \end{align}
    where the first inequality is derived from combining \eqref{eq:exists-optimal-2-1}, \eqref{eq:exists-optimal-2-2}, and \eqref{eq:exists-optimal-2-3}.
    So from \eqref{eq:exists-optimal-1} and \eqref{eq:exists-optimal-2}, we have that $[\BB_h \hat{V}^k_{h+1}](s, a_l) - [\BB_h \hat{V}^k_{h+1}](s, a_{l+1}) \leq 2\chi\sqrt{L_0}\zeta$ holds in both cases.
\end{proof}
\subsection{Proof of Lemma~\ref{lm:exists-optimal}}
\begin{proof}[Proof of Lemma~\ref{lm:exists-optimal}]
    We prove by induction on $l$. The induction basis holds at $l = 0$ by selecting $a_1 = \argmax_{a \in \cA} [\BB_h \hat{V}^k_{h+1}](s, a) \in \cA$ which ensures $ \max_{a\in \cA}[\BB_h \hat{V}^k_{h+1}](s, a) - [\BB_h \hat{V}^k_{h+1}](s, a_{1}) = 0$. Additionally, if the induction hypothesis holds for $l-1$, we have that 
    \begin{align*}
         & \max_{a\in \cA}[\BB_h \hat{V}^k_{h+1}](s, a) - [\BB_h \hat{V}^k_{h+1}](s, a_{l+1}) \\
         &= \big(\max_{a\in \cA}[\BB_h \hat{V}^k_{h+1}](s, a) - [\BB_h \hat{V}^k_{h+1}](s, a_{l}) \big) + \big( [\BB_h \hat{V}^k_{h+1}](s, a_{l}) - [\BB_h \hat{V}^k_{h+1}](s, a_{l+1})\big) \\
         &\leq 2(l-1)\chi\sqrt{L_0}\zeta + 2\chi\sqrt{L_0}\zeta \\
         &= 2l \cdot \chi\sqrt{L_0}\zeta,
    \end{align*}
    where the first inequality term is derived from combining induction hypothesis with Lemma~\ref{lm:exists-optimal-pre}.
    We can then reach desired statement holds for all $l$ in the range by induction.
\end{proof}
\subsection{Proof of Lemma~\ref{lm:value-func-lower}}
\begin{proof}[Proof of Lemma~\ref{lm:value-func-lower}]
    According to Lemma~\ref{lm:exists-optimal}, there exists some action $a_l \in \cA^k_{h,l}(s)$ that 
    \begin{align} \label{eq:value-func-lower-1}
        \max_{a\in \cA}[\BB_h \hat{V}^k_{h+1}](s, a) - [\BB_h \hat{V}^k_{h+1}](s, a_l) \leq 2(l-1) \chi\sqrt{L_0}\zeta.
    \end{align}
    Moreover, we have 
    \begin{align} \label{eq:value-func-lower-2}
        [\BB_h \hat{V}^k_{h+1}](s, a_l) - V^k_{h,l}(s) &
        \leq [\BB_h \hat{V}^k_{h+1}](s, a_l) - Q^k_{h,l}(s, a_l) \leq 2 \cdot 2^{-l} + \chi\sqrt{L_0}\zeta,
    \end{align}
    where the first inequality comes from the definition $V^k_{h,l}(s) = \max_{a \in \cA^k_{h,l}} Q^k_{h,l}(s, a)$ and the second inequality holds according to Lemma~\ref{lm:estimate-error} with $l \leq L_0$. Adding up \eqref{eq:value-func-lower-1} and \eqref{eq:value-func-lower-2} leads to
    \begin{align*}
        \max_{a\in \cA}[\BB_h \hat{V}^k_{h+1}](s, a) - V^k_{h,l} \leq 2 \cdot 2^{-l} + (2l-1) \chi\sqrt{L_0}\zeta.
    \end{align*}
    This completes the proof.
\end{proof}
\subsection{Proof of Lemma~\ref{lm:value-func-upper}}
\begin{proof}[Proof of Lemma~\ref{lm:value-func-upper}]
    The statement holds by simply checking:
    \begin{align*}
        V^k_{h,l}(s) - \max_{a\in \cA}[\BB_h \hat{V}^k_{h+1}](s, a) &\leq 
        V^k_{h,l}(s) - [\BB_h \hat{V}^k_{h+1}](s, \pi^k_{h,l}(s)) \\
        &= Q^k_{h,l}(s, \pi^k_{h,l}(s)) - [\BB_h \hat{V}^k_{h+1}](s, \pi^k_{h,l}(s)) \\
        &\leq 2 \cdot 2^{-l} + \chi\sqrt{L_0}\zeta,
    \end{align*}
    where the first inequality holds from $\max_{a\in \cA}[\BB_h \hat{V}^k_{h+1}](s, a) \geq [\BB_h \hat{V}^k_{h+1}](s, \pi^k_{h,l}(s))$, the equality is from the definition $V^k_{h,l}(s) = Q^k_{h,l}(s, \pi^k_{h,l}(s))$, and the last inequality holds according to Lemma~\ref{lm:estimate-error} with the condition $l \leq L_0$.
\end{proof}
\subsection{Proof of Lemma~\ref{lm:opti-value-func-lower}}
\begin{proof}[Proof of Lemma~\ref{lm:opti-value-func-lower}]
    The statement holds by checking
    \begin{align*}
        &\quad \min\big\{V^k_{h,l}(s) + 3 \cdot 2^{-l}, \hat{V}^{k}_{h,l-1}(s)\big\} - \max_{a\in \cA}[\BB_h \hat{V}^k_{h+1}](s, a) \\
        &= \min_{\ell=1}^l\{V^k_{h,\ell}(s) + 3 \cdot 2^{-\ell}\} - \max_{a\in \cA}[\BB_h \hat{V}^k_{h+1}](s, a)\\
        &\geq \min_{\ell=1}^l \{3 \cdot 2^{-\ell} - (2 \cdot 2^{-l} + (2\ell-1)\chi\sqrt{L_0}\zeta)\} \\
        &= 2^{-l} - (2l-1)\chi\sqrt{L_0}\zeta,
    \end{align*}
    where the first equality holds due to $\hat{V}^{k}_{h,l}(s) = \min_{\ell=1}^l\{V^k_{h,\ell}(s) + 3 \cdot 2^{-\ell}\}$, the inequality holds according to Lemma~\ref{lm:value-func-lower}, and the last equality holds since $2^{-l}$ decreases as $l$ increases.
\end{proof}
\subsection{Proof of Lemma~\ref{lm:pess-value-func-upper}}
\begin{proof}[Proof of Lemma~\ref{lm:pess-value-func-upper}]
    The statement holds by checking
    \begin{align*}
        &\quad \max_{a\in \cA}[\BB_h \hat{V}^k_{h+1}](s, a) - \max\big\{V^k_{h,l}(s) - 3 \cdot 2^{-l}, \widecheck{V}^{k}_{h,l-1}(s)\big\} \\
        &= \max_{a\in \cA}[\BB_h \hat{V}^k_{h+1}](s, a) - \max_{\ell=1}^l\{V^k_{h,\ell}(s) - 3 \cdot 2^{-\ell}\} \\
        &= \min_{\ell=1}^l \big\{\max_{a\in \cA}[\BB_h \hat{V}^k_{h+1}](s, a) - V^k_{h,\ell}(s) + 3 \cdot 2^{-\ell} \big\} \\
        &\geq \min_{\ell=1}^l \{-(2 \cdot 2^{-l} + \chi\sqrt{L_0}\zeta) + 3 \cdot 2^{-\ell}\} \\
        &= 2^{-l} - \chi\sqrt{L_0}\zeta,
    \end{align*}
    where the first equality holds due to the design of Algorithm~\ref{alg:Lin} such that $\widecheck{V}^{k}_{h,l}(s) = \max_{\ell=1}^l\{V^k_{h,\ell}(s) - 3 \cdot 2^{-\ell}\}$, the inequality holds according to Lemma~\ref{lm:value-func-upper}, and the last equality holds since $2^{-l}$ decreases as $l$ increases.
\end{proof}
\subsection{Proof of Lemma~\ref{lm:final-value-func-lower}}

We prove Lemma~\ref{lm:final-value-func-lower} in this subsection. The first lemma which we introduce establishes an upper bound on the underestimation of the state value function $\hat{V}^k_{h}$ for every action and every state through a categorised discussion based on whether Algorithm~\ref{alg:Lin} reaches phase $\Leps$ for state $s$. Specifically, if the process does not reach phase $\Leps$, we can substantiate the statement by applying Lemma~\ref{lm:opti-value-func-lower} to phase $l^k_h(s)-1$. Conversely, if the process reaches phase $\Leps$, the statement can be proven by applying Lemma~\ref{lm:value-func-lower} to phase $\Leps$.
\begin{lemma} \label{lm:final-value-func-lower-pre}
    Under event $\cG_1$ and for all $\eps > 0$ that $\cG_\eps$ is satisfied, for any $(k, h, s) \in [K] \times [H] \times \cS$,
    \begin{align*}
        \max_{a\in \cA}[\BB_h \hat{V}^k_{h+1}](s, a) - \hat{V}^k_h(s) \leq 0.07\eps/H.
    \end{align*}
\end{lemma}

Now we are ready to prove Lemma~\ref{lm:final-value-func-lower} by induction.
\begin{proof}[Proof of Lemma~\ref{lm:final-value-func-lower}]
    We prove by induction on stage $h \in [H]$.
    It is sufficient to show for any $h \in [H], s \in \cS$,
    \begin{align} \label{eq:final-value-func-lower-induction}
        V^*_h(s) - \hat{V}^k_h(s) \leq 0.07\eps \cdot (H+1-h)/H.
    \end{align}
    We use induction on $h$ from $H+1$ to $1$ to prove the statement.
    The induction basis holds from the definition that $V^*_{H+1}(s) = \hat{V}^k_{H+1}(s) = 0$. Assume the induction hypothesis \eqref{eq:final-value-func-lower-induction} holds for $h+1$, we have 
    \begin{align} \label{eq:final-value-func-lower-hyp-1}
         \max_{a\in \cA}[\BB_h V^*_{h+1}](s, a) -\max_{a\in \cA}[\BB_h \hat{V}^k_{h+1}](s, a) &\leq \max_{a\in \cA}[\BB_h (V^*_{h+1} - \hat{V}^k_{h+1})](s, a) \notag \\
         &\leq \max_{s' \in \cS} \big(V^*_{h+1}(s') - \hat{V}^k_{h+1}(s') \big) \notag \\
         &\leq 0.07\eps \cdot (H-h)/H.
    \end{align}
    So for level $h$, it holds that 
    \begin{align*}
        &\quad\ V^*_h(s) - \hat{V}^k_h(s)\\
        &= \big(\max_{a\in \cA}[\BB_h V^*_{h+1}](s, a) -\max_{a\in \cA}[\BB_h \hat{V}^k_{h+1}](s, a) \big) + \big(\max_{a\in \cA}[\BB_h \hat{V}^k_{h+1}](s, a) - \hat{V}^k_h(s) \big) \\
        &\leq 0.07\eps \cdot (H-h)/H + 0.07\eps/H \leq 0.07\eps \cdot (H+1-h)/H,
    \end{align*}
    where the first inequality holds by combining \eqref{eq:final-value-func-lower-hyp-1} with Lemma~\ref{lm:final-value-func-lower-pre}. This proves the induction statement \eqref{eq:final-value-func-lower-induction} for $h$, which leads to the desired statement.
\end{proof}
\subsection{Proof of Lemma~\ref{lm:local-overestimation}}
We prove Lemma~\ref{lm:local-overestimation} in this subsection, the first lemma we use establishes an upper bound on the overestimation of the state value function $\hat{V}^k_{h}$ for the executed policy $\pi^{k}_h(s)$ across all states.

\begin{lemma} \label{lm:all-optimal}
    Under event $\cG_1$ and for all $\eps > 0$ that $\cG_\eps$ is satisfied, for any $(k, h, s) \in [K] \times [H] \times \cS$,
    \begin{align*}
        \max_{a\in \cA}[\BB_h \hat{V}^k_{h+1}](s, a) - [\BB_h \hat{V}^k_{h+1}](s, \pi^{k}_h(s)) \leq 16 \cdot 2^{-l^k_h(s)} + 0.10\eps/H.
    \end{align*}
\end{lemma}

Then the following lemma establishes an upper bound on the decision error induced by the arm-elimination process with respect to the state-action value function given by the ground-truth transform.

\begin{lemma} \label{lm:final-value-func-upper}
    Under event $\cG_1$ and for all $\eps > 0$ that $\cG_\eps$ is satisfied, for any $(k, h, s) \in [K] \times [H] \times \cS$, 
    \begin{align*}
        \hat{V}^k_h(s) - \max_{a\in \cA}[\BB_h \hat{V}^k_{h+1}](s, a) \leq 10 \cdot 2^{-l^k_h(s)} + 0.06\eps/H.
    \end{align*}
\end{lemma}

\begin{proof}[Proof of Lemma~\ref{lm:local-overestimation}]
    We can directly reach the desired result by taking summation on Lemma~\ref{lm:all-optimal} and Lemma~\ref{lm:final-value-func-upper}:
    \begin{align*}
        &\hat{V}^k_h(s) - [\BB_h \hat{V}^k_{h+1}](s, \pi^{k}_h(s)) \\
        &\leq 
        \big(\max_{a\in \cA}[\BB_h \hat{V}^k_{h+1}](s, a) - [\BB_h \hat{V}^k_{h+1}](s, \pi^{k}_h(s))\big) + 
        \big(\hat{V}^k_h(s) - \max_{a\in \cA}[\BB_h \hat{V}^k_{h+1}](s, a)\big) \\
        &\leq \big(16 \cdot 2^{-l^k_h(s)} + 0.10\eps/H\big) + \big(10 \cdot 2^{-l^k_h(s)}  + 0.06\eps/H\big) \\
        &= 26 \cdot 2^{-l^k_h(s)} + 0.16\eps/H.
    \end{align*}
\end{proof}
\subsection{Proof of Lemma~\ref{lm:extra-concentration}}

We can prove the statement by applying a union bound to the concentration event, as given by the Azuma-Hoeffding inequality.
\begin{proof}[Proof of Lemma~\ref{lm:extra-concentration}]
    Consider some fixed $h \in [H]$ and $\eps = 2^{-l} > 0$. List the episodes index $k$ such that $V^*_h(s^k_h) - V^{\pi^k}_h(s^k_h) > \eps$ holds in ascending order as $\{\tau_i\}_i$. Recall that 
    \begin{align*}
    \eta^{\tau_i}_h =  [\PP_h(\hat{V}^{\tau_i}_{h+1} - V^{\pi^{\tau_i}}_{h+1})] (s^{\tau_i}_{h}, \pi^{\tau_i}_{h}(s^{\tau_i}_{h})) - \big(\hat{V}^{\tau_i}_{h+1}(s^{\tau_i}_{h+1}) -  V^{\pi^{\tau_i}}_{h+1}(s^{\tau_i}_{h+1}) \big) .
    \end{align*}
    Since the environment sample $s^{\tau_i}_{h'+1}$ according to $\PP_{h'}(\cdot|s^{\tau_i}_{h'}, a^{\tau_i}_{h'})$, we have $\eta^{\tau_i}_{h'}$ is $\cF^{\tau_i}_{h'+1}$-measurable with $\Expt\big[\eta^{\tau_i}_{h'}\big|\cF^{\tau_i}_{h'}\big] = 0$. Since both $0 \leq \hat{V}^{\tau_i}_{h'+1}(s^{\tau_i}_{h'+1}) \leq H$ and $0 \leq V^{\pi^{\tau_i}}_{h'+1}(s^{\tau_i}_{h'+1}) \leq H$ hold, we have $|\eta^{\tau_i}_{h'}| \leq 2H$. According to Lemma~\ref{lm:azuma–hoeffding} over filtration $$\cF^{\tau_1}_{h} \subseteq \cF^{\tau_1}_{h+1} \subseteq  \cdots \subseteq \cF^{\tau_1}_H \subseteq  \cF^{\tau_2}_{h} \subseteq \cF^{\tau_2}_{h+1} \subseteq  \cdots \subseteq  \cF^{\tau_2}_H \subseteq  \cdots \subseteq \cF^{\tau_i}_{h'} \subseteq \cdots$$ for some fixed $S = |\Kepsh|$, the good event that
     \begin{align*}
         \sum_{i=1}^{|\Kepsh|}\sum_{h'=h}^{H} \eta^{\tau_i}_{h'} \leq 2H\sqrt{2HS\log (4HS^2l^2/\delta)} = 4\sqrt{H^3|\Kepsh|\log (4H|\Kepsh|\log(\eps^{-1})/\delta)}
     \end{align*}
     happens with probability at least $1 - \delta/(4HS^2l^2)$. By the union bound statement over all $(h, S, l) \in [H] \times [K] \times \NN^+$, we have the bad event happens with probability at most 
     \begin{align*}
         \Pr[\cG_2^{\complement}] \leq \sum_{h=1}^{H}\sum_{S=1}^{K} \sum_{l=1}^{\infty} \Pr[\cB_2(h, 2^{-l})] \leq \sum_{h=1}^{H}\sum_{S=1}^{K} \sum_{l=1}^{\infty} \frac{\delta}{4HS^2l^2} \leq \delta,
     \end{align*}
     where the last inequality holds from $\sum_{n\geq 1}n^{-2} = \pi^2/6$, which reach the desired statement.
\end{proof}
\subsection{Proof of Lemma~\ref{lm:one-step-regret-final}}
We first provide the following instantaneous regret upper bound by combining Lemma~\ref{lm:final-value-func-lower} and Lemma~\ref{lm:local-overestimation}.
\begin{lemma}\label{lm:one-step-regret}
   Under event $\cG_1$ and for all $\eps > 0$ that $\cG_\eps$ is satisfied, for any $(k, h) \in [K] \times [H]$, 
    \begin{align*}
        V^*_h(s^k_h) - V^{\pi^k}_h(s^k_h) \leq 0.23\eps + 26\sum_{h'=h}^H 2^{-l^k_{h'}(s^k_{h'})} + \sum_{h'=h}^{H} \eta^k_{h'},
    \end{align*}
    where $\eta^k_h = [\PP_h(\hat{V}^k_{h+1} - V^{\pi^k}_{h+1})] (s^k_{h}, \pi^k_{h}(s^k_{h})) - \big(\hat{V}^k_{h+1}(s^k_{h+1}) -  V^{\pi^k}_{h+1}(s^k_{h+1}) \big)$ is a $\cF_{h+1}^k$-measurable random variable that $\EE [\eta_h^k | \cF_h^k] = 0$ and $|\eta_h^k| \le H$.
\end{lemma}

Together with Lemma~\ref{lm:uncertainty-sum} and the definition of $\cG_2$, we can provide an upper bound for arbitrary subsets.
\begin{proof}[Proof of Lemma~\ref{lm:one-step-regret-final}]
    Taking summation on result given by Lemma~\ref{lm:one-step-regret} to all $k \in \cK$ gives
    \begin{align} \label{eq:theorem-upper-initial}
        \sum_{k \in \cK}\big(V^*_h(s^k_h) - V^{\pi^k}_h(s^k_h) \big) \leq 0.23|\cK| \eps + 26 \sum_{k \in \cK}\sum_{h'=h}^H 2^{-l^k_{h'}(s^k_{h'})}  + \sum_{k \in \cK}\sum_{h'=h}^{H} \eta^k_{h'}.
    \end{align}
    We can bound the second term according to Lemma~\ref{lm:uncertainty-sum},
    \begin{align} \label{eq:theorem-upper-initial-2}
         26 \sum_{k \in \cK}\sum_{h'=h}^H 2^{-l^k_{h'}(s^k_{h'})}  \leq 0.26|\cK| \eps + 2^{17} \Leps d H^2 \gamma_{\Leps}^2\eps^{-1}.
    \end{align}
    Under event $\cG_2$, the third term satisfies that 
    \begin{align} \label{eq:theorem-upper-initial-3}
        \sum_{k \in \cK}\sum_{h'=h}^{H} \eta^k_{h'} \leq 4\sqrt{H^3|\cK|\log (4H|\cK|\log(\eps^{-1})/\delta)}.
    \end{align}
    Plugging \eqref{eq:theorem-upper-initial-2} and \eqref{eq:theorem-upper-initial-3} into \eqref{eq:theorem-upper-initial} gives
    \begin{align}
        \sum_{k \in \cK}\big(V^*_h(s^k_h) - V^{\pi^k}_h(s^k_h) \big) \leq 0.49|\cK| \eps + 2^{17}\Leps d H^2 \gamma_{\Leps}^2 \eps^{-1} + 4\sqrt{H^3|\cK|\log (4H|\cK|\log(\eps^{-1})/\delta)}.
    \end{align}
\end{proof}

\section{Proof of Lemmas in Appendix~\ref{app:add}} \label{app:add2}
\subsection{Proof of Lemma~\ref{lm:final-value-func-lower-pre}}
\begin{proof}[Proof of Lemma~\ref{lm:final-value-func-lower-pre}]
    We start the proof by discussing different cases. First, if $l^k_h(s) \leq \Leps$, we have $l^k_h(s) - 1 \leq \min\{\Leps, l^k_h(s) - 1\}$, according to the definition of $\hat V_{h, l}^k(s)$,
    \begin{align} \label{eq:final-value-func-lower-1}
        \max_{a\in \cA}[\BB_h \hat{V}^k_{h+1}](s, a) - \hat{V}^k_h(s) 
        &= \max_{a\in \cA}[\BB_h \hat{V}^k_{h+1}](s, a) - \hat{V}^k_{h,l^k_h(s) - 1}(s) \notag \\
        &\leq -2^{-(l^k_h(s) - 1)}+ 2(l^k_h(s) - 1)\chi\sqrt{\Leps}\zeta \notag \\
        &\leq 0 + 2 \chi \Leps^{1.5} \zeta \notag \\
        &\leq 0.02\eps/H,
    \end{align}
    where the first inequality holds from Lemma~\ref{lm:opti-value-func-lower}, and the last inequality holds due to $\chi \Leps^{1.5} \zeta \leq 2^{-\Leps} \leq 0.01\eps /H$ given by $\cG_\eps$.
    
    On the other hand, when $l_h^k(s) > \Leps$, we have $\Leps \leq \min\{\Leps, l^k_h(s) - 1\}$ and thus 
    \begin{align} \label{eq:final-value-func-lower-2-pre}
        \hat{V}^k_h(s) \geq \widecheck{V}^k_{h,\Leps}(s)
        &\geq V^k_{h,\Leps}(s) - 3 \cdot 2^{-\Leps}
    \end{align}
    where the first inequality is due to Lemma~\ref{lm:final-value-func-pre} and the second inequality holds due to the definition of $\widecheck{V}^k_{h,\Leps}(s)$. Therefore, $\Leps \leq \min\{\Leps, l^k_h(s) - 1\}$ yields
    \begin{align} \label{eq:final-value-func-lower-2}
        \max_{a\in \cA}[\BB_h \hat{V}^k_{h+1}](s, a) - \hat{V}^k_h(s) &\leq \max_{a\in \cA}[\BB_h \hat{V}^k_{h+1}](s, a) - V^k_{h,\Leps}(s) + 3 \cdot 2^{-\Leps} \notag \\
        &\leq 5 \cdot 2^{-\Leps} + (2\Leps-1) \chi\sqrt{\Leps}\zeta \notag \\
        &\leq 0.05\eps/H + 0.02\eps/H = 0.07\eps/H,
    \end{align}
    where the first inequality is given by \eqref{eq:final-value-func-lower-2-pre}, the second inequality is given by Lemma~\ref{lm:value-func-lower}, and the last inequality holds from $\chi\Leps^{1.5}\zeta \leq 2^{-\Leps} \leq 0.01\eps /H$ given by $\cG_\eps$. So considering both \eqref{eq:final-value-func-lower-1} and \eqref{eq:final-value-func-lower-2}, we have  the first statement 
    \begin{align*} 
        \max_{a\in \cA}[\BB_h \hat{V}^k_{h+1}](s, a) - \hat{V}^k_h(s) \leq 0.07\eps/H
    \end{align*} always holds under event $\cG_1$.
\end{proof}
\subsection{Proof of Lemma~\ref{lm:all-optimal}}
We prove Lemma~\ref{lm:all-optimal} by applying Lemma~\ref{lm:all-optimal-pre} on phase $\min\{\Leps, l^k_h(s) - 1\}$, in this subsection. 
\begin{proof}[Proof of Lemma~\ref{lm:all-optimal}]
    Note we have $\pi^k_{h,l^k_h(s)-1}(s) \in \cA^k_{h,l^k_h(s)}(s)$ according to the definition of $\cA^k_{h,l+1}(s)$. This implies $\pi^k_h(s) \in \cA^k_{h,l^k_h(s)}(s)$ during the elimination process.

    If $l^k_h(s) \leq \Leps$, we have $l^k_h(s) - 1 \leq \min\{\Leps, l^k_h(s) - 1\}$. Thus,
    \begin{align} \label{eq:all-optimal-second-1}
        \max_{a\in \cA}[\BB_h \hat{V}^k_{h+1}](s, a) - [\BB_h \hat{V}^k_{h+1}](s, \pi^{k}_h(s)) &\leq 8 \cdot 2^{-(l^k_h(s)-1)} + 2l^k_h(s) \cdot \chi\sqrt{\Leps}\zeta \notag \\
        &\leq 16 \cdot 2^{-l^k_h(s)} + 2\chi\Leps^{1.5}\zeta \notag \\
        &\leq 16 \cdot 2^{-l^k_h(s)} + 0.02\eps/H,
    \end{align}
    where the first inequality follows from Lemma~\ref{lm:all-optimal-pre} with $\pi^k_h(s) \in \cA^k_{h,l^k_h(s)}(s)$ and the last inequality holds due to $\chi\Leps^{1.5}\zeta \leq 0.01\eps /H$ given by $\cG_\eps$.    
    
    Otherwise, we have $\Leps \leq \min\{\Leps, l^k_h(s) - 1\}$. In this case, we have
    \begin{align} \label{eq:all-optimal-second-2}
        \max_{a\in \cA}[\BB_h \hat{V}^k_{h+1}](s, a) - [\BB_h \hat{V}^k_{h+1}](s, \pi^{k}_h(s)) &\leq 8 \cdot 2^{-\Leps} + 2\chi\Leps^{1.5}\zeta \notag\\
        &\leq 0.08 \eps/H + 0.02 \eps/H
        = 0.10\eps/H,
    \end{align}
    where the first inequality follows from Lemma~\ref{lm:all-optimal-pre} with $\pi^k_h(s) \in \cA^k_{h,l^k_h(s)}(s) \subseteq \cA^k_{h,\Leps}(s)$ according to the elimination routine and the final inequality holds due to $\chi\Leps^{1.5}\zeta \leq 2^{-\Leps} \leq 0.01\eps /H$ given by $\cG_\eps$. So by combining \eqref{eq:all-optimal-second-1} and \eqref{eq:all-optimal-second-2}, we have the desired statement that 
    \begin{align*}
        \max_{a\in \cA}[\BB_h \hat{V}^k_{h+1}](s, a) - [\BB_h \hat{V}^k_{h+1}](s, \pi^{k}_h(s)) \leq 16 \cdot 2^{-l^k_h(s)} + 0.10\eps/H.
    \end{align*}
\end{proof}
\subsection{Proof of Lemma~\ref{lm:final-value-func-upper}}
We prove Lemma~\ref{lm:final-value-func-upper} in this section by applying Lemma~\ref{lm:value-func-upper} on phase $\min\{\Leps, l^k_h(s) - 1\}$.

\begin{proof}[Proof of Lemma~\ref{lm:final-value-func-upper}]
    If $l^k_h(s) \leq \Leps$, we have $l^k_h(s) - 1 \leq \min\{\Leps, l^k_h(s) - 1\}$. Firstly, we have 
    \begin{align} \label{eq:final-value-func-upper-1-pre}
        \hat{V}^k_h(s) \leq \hat{V}^k_{h,l^k_h(s)-1}(s) \leq V^k_{h,l^k_h(s)-1}(s) + 3 \cdot 2^{-(l^k_h(s)-1)}.
    \end{align}
    where the first inequality is given by Lemma~\ref{lm:final-value-func-pre} and the second inequality follows from the definition of $\hat{V}^k_{h,l^k_h(s)-1}(s)$. This leads to 
    \begin{align} \label{eq:final-value-func-upper-1}
        \hat{V}^k_h(s) - \max_{a\in \cA}[\BB_h \hat{V}^k_{h+1}](s, a) &\leq \big(\hat{V}^k_h(s) - V^k_{h,l^k_h(s)-1}(s)\big) + \big(V^k_{h,l^k_h(s)-1}(s) -  \max_{a\in \cA}[\BB_h \hat{V}^k_{h+1}](s, a)\big) \notag \\
        &\leq 3 \cdot 2^{-(l^k_h(s)-1)} + 2 \cdot 2^{-(l^k_h(s)-1)}  + \chi\sqrt{\Leps}\zeta \notag \\
        &\leq 10 \cdot 2^{-l^k_h(s)}  + 0.01\eps/H,
    \end{align}
    where in the second inequality, the first term is given by \eqref{eq:final-value-func-upper-1-pre} and the second term holds according to Lemma~\ref{lm:value-func-upper}, and the third inequality holds from $\chi\sqrt{\Leps}\zeta \leq 0.01\eps/H$ given by $\cG_\eps$.
    
    Otherwise, we have $\Leps \leq \min\{\Leps, l^k_h(s) - 1\}$, this leads to 
    \begin{align} \label{eq:final-value-func-upper-2}
        \hat{V}^k_h(s) - \max_{a\in \cA}[\BB_h \hat{V}^k_{h+1}](s, a) &\leq \big(\hat{V}^k_h(s) - V^k_{h,\Leps}(s)\big) + \big(V^k_{h,\Leps}(s) -  \max_{a\in \cA}[\BB_h \hat{V}^k_{h+1}](s, a)\big) \notag \\
        &\leq 3 \cdot 2^{-\Leps}  + 2 \cdot 2^{-\Leps} + \chi\sqrt{\Leps}\zeta \notag \\
        &\leq 0.03\eps/H + 0.02\eps/H + 0.01\eps/H = 0.06\eps/H,
    \end{align}
    where in the second inequality, the first term is given by the definition of $\hat{V}^k_h(s)$ and the second term holds according to Lemma~\ref{lm:value-func-upper}, and the third inequality holds from $\chi\Leps^{1.5}\zeta \leq 2^{-\Leps} \leq 0.01\eps /H$ given by $\cG_\eps$. Combining \eqref{eq:final-value-func-upper-1} and \eqref{eq:final-value-func-upper-2} gives the desired statement
    \begin{align*}
        \hat{V}^k_h(s) - \max_{a\in \cA}[\BB_h \hat{V}^k_{h+1}](s, a) \leq 10 \cdot 2^{-l^k_h(s)} + 0.06\eps/H.
    \end{align*}
\end{proof}
\subsection{Proof of Lemma~\ref{lm:one-step-regret}}
\begin{proof}[Proof of Lemma~\ref{lm:one-step-regret}]
According to the definition in which $V^{\pi^k}_h(s^k_h) = [\BB_h V^{\pi^k}_{h+1}](s^k_h, \pi^k_h(s^k_h))$ and $\eta^k_h + [\PP_h(\hat{V}^k_{h+1} - V^{\pi^k}_{h+1})] (s^k_{h}, \pi^k_{h}(s^k_{h})) - \big(\hat{V}^k_{h+1}(s^k_{h+1}) -  V^{\pi^k}_{h+1}(s^k_{h+1}) \big)$. We can write 
\begin{align*}
    \hat{V}^k_h(s^k_h) - V^{\pi^k}_h(s^k_h) &= \big(\hat{V}^k_h(s^k_h) - [\BB_h \hat{V}^k_{h+1}](s^k_h, \pi^k_h(s^k_h))\big) + \eta^k_h +  \big(\hat{V}^k_{h+1}(s^k_{h+1}) -  V^{\pi^k}_{h+1}(s^k_{h+1}) \big).
\end{align*}
By a telescoping statement from $h$ to $H$ with the final terminal value $\hat{V}^k_{H+1}(\cdot) = V^{\pi^k}_{H+1}(\cdot) = 0$, we reach
\begin{align} \label{eq:one-step-regret-decomposition}
    \hat{V}^k_h(s^k_h) - V^{\pi^k}_h(s^k_h) &= \sum_{h'=h}^{H} \big(\hat{V}^k_h(s^k_h) - [\BB_h \hat{V}^k_{h+1}](s^k_h, \pi^k_h(s^k_h))\big) + \sum_{h'=h}^{H} \eta^k_{h'}.
\end{align}
As a result, we can bound the desired term by
\begin{align*}
    V^*_h(s^k_h) - V^{\pi^k}_h(s^k_h) &\leq \hat{V}^k_h(s^k_h) - V^{\pi^k}_h(s^k_h) + 0.07\eps \\
    &=  \sum_{h'=h}^{H} \big(\hat{V}^k_h(s^k_h) - [\BB_h \hat{V}^k_{h+1}](s^k_h, \pi^k_h(s^k_h))\big) + \sum_{h'=h}^{H} \eta^k_{h'} + 0.07\eps \\
    &\leq \sum_{h'=h}^{H} \big(26 \cdot 2^{-l^k_{h'}(s^k_{h'})} + 0.16\eps/H\big) + \sum_{h'=h}^{H} \eta^k_{h'} + 0.07\eps \\
    &= 0.23\eps + 26\sum_{h'=h}^H 2^{-l^k_{h'}(s^k_{h'})} + \sum_{h'=h}^{H} \eta^k_{h'}.
    \end{align*}
    where the first inequality is given by Lemma~\ref{lm:final-value-func-lower}, the first equality is given by \eqref{eq:one-step-regret-decomposition}, and the final inequality is given by Lemma~\ref{lm:local-overestimation}.
\end{proof}

\section{Technical Numerical Lemmas}
\begin{lemma} \label{lm:level-size-bound-auilixary}
    If $|\cC^k_{h,l}| \leq 4^ld + 2.5\cdot 4^l\gamma_{l}^2d\ln\big(1 + |\cC^k_{h,l}|/(16d)\big)$, then $|\cC^k_{h,l}| \leq 16l \cdot 4^l\gamma_{l}^2d$.
\end{lemma}
\begin{proof}
    Denote $c = |\cC^k_{h,l}| / (l \cdot 4^l\gamma_{l}^2d)$. We have that 
    \begin{align*}
        cl \cdot 4^l\gamma_{l}^2d \leq 4^ld + 2.5 \cdot 4^l\gamma_{l}^2d \ln(1 + cl \cdot 4^l\gamma_{l}^2/16).
    \end{align*}
    Dividing both sides by $4^l\gamma_{l}^2d$, we have that 
    \begin{align*}
        cl &\leq 1/\gamma_{l}^2 + 2.5 \ln(1 + cl \cdot 4^l\gamma_{l}^2/16) \notag \\
                &\leq 1/\gamma_{l}^2 + 2.5 \ln(4c \cdot 5^l\gamma_{l}^2/16) \leq 1/\gamma_{l}^2 + 4.1l + 2.5 \ln(c).
    \end{align*}
    Since $l \geq 1$ and $\gamma_l \geq 1$, we can further conclude that 
    \begin{align*}
        c \leq 5.1 + 2.5 \ln(c) \leq 5.1 + 2.5(1 + c/6).
    \end{align*}
    The necessary condition for the above inequality is $c \leq 16$, which proves the desired statement.
\end{proof}

\begin{lemma} \label{lm:gamma-monotone}
    For any $l \geq 1$, $\gamma_{l+1}/\gamma_l\leq 1.4$.
\end{lemma}
\begin{proof}
    Firstly, we have that 
    \begin{align} \label{eq:gamma-monotone-1}
    \frac{l+22+\log(l+1)}{l+20+\log(l)} \leq \frac{l+22+0.2l+2}{l+20} = 1.2,
    \end{align} 
    where the first inequality holds due to $\log(x+1) \leq 0.2x+2$. In addition, we have 
    \begin{align} \label{eq:gamma-monotone-2}
    \frac{4+\log(l+1)}{4+\log(l)} \leq \frac{4+\log(l)+1}{4+\log(l)} \leq 1.25,
    \end{align}
    where the first inequality holds due to $\log(x+1)\leq \log(x)+1$. As a result, we can reach the desired statement according to 
    \begin{align*}
        \frac{\gamma_{l+1}}{\gamma_{l}} &= \frac{5(l+1+\lceil20+\log((l+1)d)\rceil)dH\sqrt{\log(16(l+1)dH/\delta)}}{5(l+\lceil20+\log(ld)\rceil)dH\sqrt{\log(16ldH/\delta)}} \\
        &\leq \frac{l+22+\log(l+1)+\log(d)}{l+20+\log(l)+\log(d)} \cdot \sqrt{\frac{\log(l+1)+\log(16dH/\delta)}{\log(l)+\log(16dH/\delta)}} \\
        &\leq \frac{l+22+\log(l+1)}{l+20+\log(l)} \cdot \sqrt{\frac{\log(l+1)}{\log(l)}} \\
        &\leq 1.2\sqrt{1.25} \\
        &\leq 1.4,
    \end{align*}
    where the third inequality holds from plugging both \eqref{eq:gamma-monotone-1} and \eqref{eq:gamma-monotone-2}.
\end{proof}

\begin{lemma} \label{lm:concentration-V-auxiliary}
    \begin{align*}
        \sqrt{2d \ln(1 + l \cdot 4^l \gamma_l^2) + 2 \ln(l^2 H (2^{22}d^6H^4)^{\Lplus ^2d^2}/\delta)} \leq \gamma_{l, \Lplus}
    \end{align*}
\end{lemma}
\begin{proof}
    By calculation, we have that 
    \begin{align*}
        &\quad H\sqrt{2d \ln(1 + l \cdot 4^l\gamma_l^2) + 2 \ln(l^2 H (2^{22}d^6H^4)^{\Lplus^2d^2}/\delta)} \\
        &\leq H\sqrt{2d\ln(1 + l \cdot 4^l \cdot 1.4^{2l} \gamma_1^2)} + H\sqrt{12\Lplus ^2d^2 \ln(2^4ldH/\delta)} \\
        &\leq \Lplus d H\sqrt{2\ln(2^4ldH/\delta)} + \Lplus d H\sqrt{12\ln(2^4ldH/\delta)} \\
        &\leq 5\Lplus dH\sqrt{\log(2^4\gamma_{\Lplus}ldH/\delta)} \\
        &= \gamma_{l, \Lplus}.
    \end{align*}
\end{proof}
\begin{lemma} \label{lm:main-thm-auilixary-1}
    If some constant $c_1, c_2 > 0$ that 
    \begin{align*}
        |\Kepsh|  < c_1 \Leps(\Leps+\log(dH))^2 d^3 H^4 \eps^{-2}\log(\Leps d/\delta) + \eps^{-1}\sqrt{c_2 H^3|\Kepsh|\log (H|\Kepsh|\log(\eps^{-1})/\delta)}.
    \end{align*}
    Then, there exists $c_3 > 0$ such that 
    \begin{align*}
        |\Kepsh| <  c_3 \Leps(\Leps+\log(dH))^2 d^3 H^4 \eps^{-2} \log(\Leps d) \log(\delta^{-1})\iota,
    \end{align*}
    where $\iota$ is a polynomial of $\log\log(\Leps dH\delta^{-1})$.
\end{lemma}
\begin{proof}
    Let $x = |\Kepsh|/\log(|\Kepsh|)$. We have that 
    \begin{align*}
        x < c_1 \Leps(\Leps+\log(dH))^2 d^3 H^4 \eps^{-2} \log(\Leps d/\delta) + \eps^{-1}\sqrt{c_2 H^3  x  \log (H\log(\eps^{-1})/\delta)}.
    \end{align*} 
    Since $x < a + \sqrt{bx}$ implies $x < 2a + 2b$, so the above inequality implies
    \begin{align*}
        x < 2c_1 \Leps(\Leps+\log(dH))^2 d^3 H^4 \eps^{-2} \log(\Leps d/\delta) + 2c_2 H^3 \eps^{-2} \log (H\log(\eps^{-1})/\delta).
    \end{align*} 
    Moreover, since $y/\log(y) < a$ implies $y < 2a\log a$, we can conclude that there exists $c_3 > 0$ that 
    \begin{align*}
        |\Kepsh| < c_3 \Leps(\Leps+\log(dH))^2 d^3 H^4 \eps^{-2} \log(\Leps d) \log(\delta^{-1})\iota,
    \end{align*}
    where $\iota$ is a polynomial of $\log\log(\Leps dH\eps^{-1}\delta^{-1})$.
\end{proof}
\section{Auxiliary Lemmas}
This section provides some auxiliary concentration lemmas frequently used in the proof. 
\begin{lemma}[Lemma~11, \citet{abbasi2011improved}] \label{lm:vector-potential}
    Let $\{\bphi^k\}_{k=1}^{\infty}$ be any bounded sequence such that $\bphi^k \in \RR^d$ and $\|\bphi^k\|_2 \leq B$ for some constant $B > 0$. For $k \geq 1$, let $\Ub^k = \lambda \Ib + \sum_{\tau=1}^{k-1} \bphi^\tau (\bphi^\tau)^\top$. Let $\lambda > 0$, then for all $k \in [K]$, we have that 
    \begin{align*}
        \sum_{\tau=1}^{k} \min\big\{1, \|\bphi^\tau\|^2_{(\Ub^\tau)^{-1}}\big\} \leq 2d \ln \big(1 + kB^2/(d\lambda)\big).
    \end{align*}
\end{lemma}
\begin{lemma}[Self-Normalized Martingale, \citet{abbasi2011improved}] \label{lm:hoeffding-concentration}
    Let $\{\cF^k\}_{k=1}^{\infty}$ be a filtration, and $\{\bphi^k, \eta^k\}_{k=1}^{\infty}$ be a stochastic process where $\bphi^k \in \RR^d$ is $\cG^k$-measurable and $\eta^k$ is $\cF^{k+1}$-measurable such that 
    \begin{align*}
        |\EE[\eta^k|\cF^k]| = 0, |\eta^k| \leq R, \|\bphi^k\|_2 \leq B
    \end{align*}
    for some constant $B, R > 0$. Let $\lambda > 0$. For $k\geq 1$, let $\Ub^k = \lambda \Ib + \sum_{\tau=1}^{k-1} \bphi^\tau (\bphi^\tau)^\top$.
    Then for any $\delta \in (0, 1)$, with probability at least $1 - \delta$, for all $k \geq 1$, we have that 
    \begin{align*}
        \Bigg \|\sum_{\tau = 1}^{k-1} \eta^\tau \bphi^\tau  \Bigg\|_{(\Ub^k)^{-1}} 
        \leq  R\sqrt{2d \ln \big(1 + kB^2/(d\lambda)\big) + 2\ln \delta^{-1}}.
    \end{align*}
\end{lemma}
\begin{lemma}[Lemma~8, \citet{zanette2020learning}] \label{lm:misspecify}
    Let $\{\bphi^k, \eta^k\}_{k=1}^{\infty}$ be any bounded sequence satisfying $\bphi^k \in \RR^d$ and $|\eta^k| \leq \zeta$ for some constant $\zeta > 0$. For $k\geq 1$, let $\Ub^k = \lambda \Ib + \sum_{\tau=1}^{k-1} \bphi^\tau (\bphi^\tau)^\top$.
    Then, for all $k \geq 1$, we have that 
    \begin{align*}
        \Bigg \|\sum_{\tau = 1}^{k-1} \eta^\tau \bphi^\tau  \Bigg\|_{(\Ub^k)^{-1}} 
        \leq \zeta\sqrt{k}.
    \end{align*}
\end{lemma}
\begin{lemma}[Azuma–Hoeffding inequality, \citet{Hoeffding1963ProbabilityIF}] \label{lm:azuma–hoeffding}
    Let $\{\eta^k\}_{k=1}^{K}$ be a martingale difference sequence with respect to a filtration $\{\cF^k\}_{k=1}^{K}$ satisfying $|\eta^k|\leq M$ for some constant $M > 0$ and $\eta^k$ is $\cF^{k+1}$-measurable with $|\EE[\eta^k|\cF^k]| = 0$. Then for some fixed $k\in[K]$ and any $\delta \in (0, 1)$, with probability at least $1 - \delta$, we have 
    \begin{align*}
        \sum_{\tau=1}^k \eta^\tau \leq M\sqrt{2k \ln \delta^{-1} }.
    \end{align*}
\end{lemma}
\begin{lemma}[Freedman inequality, \citet{cesa2006prediction}] \label{lm:freedman}
    Let $\{\eta^k\}_{k=1}^{K}$ be a martingale difference sequence with respect to a filtration $\{\cF^k\}_{k=1}^{K}$ satisfying $|\eta^k|\leq M$ for some constant $M > 0$ and $\eta^k$ is $\cF^{k+1}$-measurable with $|\EE[\eta^k|\cF^k]| = 0$. Then for some fixed $k\in[K]$, $a > 0$ and $v > 0$, we have
    \begin{align*}
        \Pr\Big(\sum_{\tau=1}^k \eta^\tau \geq a, \sum_{\tau=1}^k \Var[\eta^\tau|\cF^\tau] \leq v\Big) \leq \exp \Big( \frac{-a^2}{2v + 2aM/3} \Big).
    \end{align*}
\end{lemma}
\section{Numerical Simulation}
We added experiments on synthetic datasets to verify the performance of the algorithm and the contribution of each component. Specifically, we consider a linear MDP with $S = 4$, $A = 5$, $H = 2$, and $d = 8$. Each element in the feature vector $\bphi(s, a)$ and $\bmu(s')$ is generated by a uniform distribution $U(0, 1)$. Subsequently, $\bphi$ is normalized to ensure that $\PP(s' | s, a)$ is a probability measure, i.e., $\bphi(s, a) = \bphi(s, a) / \sum_{s'} \bphi^\top(s, a) \bmu(s')$. The reward is defined by $r(s, a) = \bphi^\top(s, a) \btheta$, where $\btheta \sim N(0, I_d)$. The model misspecification is also added to the transition $\PP$ and reward function $r$. For a given misspecification $\zeta$, the ground truth reward function is defined by $r(s, a) = \phi^\top(s, a) \theta + Z(s, a)$, where $Z(s, a) \sim U(-\zeta, \zeta)$. When adding the model misspecification to the transition kernel, we first random sample a subset $\mathcal S_+ \subset \mathcal S$ such that $|\mathcal S_+| = |S| / 2$. Then the misspecified transition kernel is then generated by
\begin{align}
    \PP'(s' | s, a) = \PP(s' | s, a) + 2\frac\zeta S\ind[s' \in \mathcal S_+] - \frac\zeta S, \notag 
\end{align}
we can verify that $\lVert \PP(\cdot | s, a) - \PP'(\cdot | s, a)\rVert_{\mathrm{TV}} = \zeta$. We investigated the misspecification level from $\zeta = 0, 0.01, \cdots, 0.3$ in $16$ randomly generated environments over $2000$ episodes. We report the cumulative regret and runtime with respect to different misspecification levels. Additionally, we performed an ablation study by 1) removing the certified estimation (Algorithm~\ref{alg:Lin}, Line~\ref{ln:cond3}) and 2) removing the quantization (Algorithm~\ref{alg:LSVI}, Line~\ref{ln:quanti}).

The results of these configurations are presented in the following table. The detailed regret for all misspecification level is presented in Table~\ref{tab:1} and Figure~\ref{fig:1}, we plot the cumulative regret for 2000 episodes with respect to the misspecification level $\zeta$. The cumulative regret curve is plotted in Figure~\ref{fig:2}.

The experimental results suggest several key findings that support our theoretical analysis:
\begin{itemize}[leftmargin=*,nosep]
    \item When the misspecification level is low, it is possible to achieve constant regret, where the instantaneous regret in the final rounds is approximately zero.
    \item The certified estimator and the quantization do not significantly affect the algorithm's runtime. In contrast, the certified estimator provides an `early-stopping' condition in Algorithm~\ref{alg:Lin}, which slightly reduces the algorithm's runtime. In particular, our algorithm yields a computational complexity of $O(d^2AHK^2\log K)$, which is the same as~\citet{vial2022improved} and only $\log K$ greater than the vanilla LSVI-UCB~\citep{jin2020provably} due to the multi-phased algorithm.
    \item The certified estimator helps the algorithm by providing robust estimation in the presence of misspecification. As shown in the table, using the certified estimator does not make a significant difference when the misspecification level $\zeta$ is low, but it becomes significant as the misspecification level increases.
    \item The quantization does not contribute significantly to the results, as the numerical results are intrinsically discrete and quantized. In Figure~\ref{fig:2}, the regret curve with quantization and the one without quantization are highly overlapped.
\end{itemize}
\begin{table}[!htbp]
\centering
\resizebox{\textwidth}{!}{%
\begin{tabular}{|c|c|c|c|c|c|c|c|c|c|c|c|}
\hline 
C. & Q. & $\zeta = 0.0$ & 0.01 & 0.02 & 0.03 & 0.04\\
\hline
$\times$ & $\times$ & 189.01 $\pm$ \small{57.21} & 190.57 $\pm$ \small{59.49} & \textbf{194.74} $\pm$ \small{61.99} & \textbf{201.36} $\pm$ \small{63.67} & \textbf{211.67} $\pm$ \small{66.06} \\
$\times$ & $\checkmark$ & \textbf{189.00} $\pm$ \small{57.18} & \textbf{190.55} $\pm$ \small{59.49} & 194.75 $\pm$ \small{61.99} & 201.41 $\pm$ \small{63.67} & 211.68 $\pm$ \small{66.06}\\
$\checkmark$ & $\times$ & 196.32 $\pm$ \small{64.09} & 203.68 $\pm$ \small{77.24} & 200.41 $\pm$ \small{66.88} & 207.47 $\pm$ \small{70.97} & 222.61 $\pm$ \small{78.77}\\
$\checkmark$ & $\checkmark$ & 196.31 $\pm$ \small{64.06} & 199.07 $\pm$ \small{68.45} & 200.42 $\pm$ \small{66.88} & 207.52 $\pm$ \small{70.96} & 222.62 $\pm$ \small{78.77}\\
\hline
\end{tabular}
}
\vskip 0.5em
\centering
\resizebox{\textwidth}{!}{%
\begin{tabular}{|c|c|c|c|c|c|c|c|c|c|c|c|}
\hline
C. & Q. & $\zeta = 0.05$ & 0.06 & 0.07 & 0.08 & 0.09\\
\hline
$\times$ & $\times$ & 220.52 $\pm$ \small{67.85} & 233.57 $\pm$ \small{66.54} & 248.08 $\pm$ \small{68.70} & 264.49 $\pm$ \small{72.05} & 279.82 $\pm$ \small{78.28}\\
$\times$ & $\checkmark$ & \textbf{220.50} $\pm$ \small{67.86} & 233.48 $\pm$ \small{66.50} & \textbf{248.06} $\pm$ \small{68.69} & 264.56 $\pm$ \small{72.00} & 279.81 $\pm$ \small{78.25}\\
$\checkmark$ & $\times$ & 221.87 $\pm$ \small{73.51} & 232.75 $\pm$ \small{73.75} & 261.16 $\pm$ \small{84.44} & \textbf{262.07} $\pm$ \small{82.05} & 273.56 $\pm$ \small{96.19}\\
$\checkmark$ & $\checkmark$ & 221.86 $\pm$ \small{73.51} & \textbf{232.70} $\pm$ \small{73.81} & 261.14 $\pm$ \small{84.43} & 262.15 $\pm$ \small{82.01} & \textbf{273.55} $\pm$ \small{96.17}\\
\hline
\end{tabular}
}
\vskip 0.5em
\centering
\resizebox{\textwidth}{!}{
\begin{tabular}{|c|c|c|c|c|c|c|c|c|c|c|c|}
\hline
C. & Q. & $\zeta=$0.1 & 0.11 & 0.12 & 0.13 & 0.14 \\
\hline
$\times$ & $\times$ & 292.59 $\pm$ \small{80.52} & 305.73 $\pm$ \small{83.47} & 323.93 $\pm$ \small{90.25} & 337.65 $\pm$ \small{94.39} & 355.49 $\pm$ \small{106.22}\\
$\times$ & $\checkmark$ & 292.68 $\pm$ \small{80.55} & 305.71 $\pm$ \small{83.44} & 323.91 $\pm$ \small{90.23} & 337.65 $\pm$ \small{94.39} & 355.49 $\pm$ \small{106.22}\\
$\checkmark$ & $\times$ & \textbf{285.81} $\pm$ \small{102.54} & 297.98 $\pm$ \small{107.21} & 315.22 $\pm$ \small{114.50} & \textbf{335.69} $\pm$ \small{110.76} & \textbf{339.81} $\pm$ \small{99.26}\\
$\checkmark$ & $\checkmark$ & 285.90 $\pm$ \small{102.57} & \textbf{297.97} $\pm$ \small{107.18} & \textbf{315.20} $\pm$ \small{114.49} & 335.69 $\pm$ \small{110.76} & 339.82 $\pm$ \small{99.27}\\
\hline
\end{tabular}
}
\vskip 0.5em
\centering
\resizebox{\textwidth}{!}{
\begin{tabular}{|c|c|c|c|c|c|c|c|c|c|c|c|}
\hline
C. & Q. & $\zeta=$0.15 & 0.2 & 0.25 & 0.3 & Time(s) \\
\hline
$\times$ & $\times$ & 377.27 $\pm$ \small{127.21} & 450.49 $\pm$ \small{154.80} & 526.45 $\pm$ \small{181.90} & 634.46 $\pm$ \small{245.70} & 1654.76 $\pm$ \small{125.40} \\
$\times$ & $\checkmark$ & 377.23 $\pm$ \small{127.13} & 450.48 $\pm$ \small{154.79} & \textbf{526.36} $\pm$ \small{181.91} & 634.48 $\pm$ \small{245.68} & 1654.21 $\pm$ \small{141.31} \\
$\checkmark$ & $\times$ & 351.52 $\pm$ \small{118.09} & 436.30 $\pm$ \small{154.89} & 530.15 $\pm$ \small{194.89} & \textbf{605.64} $\pm$ \small{233.19} & 1599.66 $\pm$ \small{138.18} \\
$\checkmark$ & $\checkmark$ & \textbf{351.50} $\pm$ \small{118.03} & \textbf{436.29} $\pm$ \small{154.91} & 530.58 $\pm$ \small{195.24} & 605.67 $\pm$ \small{233.18} & \textbf{1593.38} $\pm$ \small{98.11}\\
\hline
\end{tabular}
}
\vskip 0.5em
\caption{Average cumulative regret {\small{($\pm$ standard derivation)}} and execution time over 2000 episodes. The results are averaged over 16 individual runs. \textbf{C} indicates if Certified Estimator is used. \textbf{Q} indicates if Quantization is used.} \label{tab:1}
\end{table}
\begin{figure}
     \centering
     \begin{subfigure}[b]{0.49\textwidth}
         \centering
         \includegraphics[width=\textwidth]{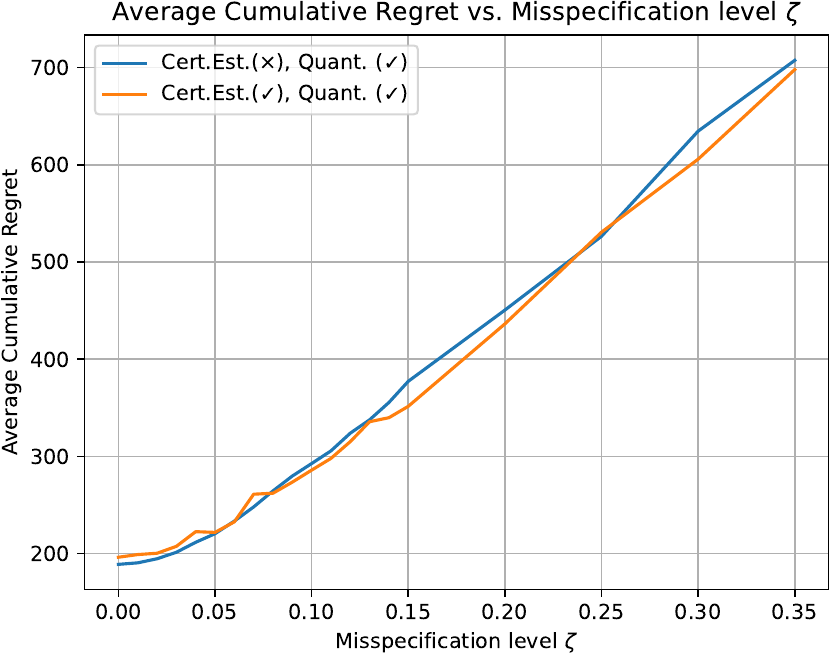}
         \caption{With quantization}
         \label{fig:y equals x}
     \end{subfigure}
     \hfill
     \begin{subfigure}[b]{0.49\textwidth}
         \centering
         \includegraphics[width=\textwidth]{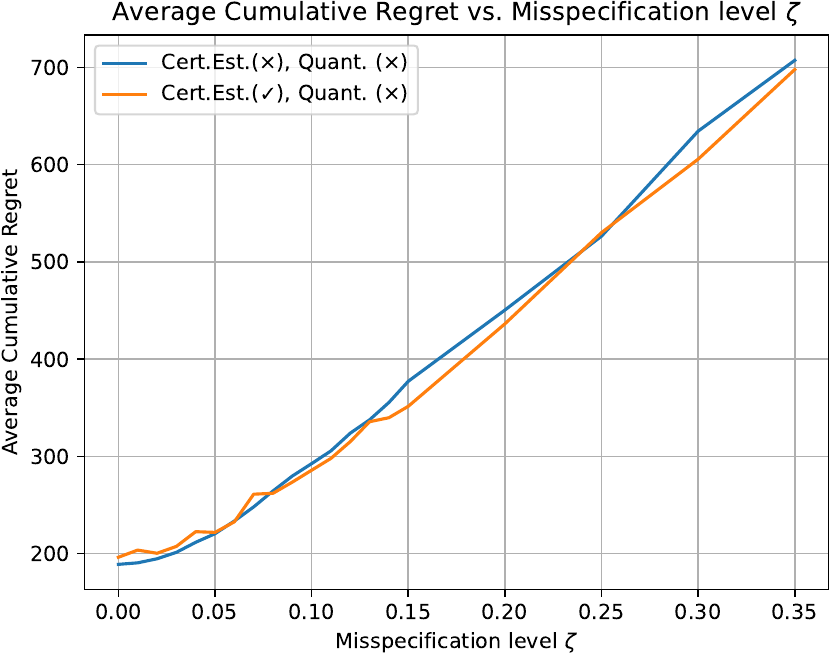}
         \caption{Without quantization}
         \label{fig:three sin x}
     \end{subfigure}
        \caption{Cumulative regret over 2000 episodes with respect to different misspecification level $\zeta$. The result is averaged over 16 individual environments. }
        \label{fig:1}
\end{figure}
\begin{figure}
     \centering
     \begin{subfigure}[b]{0.32\textwidth}
         \centering
         \includegraphics[width=\textwidth]{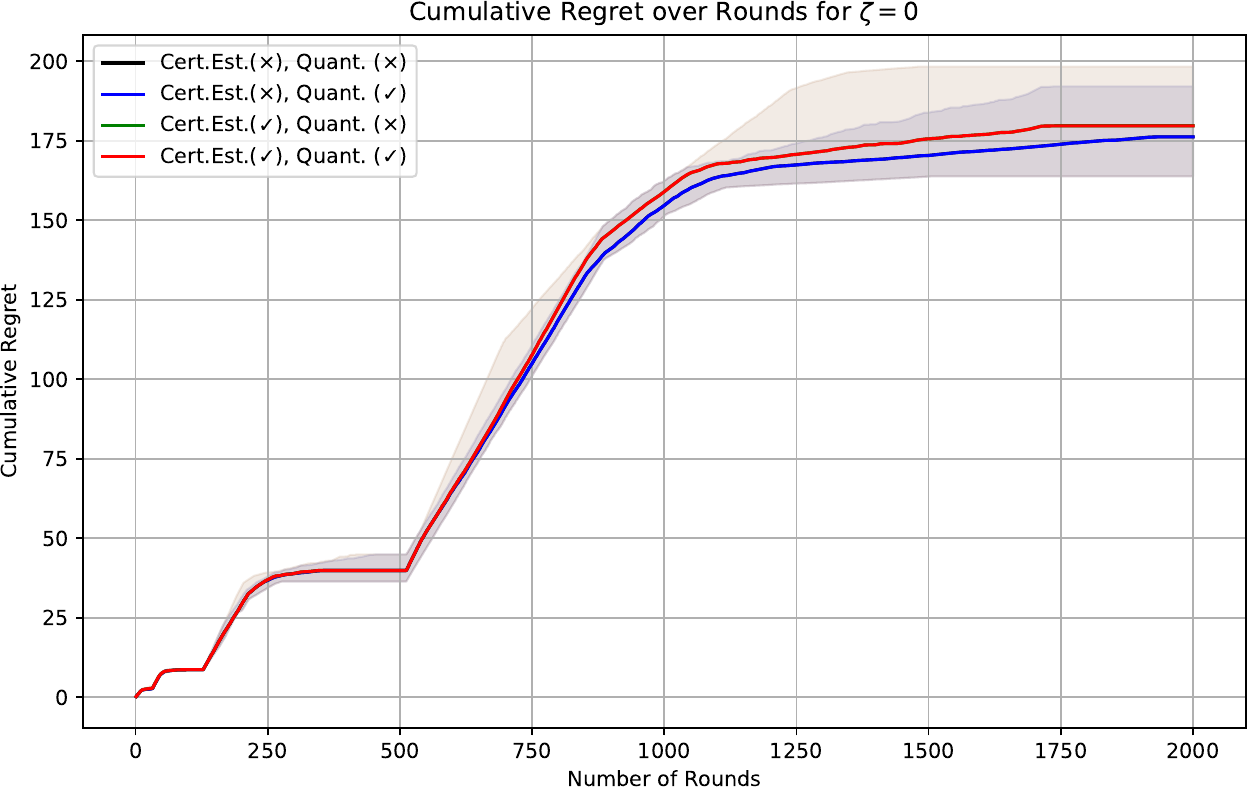}
         \caption{$\zeta = 0$}
     \end{subfigure}
     \hfill
     \begin{subfigure}[b]{0.32\textwidth}
         \centering
         \includegraphics[width=\textwidth]{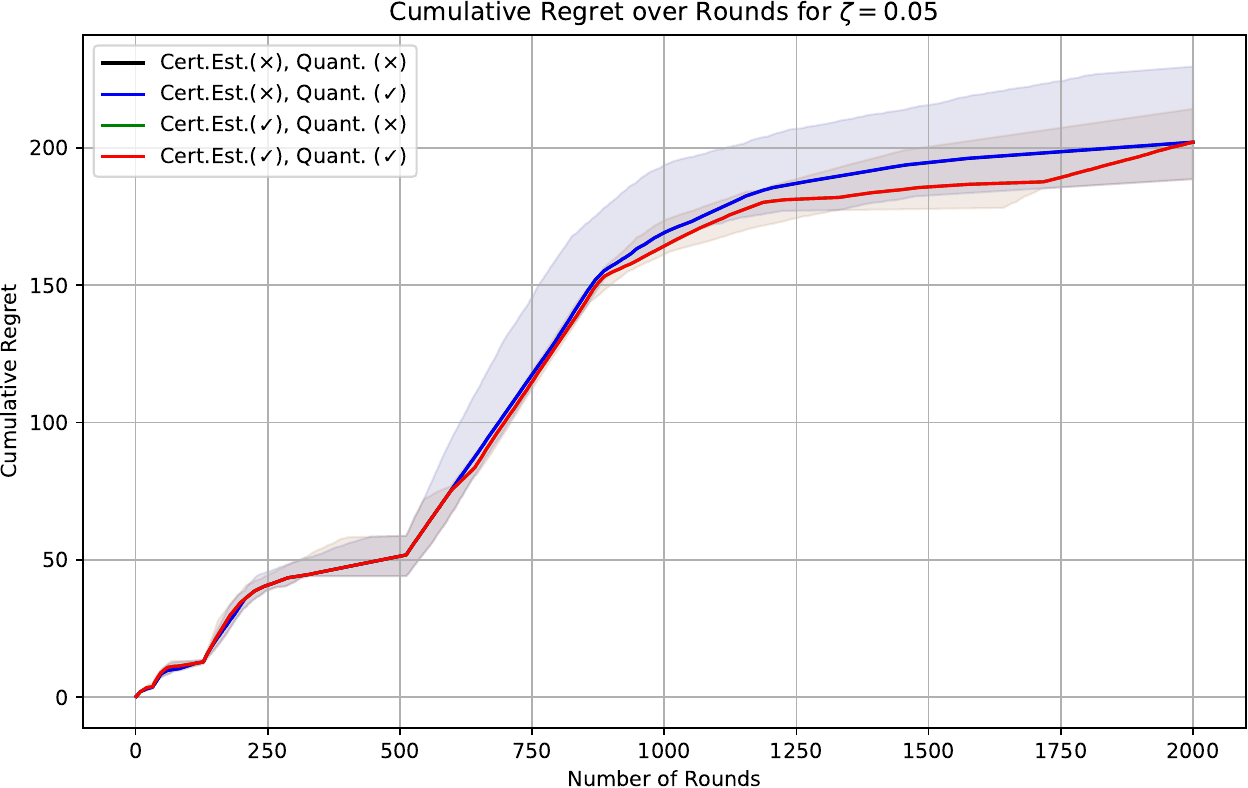}
         \caption{$\zeta = 0.05$}
     \end{subfigure}
     \hfill
     \begin{subfigure}[b]{0.32\textwidth}
         \centering
         \includegraphics[width=\textwidth]{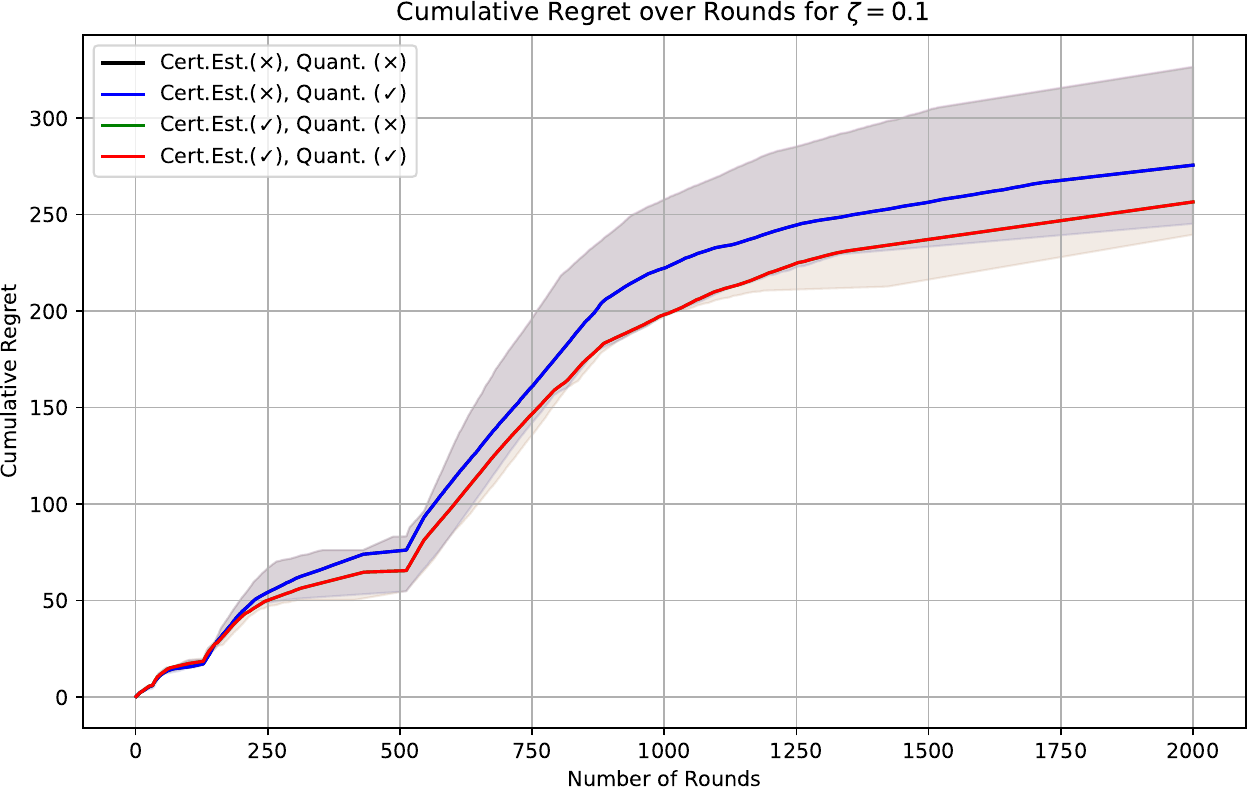}
         \caption{$\zeta = 0.10$}
     \end{subfigure}
     \hfill
     \begin{subfigure}[b]{0.32\textwidth}
         \centering
         \includegraphics[width=\textwidth]{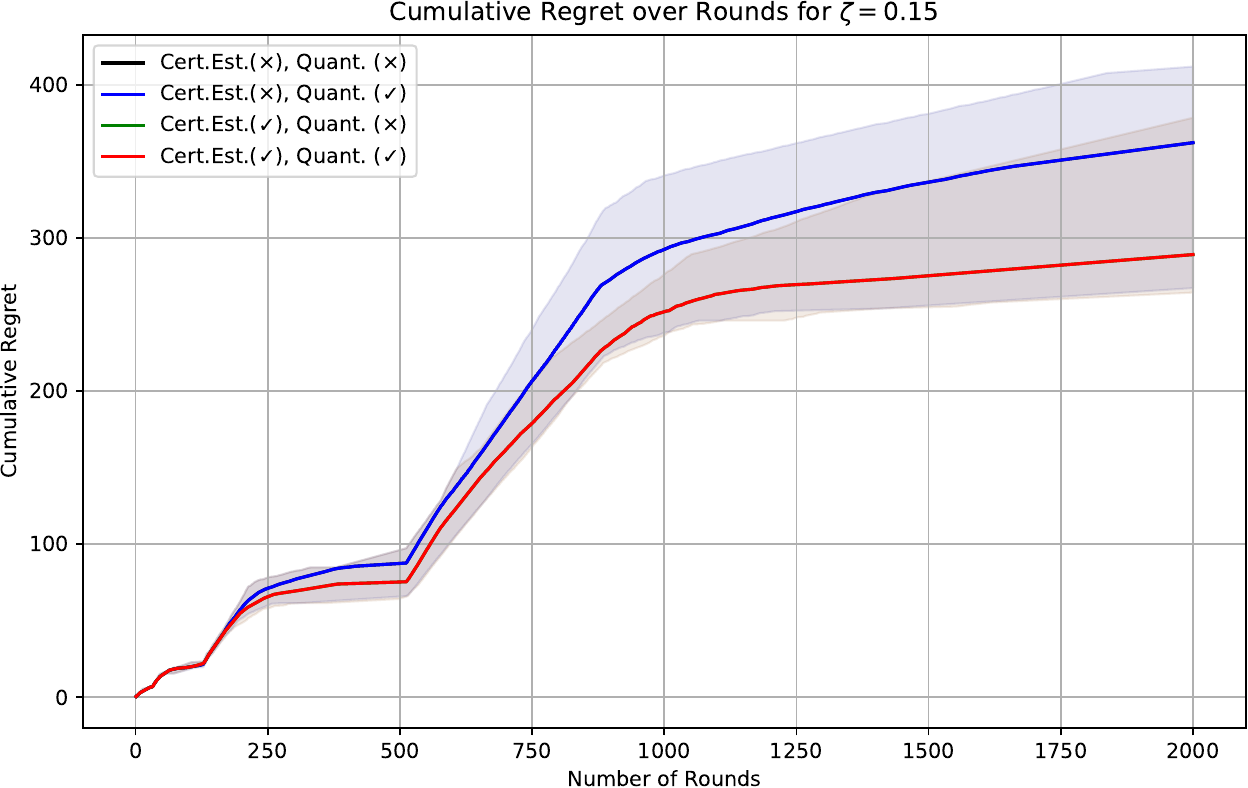}
         \caption{$\zeta = 0.15$}
     \end{subfigure}
     \hfill
     \begin{subfigure}[b]{0.32\textwidth}
         \centering
         \includegraphics[width=\textwidth]{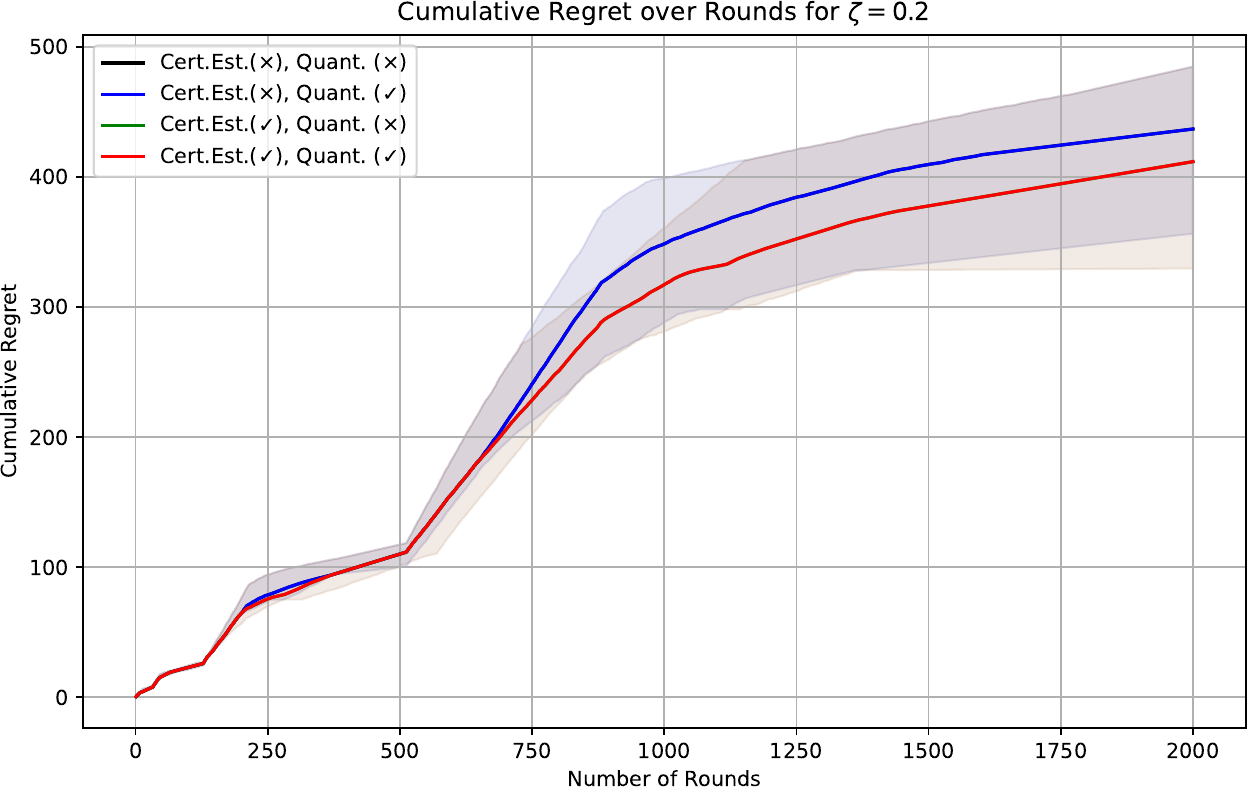}
         \caption{$\zeta = 0.20$}
     \end{subfigure}
     \hfill
     \begin{subfigure}[b]{0.32\textwidth}
         \centering
         \includegraphics[width=\textwidth]{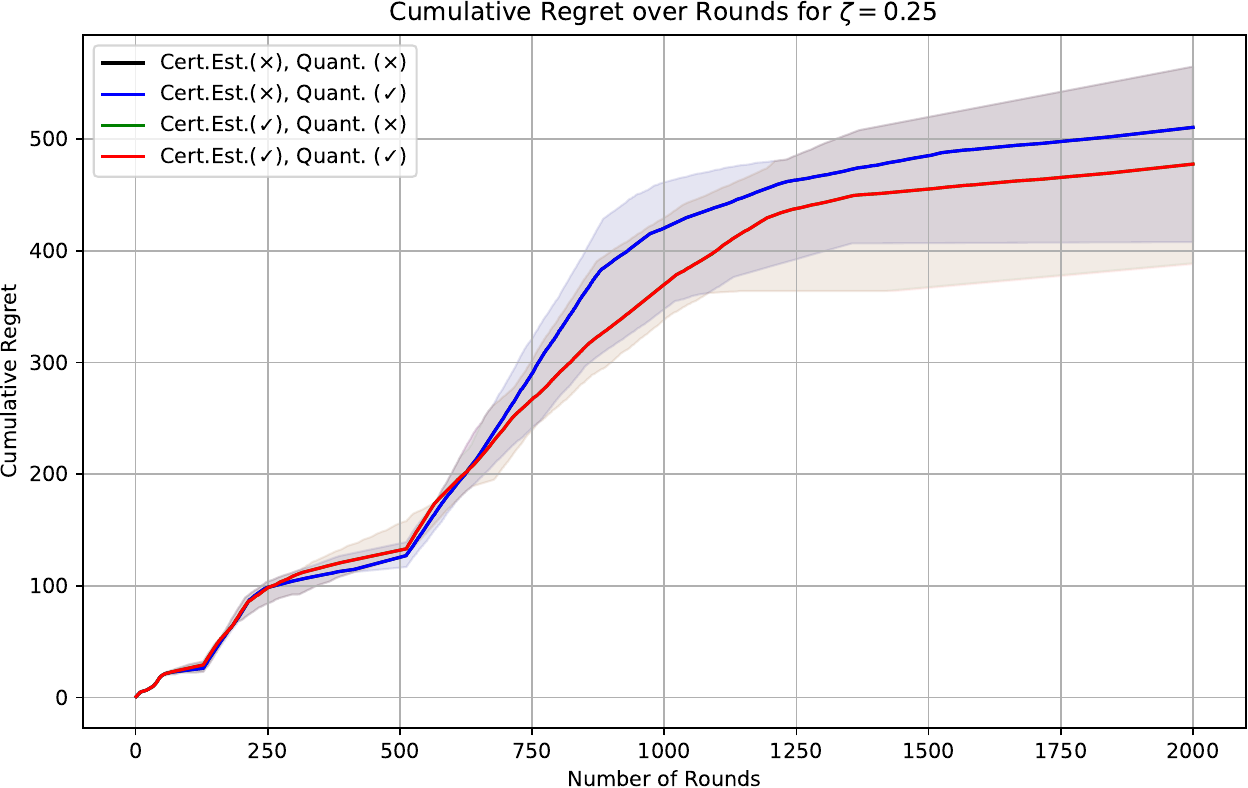}
         \caption{$\zeta = 0.25$}
     \end{subfigure}
        \caption{Cumulative regret with respect to the number of episodes. We reported the median cumulative regret with the shadow area as the region from 25\% percentage to 75\% percentage statistics over 16 runs. }
        \label{fig:2}
\end{figure}
\newpage
\section*{NeurIPS Paper Checklist}
\begin{enumerate}

\item {\bf Claims}
    \item[] Question: Do the main claims made in the abstract and introduction accurately reflect the paper's contributions and scope?
    \item[] Answer: \answerYes{}
    \item[] Justification: We study the constant regret analysis in reinforcement learning from a theoretical perceptive. The contribution, assumptions and scope are clearly claimed in the abstract and introduction. 

\item {\bf Limitations}
    \item[] Question: Does the paper discuss the limitations of the work performed by the authors?
    \item[] Answer: \answerYes{}
    \item[] Justification: We discussed the limitations and future potential directions in Section~\ref{sec:conc}.

\item {\bf Theory Assumptions and Proofs}
    \item[] Question: For each theoretical result, does the paper provide the full set of assumptions and a complete (and correct) proof?
    \item[] Answer: \answerYes{}.
    \item[] Justification: We provide a detailed proof for all theorems in the appendix, starting from Appendix~\ref{app:proof}.

    \item {\bf Experimental Result Reproducibility}
    \item[] Question: Does the paper fully disclose all the information needed to reproduce the main experimental results of the paper to the extent that it affects the main claims and/or conclusions of the paper (regardless of whether the code and data are provided or not)?
    \item[] Answer: \answerNA{} 
    \item[] Justification: Our submission paper does not include experiments

\item {\bf Open access to data and code}
    \item[] Question: Does the paper provide open access to the data and code, with sufficient instructions to faithfully reproduce the main experimental results, as described in supplemental material?
    \item[] Answer: \answerNA{} 
    \item[] Justification: Our paper does not include experiments requiring code.

\item {\bf Experimental Setting/Details}
    \item[] Question: Does the paper specify all the training and test details (e.g., data splits, hyperparameters, how they were chosen, type of optimizer, etc.) necessary to understand the results?
    \item[] Answer: \answerNA{} 
    \item[] Justification: Our paper does not include experiments.

\item {\bf Experiment Statistical Significance}
    \item[] Question: Does the paper report error bars suitably and correctly defined or other appropriate information about the statistical significance of the experiments?
    \item[] Answer: \answerNA{} 
    \item[] Justification: The paper does not include experiments.

\item {\bf Experiments Compute Resources}
    \item[] Question: For each experiment, does the paper provide sufficient information on the computer resources (type of compute workers, memory, time of execution) needed to reproduce the experiments?
    \item[] Answer: \answerNA{} 
    \item[] Justification: The paper does not include experiments.
    
\item {\bf Code Of Ethics}
    \item[] Question: Does the research conducted in the paper conform, in every respect, with the NeurIPS Code of Ethics \url{https://neurips.cc/public/EthicsGuidelines}?
    \item[] Answer: \answerYes{} 
    \item[] Justification: We have reviewed the NeurIPS Code of Ethics.

\item {\bf Broader Impacts}
    \item[] Question: Does the paper discuss both potential positive societal impacts and negative societal impacts of the work performed?
    \item[] Answer: \answerNA{} 
    \item[] Justification: Our work provides the theoretical understanding of reinforcement learning. Although there might be some potential social impacts on reinforcement learning applications, according to the guidelines, we believe our result does not have a direct connection with these issues.
    
\item {\bf Safeguards}
    \item[] Question: Does the paper describe safeguards that have been put in place for responsible release of data or models that have a high risk for misuse (e.g., pretrained language models, image generators, or scraped datasets)?
    \item[] Answer: \answerNA{} 
    \item[] Justification: We place this paper on the theoretical understand of reinforcement learning thus the paper poses no such risks. 

\item {\bf Licenses for existing assets}
    \item[] Question: Are the creators or original owners of assets (e.g., code, data, models), used in the paper, properly credited and are the license and terms of use explicitly mentioned and properly respected?
    \item[] Answer: \answerNA{} 
    \item[] Justification: We do not use existing assets in the paper. 

\item {\bf New Assets}
    \item[] Question: Are new assets introduced in the paper well documented and is the documentation provided alongside the assets?
    \item[] Answer: \answerNA{} 
    \item[] Justification: This paper does not release new assets. 

\item {\bf Crowdsourcing and Research with Human Subjects}
    \item[] Question: For crowdsourcing experiments and research with human subjects, does the paper include the full text of instructions given to participants and screenshots, if applicable, as well as details about compensation (if any)? 
    \item[] Answer: \answerNA{} 
    \item[] Justification: The paper does not involve crowdsourcing nor research with human subjects

\item {\bf Institutional Review Board (IRB) Approvals or Equivalent for Research with Human Subjects}
    \item[] Question: Does the paper describe potential risks incurred by study participants, whether such risks were disclosed to the subjects, and whether Institutional Review Board (IRB) approvals (or an equivalent approval/review based on the requirements of your country or institution) were obtained?
    \item[] Answer: \answerNA{} 
    \item[] Justification: The paper does not involve crowdsourcing nor research with human subjects

\end{enumerate}

\end{document}